\documentclass{article}

\usepackage[accepted]{icml2023}

\usepackage{microtype}
\usepackage{graphicx}
\usepackage[caption=false]{subfig}
\usepackage{booktabs}

\usepackage{hyperref}

\usepackage[utf8]{inputenc} 
\usepackage[T1]{fontenc}    
\usepackage{url}            
\usepackage{booktabs}       
\usepackage{amsfonts}       
\usepackage{nicefrac}       
\usepackage{microtype}      
\usepackage{nccmath}
\usepackage{times}
\usepackage{natbib}
\usepackage{algorithm}
\usepackage{algorithmic}
\usepackage{amsmath}

\usepackage{mathtools}
\usepackage{amsthm}
\usepackage{graphicx}
\usepackage{enumerate}
\usepackage{algorithm}
\usepackage{algorithmic}
\usepackage{thm-restate}

\usepackage{dsfont}
\usepackage[mathscr]{euscript}
\usepackage{booktabs}
\usepackage{makecell}
\usepackage{tabularx}
\usepackage{xcolor, colortbl}
\usepackage[normalem]{ulem}
\usepackage{soul}
\usepackage{prettyref}
\usepackage{bbm}
\usepackage{multirow}

\usepackage{diagbox}

\usepackage{MnSymbol}
\DeclareMathAlphabet\mathbb{U}{msb}{m}{n}
\usepackage{xpatch}

\DeclareMathOperator{\sign}{sign}

\newcommand{\sC}{{\mathscr C}}

\newcommand{\sH}{{\mathscr H}}

\newcommand{\ignore}[1]{}

\newcommand{\vw}{{\boldsymbol w}}
\newcommand{\vb}{{\boldsymbol b}}
\newcommand{\vx}{{\boldsymbol x}}
\newcommand{\vp}{{\boldsymbol p}}
\newcommand{\vA}{{\boldsymbol A}}
\newcommand{\vh}{{\boldsymbol h}}

\theoremstyle{plain}
\newtheorem{theorem}{Theorem}[section]
\newtheorem{prop}{Proposition}[section]
\newtheorem{lemma}{Lemma}[section]

\newtheorem{mydef}{Definition}[section]
\newtheorem{assumption}{Assumption}[section]

\hypersetup{
  colorlinks   = true,
  urlcolor     = blue,
  linkcolor    = blue,
  citecolor    = blue
}

\usepackage[textsize=tiny]{todonotes}

\usepackage[toc, page, header]{appendix}
\icmltitlerunning{Revisiting Discriminative vs. Generative Classifiers: Theory and Implications}

\begin{document}

\twocolumn[
\icmltitle{Revisiting Discriminative vs. Generative Classifiers: Theory and Implications}

\begin{icmlauthorlist}
\icmlauthor{Chenyu Zheng}{ruc}
\icmlauthor{Guoqiang Wu}{sdu}
\icmlauthor{Fan Bao}{thu}
\icmlauthor{Yue Cao}{baai}
\icmlauthor{Chongxuan Li}{ruc}
\icmlauthor{Jun Zhu}{thu}
\end{icmlauthorlist}

\icmlaffiliation{ruc}{Gaoling School of AI, Renmin University of China; Beijing Key Laboratory of Big Data Management and Analysis Methods, Beijing, China}
\icmlaffiliation{sdu}{School of Software, Shandong University}
\icmlaffiliation{thu}{Dept. of Comp. Sci. \& Tech., Institute for AI, Tsinghua-Huawei Joint Center for AI, BNRist Center, THBI Lab, Tsinghua University}
\icmlaffiliation{baai}{Beijing Academy of Artificial Intelligence}

\icmlcorrespondingauthor{Chongxuan Li}{chongxuanli@ruc.edu.cn}

\icmlkeywords{discriminative vs. generative classifiers, deep representation learning, consistency}

\vskip 0.3in
]
\printAffiliationsAndNotice{}

\begin{abstract}
A large-scale deep model pre-trained on massive labeled or unlabeled data transfers well to downstream tasks. \emph{Linear evaluation} freezes parameters in the pre-trained model and trains a linear classifier separately, which is efficient and attractive for transfer.
However, little work has investigated the classifier in linear evaluation except for the default logistic regression.
Inspired by the statistical efficiency of na\"ive Bayes, the paper revisits the classical topic on \emph{discriminative vs. generative classifiers}~\cite{DBLP:conf/nips/NgJ01}. Theoretically, the paper considers the surrogate loss instead of the zero-one loss in analyses and generalizes the classical results from binary cases to multiclass ones. We show that, under mild assumptions, multiclass na\"ive Bayes requires $O(\log n)$ samples to approach its asymptotic error while the corresponding multiclass logistic regression requires $O(n)$ samples, where $n$ is the feature dimension. To establish it, we present a \emph{multiclass $\mathcal{H}$-consistency bound} framework and an explicit bound for logistic loss, which are of independent interests. 
Simulation results on a mixture of Gaussian validate our theoretical findings. Experiments on various pre-trained deep vision models show that naïve Bayes consistently converges faster as the number of data
increases. Besides, naïve Bayes shows promise in few-shot cases and we observe the ``two regimes'' phenomenon in pre-trained supervised models. Our code is available at \emph{\href{https://github.com/ML-GSAI/Revisiting-Dis-vs-Gen-Classifiers}{https://github.com/ML-GSAI/Revisiting-Dis-vs-Gen-Classifiers}}.
\end{abstract}

\section{Introduction}
\label{introduction}
Deep representation learning has achieved great success in many fields such as computer vision \cite{ren2015faster,he2017mask,chen2020generative,DBLP:conf/cvpr/He0WXG20Moco, DBLP:conf/icml/ChenK0H20SimCLR, DBLP:conf/cvpr/ChenH21SimSiam,grill2020bootstrap,DBLP:conf/cvpr/HeCXLDG22MAE}, natural language processing \cite{DBLP:conf/naacl/DevlinCLT19Bert, DBLP:conf/nips/BrownMRSKDNSSAA20GPT3,raffel2020exploring} and cross-modal learning~\cite{CLIP} over the past few years.
The common paradigm behind them is to (pre-)train a large-scale model on an enormous amount of labeled or unlabeled data and transfer it to downstream tasks. During the transfer, \emph{linear evaluation}~\cite{chen2020generative,DBLP:conf/cvpr/He0WXG20Moco, DBLP:conf/icml/ChenK0H20SimCLR,grill2020bootstrap,CLIP}
freezes all parameters in the pre-trained model and learns a linear classifier separately. Theoretically,  it is validated by the (approximate) linear separability of the representations extracted by pre-trained models~\cite{DBLP:conf/icml/SaunshiPAKK19, DBLP:conf/nips/LeeLSZ21, DBLP:conf/alt/ToshK021, DBLP:conf/nips/HaoChenWGM21}. 
Practically, 
linear evaluation is an efficient and attractive alternative to fine-tuning, considering the extremely large and continually growing size of modern pre-trained models.


Although new algorithms and models for deep pre-training emerge in endlessly, little work has investigated the classifier except for the default logistic regression. Directly inspired by the classical work~\cite{efron1975efficiency,DBLP:conf/nips/NgJ01} (detailed in Section~\ref{sec: Preliminaries}) on the statistical efficiency of generative linear classifiers (e.g. na\"ive Bayes), we revisit the discriminative vs. generative linear classifiers in the context of deep representation learning. 

In Section~\ref{sec: Theory}, we improve the classical theory~\cite{DBLP:conf/nips/NgJ01} in two aspects for subsequent analysis in deep representation learning. First, we characterize asymptotic behaviors of both multiclass na\"ive Bayes and logistic regression, generalizing the results in binary classification~\cite{DBLP:conf/nips/NgJ01}. Second, in logistic regression, we consider the practically used surrogate loss in our analysis instead of directly optimizing the zero-one loss as assumed in~\cite{DBLP:conf/nips/NgJ01}. To establish it, we introduce a general \emph{multiclass $\mathcal{H}$-consistency bound} framework upon recent advances~\cite{DBLP:conf/icml/AwasthiMM022} and a nontrivial explicit bound for multiclass logistic regression, which are of independent interests. We prove that for a fixed number of classes, the number of samples required to approach the corresponding optimal classifier is $O(\log n)$ and $O(n)$ for na\"ive Bayes and logistic regression respectively, where $n$ is the feature dimension. We conduct synthetic experiments with tractable $\mathcal{H}$-optimal classifiers to validate our theory.

In Section~\ref{sec: implications}, we discuss the implications of our theory in the linear evaluation of pre-trained deep models. 
We first analyze the main assumptions in our theory upon deep representations.
We then perform extensive experiments on CIFAR10 and CIFAR100 datasets with various representative pre-trained vision models~\cite{resnet,dosovitskiy2020image,DBLP:journals/corr/abs-2003-04297MocoV2,DBLP:conf/nips/simclrv2,CLIP, DBLP:conf/cvpr/simmim, DBLP:conf/cvpr/HeCXLDG22MAE}, which are trained in supervised or self-supervised manners.
The results show that na\"ive Bayes consistently converges faster as the number of data increases in all settings, which agrees with our theory. Besides, na\"ive Bayes shows promise in few-shot cases and we observe the ``two regimes'' phonomenan~\cite{DBLP:conf/nips/NgJ01} in models
pre-trained in a supervised manner, suggesting a distinction between the representations learned by supervised and self-supervised approaches.

\section{Preliminaries}
\label{sec: Preliminaries}
In this section, we present notations and preliminaries on discriminative vs. generative classifiers and $\mathcal{H}$-consistency. 

Let lower, boldface lower and capital case letters denote scalers (e.g., a), vectors (e.g., $\boldsymbol{a}$), and matrices (e.g., $\boldsymbol{A}$), respectively. For a matrix $\boldsymbol{A}$, $\boldsymbol{A}_i$ and $A_{ij}$ denote its $i$-th row and $(i, j)$-th element. For a vector $\boldsymbol{a}$, $a_i$ denotes its $i$-th element. Similarly, for a vector function $\boldsymbol{f}$, $f_i(\vx)$ denotes the $i$-th element of $\boldsymbol{f}(\vx)$.  We do not distinguish  constants and random variables in notations if there is no confusion.  We denote the KL divergence between distributions $p$ and $q$ by $D(p \Vert q)$. We use $\mathbb{E}$, $\mathbb{V}$, $\Delta_k$ to represent expectation, variance, and $k$-dimensional possibility simplex, respectively.

Let $\mathcal{X}$ denote the domain set and $\mathcal{Y}= \{1,\dots, K\}$
denote the label set, where $K$ is the number of classes. For simplicity, we assume $\mathcal{X} = \{0,1\}^n$ when inputs are discrete and $\mathcal{X} = [0,1]^n$ otherwise, where $n$ is the feature dimension. Note that our analysis can be easily extended to the general case with any bounded features. 
Let $\mathcal{H}$ be a hypothesis set of functions mapping from $\mathcal{X} \times \mathcal{Y}$ to $\mathbb{R}^K$. The prediction associated by a hypothesis $\vh \in \mathcal{H}$ and $\vx \in \mathcal{X}$ is $\mathop{\mathrm{argmax}}_{y \in \mathcal{Y}} h_y(\vx)$. In the main paper, we focus on the family of constrained linear hypotheses $\mathcal{H}_{lin} = \{\vx \to \boldsymbol{h}(\vx): h_y(\vx) = \langle \vw_y, \vx\rangle + b_y, \Vert \vw_y \Vert_2 \leq W, \vert b_y\vert \leq B, y \in \mathcal{Y}\}$, where  $W, B \in \mathbb{R}^+$. We also denote the hypothesis set of all measurable functions by $\mathcal{H}_{all}$. Given a hypothesis set $\mathcal{H}$ and distribution $\mathcal{D}$, the generalization error and minimal generalization error of a hypothesis $\vh$ with respect to  the loss function $\ell: \mathbb{R}^K \times \mathcal{Y} \rightarrow \mathbb{R}$ are defined as $R_{\ell}(\vh) = \mathbb{E}_{(\vx,y) \sim \mathcal{D}} [\ell(\vh(\vx), y)]$ and $R_{\ell, \mathcal{H}}^* = \inf_{\vh \in \mathcal{H}} R_{\ell}(\vh)$.

\subsection{Discriminative vs. Generative Classifiers}

$K$-class logistic regression is parameterized by $[\vw_1,\dots,\vw_K, \vb]$, where $\vw_i \in \mathbb{R}^n$ and $\vb \in \mathbb{R}^K$. Its prediction is given by $\mathop{\mathrm{argmax}}_{y  \in \mathcal{Y}} (\langle \vw_y, \vx\rangle + b_y)$.

It's well known that the generative counterpart of the logistic regression is na\"ive Bayes (with some constraints presented later)~\cite{DBLP:conf/nips/NgJ01, DBLP:conf/kdd/RubinsteinH97}. When inputs are discrete, a na\"ive Bayes classifier uses a training set with $m$ i.i.d examples to calculate the empirical conditional distributions $\hat{p}(x_i\vert y)$ and empirical marginal distribution $\hat{p}(y)$ as follows:
\begin{align}
    \hat{p}(x_i = 1\vert y = k) &= \frac{\#\{x_i = 1, y = k\} + \alpha}{\#\{y = k\} + K\alpha}, \label{eq:nb estimation 1}\\
    \hat{p}(y = k) &= \frac{\#\{y = k\} + \alpha}{m + K\alpha}, \label{eq:nb estimation 2}
\end{align}
where $\#\{\cdot\}$ is the counting function and $\alpha$ is a positive Laplace smoothing parameter. Corresponding population versions are denoted by $p(x_i\vert y)$ and $p(y)$ respectively. In case of continuous inputs, we let $\hat{p}(x_i \vert y = k)$ be a univariate Gaussian distribution with parameters $\hat{\mu}_{ki}$ and $\hat{\sigma}_i^2$. We note that $\hat{\sigma}^2_i$s do not depend on $y$ to keep the linearity of its decision boundary, otherwise logistic regression and na\"ive Bayes are no longer a fair discriminative-generative pair~\cite{DBLP:journals/npl/XueT08}. They are calculated as the empirical version of ${\mu}_{ki}  =\mathbb{E}[x_i \vert y = k]$ and ${\sigma}_i^2 = \mathbb{E}_y[\mathbb{V}(x_i\vert y)]$. 

\citet{DBLP:conf/nips/NgJ01} proved that in binary classification, logistic regression enjoys a lower asymptotic error but approaches it much slower (w.r.t. the sample size) than na\"ive Bayes. The theory explains the \emph{``two regimes''}~\cite{DBLP:conf/nips/NgJ01} phenomenon in practice. In particular,  na\"ive Bayes generalizes better with limited data. However, the multiclass case has not been investigated yet, which is the main focus of this paper. Besides, prior work~\cite{DBLP:conf/nips/NgJ01} assumes that the zero-one loss can be directly optimized in logistic regression, which is impractical. To weaken the assumption, we introduce tools from \emph{$\mathcal{H}$-consistency}.

\subsection{$\mathcal{H}$-consistency}


$\mathcal{H}$-consistency~\cite{long2013consistency} analyzes the relationship between the estimation error of zero-one loss w.r.t. a hypothesis class $\mathcal{H}$ and that of a surrogate loss. It includes the classical Bayes consistency~\cite{zhang2004statistical,bartlett2006convexity,tewari2007consistency} as a special case by setting $\mathcal{H}$ to $\mathcal{H}_{all}$. In this paper, we analyze the linear discriminative vs. generative classifiers upon recent advances on $\mathcal{H}$-consistency bounds~\cite{, DBLP:conf/icml/AwasthiMM022}.


We first introduce some notations. We denote by $\vp(\vx)$ the conditional distribution of $Y$ given $\vx$, i.e., $p_y(\vx) = \mathbb{P}(Y=y\vert X=\vx)$. We define the conditional risk as $\mathscr{C}_{\ell}(\boldsymbol{h}, \vx) = \sum_{y=1}^K p_y(\vx) \ell(\vh(\vx), y)$, and note that generalization error $R_{\ell}(\vh)$ can be rewritten as $\mathbb{E}_{\vx} [\mathscr{C}_{\ell} (\vh, \vx)]$. We also define its infimum $\mathscr{C}_{\ell, \mathcal{H}}^*(\vx) = \inf_{\boldsymbol{h} \in \mathcal{H}} \mathscr{C}_{\ell}(\boldsymbol{h}, \vx)$ and the gap between them $\Delta \mathscr{C}_{\ell, \mathcal{H}}(\boldsymbol{h}, \vx) = \mathscr{C}_{\ell}(\boldsymbol{h}, \vx) - \mathscr{C}_{\ell, \mathcal{H}}^*(\vx)$. A key quantity appears in our bounds is $M_{\ell, \mathcal{H}} = R_{\ell, \mathcal{H}}^* - \mathbb{E}_{\vx}(\mathscr{C}_{\ell, \mathcal{H}}^* (\vx))$, which is difficult to estimate~\cite{DBLP:conf/icml/AwasthiMM022}, but can be bounded by the approximate error. In addition, for any $\vp$ in probability simplex $\Delta_K$, we can define $\mathscr{C}_{\ell}(\boldsymbol{h}, \vx, \vp) = \sum_{y=1}^K p_y \ell(\boldsymbol{h}(\vx), y)$ and $\Delta \mathscr{C}_{\ell, \mathcal{H}}(\boldsymbol{h}, \vx, \vp) = \mathscr{C}_{\ell}(\boldsymbol{h}, \vx, \vp) - \inf_{\boldsymbol{h} \in \mathcal{H}} \mathscr{C}_{\ell}(\boldsymbol{h}, \vx, \vp)$. Furthermore, we define the $\epsilon$-regret of $t$ as $\langle t \rangle_\epsilon = t \mathbbm{1}_{t > \epsilon}$.


The general $\mathcal{H}$-consistency bound~\cite{DBLP:conf/icml/AwasthiMM022} for two loss functions $\ell_1$ and $\ell_2$ is defined as follows.
\begin{mydef}[$\mathcal{H}$-consistency bound]
\label{Def :h consistency bound}
$\mathcal{H}$-consistency bound is in the following form that holds for all $\vh \in \mathcal{H}$, $\mathcal{D} \in \mathcal{P}$  and some non-decreasing function $f: \mathbb{R}_+ \to \mathbb{R}_+$:
\begin{align}
    R_{\ell_2}(\vh) - R_{\ell_2, \mathcal{H}}^* \leq f(R_{\ell_1}(\vh) - R_{\ell_1, \mathcal{H}}^*).
\end{align}
If $\mathcal{P}$ is composed of all distributions over $\mathcal{X} \times \mathcal{Y}$, we call it a distribution-independent bound.
\end{mydef}

Note that it covers the classical Bayes consistency bounds~\cite{bartlett2006convexity} by setting $\mathcal{H}=\mathcal{H}_{all}$. When $\ell_1$ is logistic loss $\ell_{log}$ and $\ell_2$ is zero-one loss $\ell_{0-1}$,~\citet{DBLP:conf/icml/AwasthiMM022} proved the following $\mathcal{H}$-consistency bound w.r.t. the bounded linear hypotheses.

\begin{theorem}[$\mathcal{H}$-consistency bound for binary logistic loss and zero-one loss, Appendix K.1.2~\cite{DBLP:conf/icml/AwasthiMM022}]
\label{lemma: binary H consistency bounds}
Given binary linear hypothesis set $\mathcal{H} = \{\vx \to \langle \vw, \vx\rangle + b: \Vert \vw \Vert_2 \leq W, \vert b\vert \leq B\}$, if $R_{\ell_{log}}(h) - R^*_{\ell_{log}, \mathcal{H}} + M_{\ell_{log}, \mathcal{H}} \leq \frac{1}{2}({\frac{e^{B}-1}{e^{B}+ 1}})^2$, then it holds for any distribution that $R_{\ell_{0-1}}(h) - R^*_{\ell_{0-1}, \mathcal{H}} + M_{\ell_{0-1}, \mathcal{H}} \leq \sqrt{2}(R_{\ell_{log}}(h) - R^*_{\ell_{log}, \mathcal{H}} + M_{\ell_{log}, \mathcal{H}})^{\frac{1}{2}}$.
\end{theorem}

To the best of our knowledge, there is no $\mathcal{H}$-consistency bound for logistic loss and zero-one loss in multiclass classification\footnote{Most recently, the concurrent and independent work of~\citet{DBLP:journals/corr/maoanqi} also studies this problem and obtains similar results to ours.}. In this paper, we extend the binary framework~\cite{DBLP:conf/icml/AwasthiMM022} to multiclass cases and derive an explicit bound for logistic loss.



\section{Theory}
\label{sec: Theory}

In this section, we present our main theoretical results in Section~\ref{sec: Discriminative vs. Generative: Multiclass Classification}: Under some mild assumptions, for any fixed class number $K$, the number of training samples required by  na\"ive Bayes to approach its asymptotic error is $O(\log n)$ (Theorem~\ref{cor: multiclass NB sample complexity}), and that of logistic regression is $O(n)$ (Theorem~\ref{cor: sample complexity of multiclass lr}).
To establish it, we propose a general multiclass $\mathcal{H}$-consistency framework (Theorem~\ref{cor: Distribution-independent convex Psi bound})  and a nontrivial multiclass $\mathcal{H}$-consistency bound for logistic loss and zero-one loss  (Theorem~\ref{thm: H-consistency bound for log}) in Section~\ref{sec: multiclass H-consistency framework}. Notably, our theory includes the analysis for $K = 2$ in Appendix~\ref{sec: Discriminative vs. Generative: Binary Classification} as a special case.

\subsection{On Multiclass Discriminative vs. Generative  Linear Classifiers}
\label{sec: Discriminative vs. Generative: Multiclass Classification}



Let $\boldsymbol{h}_{Dis, m}$ and $\boldsymbol{h}_{Gen, m}$ denote the hypothesis returned by multiclass logistic regression and na\"ive Bayes with $m$ $i.i.d$ samples, respectively. Let $\boldsymbol{h}_{Dis, \infty}$ and $\boldsymbol{h}_{Gen, \infty}$ be the corresponding asymptotic version. We are interested in comparing the statistical efficiency of na\"ive Bayes and logistic regression~\cite{DBLP:conf/nips/NgJ01}. Formally, we need to bound $R_{\ell_{0-1}} (\boldsymbol{h}_{Gen, m}) - R_{\ell_{0-1}} (\boldsymbol{h}_{Gen, \infty})$ and $R_{\ell_{0-1}} (\boldsymbol{h}_{Dis, m}) - R_{\ell_{0-1}} (\boldsymbol{h}_{Dis, \infty})$ respectively.

\textbf{Na\"ive Bayes.}   Notably, the solution of Na\"ive Bayes is in a closed-form, as presented in Eq.~(\ref{eq:nb estimation 1}\&\ref{eq:nb estimation 2}). Therefore, we can characterize the gap between parameters in $\boldsymbol{h}_{Gen, m}$ and $\boldsymbol{h}_{Gen, \infty}$ to bound $R_{\ell_{0-1}} (\boldsymbol{h}_{Gen, m}) - R_{\ell_{0-1}} (\boldsymbol{h}_{Gen, \infty})$, similarly to the binary case~\cite{DBLP:conf/nips/NgJ01}. 


We make two mild assumptions about the data distribution similar to~\citet{DBLP:conf/nips/NgJ01}. We avoid trivial cases where $p(y = k) =1$ or $p(y = k) =0$ for some $k$ in Assumption~\ref{Assumption: p(y=k)} and  assume that the conditional distribution of $\vx$ given $y$ can not be too concentrated in Assumption~\ref{Assumption: parmeters bounded}. 
\begin{assumption}
\label{Assumption: p(y=k)}
For some fixed $\rho_1 \in (0, \frac{1}{2}]$, we have that $\rho_1 \leq p(y = k) \leq 1 - \rho_1$ for all $k \in \mathcal{Y}$.
\end{assumption}

\begin{assumption}
\label{Assumption: parmeters bounded}
For some fixed $\rho_2 \in (0, \frac{1}{2}]$, $\rho_2 \leq p(x_i = 1\vert y = k) \leq 1 - \rho_2$ for all $i, k$ in the discrete case, and $\sigma^2_i \ge \rho_2$ for all $i$ in the continuous case.
\end{assumption}

In practice, most deep learning work considers the balanced case where $\rho_1 = \frac{1}{K}$~\cite{imagenet}. 
Empirically, we found that $\rho_2 \in [10^{-5}, 10^{-2}]$ on the features extracted by representative pre-trained vision models in Section~\ref{sec: implications}. For clarity, we denote $\rho_0 = \min\{\rho_1, \rho_2\}$ throughout the paper. We now define two key quantities in our proof as follows.



\begin{mydef}[Pair activation function of na\"ive Bayes]
\label{Def :multiclass delta a}
For every $k_1, k_2 \in \mathcal{Y}$, we define the pair activation function $\Delta a_{Gen}(\vx, k_1, k_2)$ as
\begin{align}
\Delta a_{Gen}(\vx, k_1, k_2) = a_{Gen}(\vx, k_1) - a_{Gen}(\vx, k_2),
\end{align}
where $a_{Gen}(\vx, k) = \sum_{i=1}^n \log \hat{p}(x_i\vert y=k)  + \log\hat{p}(y=k)$.
\end{mydef}

The paired activation function is important because it connects the estimated parameters and predictions of the hypothesis. For instance,  $\Delta a_{Gen}(\vx, k_1, k_2) > 0$ means that $\vx$ is more likely to be predicted as an instance of class $k_1$ than class $k_2$. We can easily bound the gap between the parameters in $\boldsymbol{h}_{Gen, m}$ and $\boldsymbol{h}_{Gen, \infty}$ by standard concentration inequalities. To bound $R_{\ell_{0-1}} (\boldsymbol{h}_{Gen, m}) - R_{\ell_{0-1}} (\boldsymbol{h}_{Gen, \infty})$ as presented in Theorem~\ref{Thm: multiclass generalization bound}, we further upper bound the probability of getting ``bad training samples'', which are predicted as different classes with high probability by $\boldsymbol{h}_{Gen, m}$ and $\boldsymbol{h}_{Gen, \infty}$, via the following 
$\widetilde{G}(\tau)$.

\begin{mydef}
\label{Def :multiclass G}
We define the function $\widetilde{G}(\tau)$  as follows:
\begin{align*}
\widetilde{G}(\tau) = \max_{k_1, k_2} \mathbb{P}_{(\vx,y) \sim \mathcal{D}}(\vert \Delta a_{Gen, \infty}(\vx, k_1, k_2)\vert \leq \tau n).
\end{align*}
\end{mydef}

\begin{theorem}[Proof in Appendix~\ref{proof: Proof of Theorem Thm: multiclass generalization bound}]
\label{Thm: multiclass generalization bound}
Suppose that Assumption~\ref{Assumption: p(y=k)} and~\ref{Assumption: parmeters bounded} are valid.  Then with probability at least $1-\delta$:
\begin{align*}
    R_{\ell_{0-1}}(\boldsymbol{h}_{Gen, m}) &\leq R_{\ell_{0-1}}(\boldsymbol{h}_{Gen, \infty}) \\
    &+ \frac{K(K-1)}{2} \biggl(\widetilde{G}\bigl(O(\sqrt{\frac{1}{m} \log(\frac{n}{\delta})})\bigr) + \delta\biggr).
\end{align*}
\end{theorem}

The core of Theorem~\ref{Thm: multiclass generalization bound} is the $\widetilde{G}(\tau)$, which must be small when $\tau$ is small in order to obtain meaningful bound about $R_{\ell_{0-1}}(\boldsymbol{h}_{Gen, m}) - R_{\ell_{0-1}}(\boldsymbol{h}_{Gen, \infty})$. It holds under the following assumptions, similarly to~\citet{DBLP:conf/nips/NgJ01}.
\begin{assumption}
\label{Assumption: multiclass KL}
For all $k_1, k_2 (k_1 \ne k_2)$ and $k \in \mathcal{Y}$, it holds that $\lvert \sum_{i=1}^n (D(p(x_i \vert y=k) \Vert p(x_i \vert y=k_1)) - D(p(x_i \vert y=k) \Vert p(x_i \vert y=k_2))) \rvert = \beta_{k_1, k_2, k}n = \Omega(n)$. 
\end{assumption}

\begin{assumption}
\label{Assumption: multiclass likelihood ratio var}
    For all $k_1, k_2 (k_1 \ne k_2)$ and $k \in \mathcal{Y}$, it holds that $\mathbb{V}_{\vx}[\sum_{i=1}^n \log \frac{{p}(x_i\vert y=k_1) }{{p}(x_i\vert y=k_2) } \vert y = k] = \alpha_{k_1, k_2, k}n = O(n^r)$ for any $r \in [1,2)$.
\end{assumption}

Intuitively, Assumption~\ref{Assumption: multiclass KL} requires that $\Omega(1)$ fraction of features distinct for any two different classes. Assumption~\ref{Assumption: multiclass likelihood ratio var} is more technical. In fact, it is derived when we attempt to bound $\widetilde{G}(\tau)$ via Chebyshev's inequality\footnote{Indeed, if the na\"ive Bayes assumption really holds, we can obtain a stronger guarantee for $\widetilde{G}(\tau)$ by using Chernoff's bound. We put the result in Proposition~\ref{Prop: multiclass exp -n}.}. We empirically analyze both assumptions in Section~\ref{sec: implications}. Proposition~\ref{Prop: multiclass ploy -n} presents a meaningful bound for $\widetilde{G}(\tau)$, which is followed by the main result of na\"ive Bayes in Theorem~\ref{cor: multiclass NB sample complexity}. 

\begin{prop}[Proof in Appendix~\ref{proof: Proposition Prop: multiclass poly -n}]
\label{Prop: multiclass ploy -n}
Suppose that Assumption~\ref{Assumption: p(y=k)},~\ref{Assumption: multiclass KL} and~\ref{Assumption: multiclass likelihood ratio var} hold, then $\widetilde{G}(\tau)$ is polynomially small in $n$:
\begin{equation*}
    \widetilde{G}(\tau) \leq \frac{\alpha}{(\tau - \zeta)^2 n},
\end{equation*}
where $\alpha = \max_{k_1, k_2, k} \alpha_{k_1,k_2,k} = O(n^{r-1})$, $\mathbb{E}_{\vx}[\Delta a_{Gen, \infty}(\vx, k_1, k_2)\vert y=k] = \zeta_{k_1,k_2,k} n$, $\zeta = \min_{k_1, k_2, k} \vert\zeta_{k_1,k_2,k}\vert = \Omega(1)$ and $\tau < \zeta$.
\end{prop}

\begin{theorem}[Results for na\"ive Bayes, proof in Appendix~\ref{proof: Proof of Corollary cor: multiclass NB sample complexity}]
\label{cor: multiclass NB sample complexity}
Suppose the precondition of Proposition~\ref{Prop: multiclass ploy -n} holds. Then, it suffices to pick $m = O(\log n)$ training samples such that $R_{\ell_{0-1}}(\boldsymbol{h}_{Gen, m}) \leq R_{\ell_{0-1}}(\boldsymbol{h}_{Gen, \infty}) +\epsilon_0$  hold with probability $1 - \delta_0$, for any $\epsilon_0 \in (0,1)$ and $\delta_0 \in (0, \frac{\epsilon_0}{K^2}]$.
\end{theorem}

\textbf{Logistic Regression.} To directly compare with na\"ive Bayes, we aim to bound $R_{\ell_{0-1}} (\boldsymbol{h}_{Dis, m}) - R_{\ell_{0-1}} (\boldsymbol{h}_{Dis, \infty})$. However, the optimization of logistic regression does not have an analytic form, making the proof idea of na\"ive Bayes infeasible. Besides, \citet{DBLP:conf/nips/NgJ01} proves the bound by directly optimizing the zero-one loss, which is impractical. Instead, we present a bound considering the surrogate logistic loss in this paper.  To establish it, we exploit recent advances on $\mathcal{H}$-consistency bound~\cite{DBLP:conf/icml/AwasthiMM022} as detailed in Defition~\ref{Def :h consistency bound}. It is worth discussing an alternative approach based on \emph{Bayes consistency bounds}~\cite{bartlett2006convexity}. For a direct comparison with na\"ive Bayes, we care about the asymptotic error in $\mathcal{H}_{lin}$ instead of $\mathcal{H}_{all}$. Therefore, a $\mathcal{H}$-consistency bound is more natural and potentially tighter than a Bayes consistency bound. In fact, existing Bayes consistency bounds~\cite{bartlett2006convexity} are special cases of the $\mathcal{H}$-consistency bounds~\cite{DBLP:conf/icml/AwasthiMM022}.

Note that the binary $\mathcal{H}$-consistency bound~\cite{DBLP:conf/icml/AwasthiMM022} in Theorem~\ref{lemma: binary H consistency bounds} does not directly apply to multiclass cases. We generalize the binary framework~\cite{DBLP:conf/icml/AwasthiMM022} to multiclass cases and prove an explicit $\mathcal{H}$-consistency bound for logistic loss. We present the bound in Theorem~\ref{thm: H-consistency bound for log} and defer the establishment to Section~\ref{sec: multiclass H-consistency framework}.

\begin{theorem}
[$\mathcal{H}$-consistency bound for multiclass logistic loss and zero-one loss, proof in Appendix~\ref{proof: thm: H-consistency bound for log}]
\label{thm: H-consistency bound for log}
If $R_{\ell_{log}}(\vh) - R^*_{\ell_{log}, \mathcal{H}_{lin}} + M_{\ell_{log}, \mathcal{H}_{lin}} \leq \frac{1}{2}({\frac{e^{2B}-1}{e^{2B}+ K - 1}})^2$, then for any distribution satisfiying $\max_y p_y(\vx) - \min_y p_y(\vx) \leq \frac{e^{2B} - 1}{e^{2B} + K - 1}$ for all $\vx$, it holds that $R_{\ell_{0-1}}(\boldsymbol{h}) - R^*_{\ell_{0-1}, \mathcal{H}_{lin}} + M_{\ell_{0-1}, \mathcal{H}_{lin}} \leq \sqrt{2}(R_{\ell_{log}}(\boldsymbol{h}) - R^*_{\ell_{log}, \mathcal{H}_{lin}} + M_{\ell_{log}, \mathcal{H}_{lin}})^{\frac{1}{2}}$.
\end{theorem}

Note that $R_{\ell_{0-1}} (\boldsymbol{h}_{Dis, \infty}) = R^*_{\ell_{log}, \mathcal{H}_{lin}}$ by the definition. Besides, when $B \to +\infty$, we have $ \frac{e^{2B} - 1}{e^{2B} + K - 1} \to 1$, and Theorem~\ref{thm: H-consistency bound for log}  holds for all distribution. Theorem~\ref{thm: H-consistency bound for log} provides a tool to analyze the asymptotic behavior of multiclass logistic regression considering the surrogate loss. According to it, we need to bound the gap $R_{\ell_{log}}(\boldsymbol{h}_{Dis, m})-R_{\ell_{log}}(\boldsymbol{h}_{Dis, \infty})$ and $M_{\ell_{log}, \mathcal{H}_{lin}}$ to guarantee a small $R_{\ell_{0-1}}(\boldsymbol{h}_{Dis, m})-R_{\ell_{0-1}}(\boldsymbol{h}_{Dis, \infty})$. The following Proposition characterizes $R_{\ell_{log}}(\boldsymbol{h}_{Dis, m})-R_{\ell_{log}}(\boldsymbol{h}_{Dis, \infty})$ by Radmancher complexity~\cite{DBLP:conf/colt/BartlettBM02, mohri2018foundations} and a contraction lemma~\cite{maurer2016vector}.

\begin{prop}[Proof in appendix~\ref{proof: Prop: multiclass logisticbound}]
\label{Prop: multiclass logisticbound}
For any fixed $\delta_0 \in (0,1)$, with probability at least $1 - \delta_0$, the following holds:
\begin{equation*}
\label{Eq: multiclass logisticboundO}
      R_{\ell_{log}}(\boldsymbol{h}_{Dis, m}) \leq R_{\ell_{0-1}} (\boldsymbol{h}_{Dis, \infty}) + O(\sqrt{\frac{K^3n}{m}}).
\end{equation*}
\end{prop}

$M_{\ell, \mathcal{H}}$ is a constant determined by the hypothesis set $\mathcal{H}$, loss function $\ell$, and data distribution $\mathcal{D}$. Its value is difficult to  estimate directly~\cite{DBLP:conf/icml/AwasthiMM022}. However, according to the definition, $M_{\ell, \mathcal{H}}$ can be bounded by the corresponding approximate error. Prior works~\cite{DBLP:conf/icml/SaunshiPAKK19, DBLP:conf/nips/LeeLSZ21, DBLP:conf/alt/ToshK021, DBLP:conf/nips/HaoChenWGM21} prove the (approximate) linear separability of the representations extracted by deep pre-trained models, suggesting a small approximation error for the logistic loss. Therefore, we make the following assumption, which is validatable in the context of linear evaluation of deep models.

\begin{assumption}
\label{Assumption: Optimal classifier has finite empirical loss}
The approximate error of the logistic loss is bounded by a small constant $\nu < \frac{1}{2}({\frac{e^{2B}-1}{e^{2B}+ K - 1}})^2$. Namely,  $\mathop{\mathrm{argmin}}_{\boldsymbol{h} \in \mathcal{H}_{lin}} R_{\ell_{log}}(\boldsymbol{h}) - \mathop{\mathrm{argmin}}_{\boldsymbol{h} \in \mathcal{H}_{all}} R_{\ell_{log}}(\boldsymbol{h}) \leq \nu$, which implies that $M_{\ell_{log}, \mathcal{H}_{lin}} \le \nu$.
\end{assumption}

We characterize the number of samples required to approach the asymptotic error for logistic regression in Theorem~\ref{cor: sample complexity of multiclass lr} by combining Proposition~\ref{Prop: multiclass logisticbound} and Theorem~\ref{thm: H-consistency bound for log}.

\begin{theorem}[Results for multiclass logistic regression, proof in appendix~\ref{proof: Proof of Corollary cor: sample complexity of multiclass lr}]
\label{cor: sample complexity of multiclass lr}
Suppose that Assumption~\ref{Assumption: Optimal classifier has finite empirical loss} holds.  Then, it suffices to pick $m = O(n)$ training samples such that $R_{\ell_{0-1}}(\boldsymbol{h}_{Dis, m}) \leq R_{\ell_{0-1}}(\boldsymbol{h}_{Dis, \infty}) +\epsilon_0$  hold with probability $1 - \delta_0$, for any fixed $\epsilon_0 \in [\sqrt{2\nu}, {\frac{e^{2B}-1}{e^{2B}+ K - 1}}]$ and $\delta_0 \in (0, 1)$.
\end{theorem}

Notably, according to the multiclass fundamental theorem (Theorem 29.3 of~\citet{shalev2014understanding}), the sample complexity of $\mathcal{H}_{lin}$ for any algorithm is $\Omega(n)$ because the Natarajan dimension for $\mathcal{H}_{lin}$ is $\Omega(Kn)$, indicating the upper bound in Thereom~\ref{Prop: multiclass logisticbound} is tight with respect to the dimension $n$.

Theorem~\ref{cor: multiclass NB sample complexity} and Theorem~\ref{cor: sample complexity of multiclass lr} show that the $O(n)$ vs. $O(\log(n))$ result~\cite{DBLP:conf/nips/NgJ01} still holds in multiclass cases, which suggests that na\"ive Bayes is possibly better than logistic regression when the sample size is limited. We validate our theory on a mixture of Gaussian distribution, as presented in Figuire~\ref{figures: multiclass simulation}. For a fixed feature dimension $n$, we increase the number of samples $m$ until the two models approach the corresponding asymptotic error, which is tractable in the experiment. Detailed configurations of the experiments and additional results are presented in Appendix~\ref{app: Configurations of Simulation Experiment}.

\subsection{Multiclass $\mathcal{H}$-consistency Framework}
\label{sec: multiclass H-consistency framework}


We now present the general multiclass $\mathcal{H}$-consistency bound framework and prove the explicit bound for the logistic loss in Theorem~\ref{thm: H-consistency bound for log}, which are of independent interest. Similarly to the binary case~\cite{DBLP:conf/icml/AwasthiMM022}, we first introduce the following general multiclass $\mathcal{H}$-consistency bound between any target loss $\ell_2$ and surrogate loss $\ell_1$.
\begin{prop}[Distribution-dependent convex bound, proof in Appendix~\ref{proof: Thm: Distribution-dependent convex bound}]
\label{Thm: Distribution-dependent convex bound}
For a fixed distribution, if there exists a convex function $g: \mathbb{R}_+ \to \mathbb{R}$ with $g(0) \ge 0$ and $\epsilon \ge 0$, and the following holds for any $\boldsymbol{h} \in \mathcal{H}$ and $\vx \in \mathcal{X}$:
\begin{equation}
\label{eqn: Thm: Distribution-dependent convex bound 1}
    g(\langle \Delta \mathscr{C}_{\ell_2, \mathcal{H}}(\boldsymbol{h}, \vx) \rangle_\epsilon) \leq \Delta \mathscr{C}_{\ell_1, \mathcal{H}}(\boldsymbol{h}, \vx).
\end{equation}
Then it holds for all $\boldsymbol{h} \in \mathcal{H}$ that
\begin{align}
    &g(R_{\ell_2}(\boldsymbol{h}) - R^*_{\ell_2, \mathcal{H}} + M_{\ell_2, \mathcal{H}}) \nonumber \\
    &\leq R_{\ell_1}(\boldsymbol{h}) - R^*_{\ell_1, \mathcal{H}} + M_{\ell_1, \mathcal{H}} + \max(g(0), g(\epsilon)).\label{eq:results-of-dependent}
\end{align}
\end{prop}

We present the concave counterpart of it as Proposition~\ref{Thm: Distribution-dependent concave bound} of Appendix~\ref{sec: Deferred Results: Multiclass H-consistency}. For simplicity, we fix the target loss $\ell_2$ as the zero-one loss in the following. Note that
Proposition~\ref{Thm: Distribution-dependent convex bound} is distribution-dependent while an asymptotically distribution-independent version is necessary for our analysis in Section~\ref{sec: Discriminative vs. Generative: Multiclass Classification}. 
To this end, we introduce a tool called \emph{multiclass $\mathcal{H}$-estimation error transformation}.

\begin{mydef}[Multiclass $\mathcal{H}$-estimation error transformation] \label{def:trans}
The multiclass $\mathcal{H}$-estimation error transformation of a surrogate loss $\ell$ is defined on $t \in [0,1]$ as $\mathcal{J}_{\ell}(t) = \inf_{\hat{y} \in \mathcal{Y}, \vp \in \mathcal{P}_{\hat{y}}(t), \vx \in \mathcal{X}, \boldsymbol{h} \in \mathcal{H}_{\hat{y}}(\vx) }\Delta \mathscr{C}_{\ell, \mathcal{H}}(\boldsymbol{h}, \vx, \vp)$. Here $\mathcal{H}_{\hat{y}}(\vx) \coloneqq \{\boldsymbol{h} \in \mathcal{H}: \mathop{\mathrm{argmax}}_{y  \in \mathcal{Y}} h_y(\vx) = \hat{y}\}$ is a collection of hypotheses that predicts $\vx$ as class $\hat{y}$. $\mathcal{P}_{\hat{y}}(t) \coloneqq \{\vp \in \Delta_K: \max_y p_y - p_{\hat{y}} = t\}$ is a subset of $K$-dimensional simplex indexed by classes and the gap between the max component and class-indexed component of $\vp$.
\end{mydef}

$\mathcal{J}_{\ell}(t)$ in Defition~\ref{def:trans} is carefully derived such that 
plugging it to the right-hand side of Eq.~(\ref{eqn: Thm: Distribution-dependent convex bound 1}) provides a sufficient condition such that Eq.~(\ref{eq:results-of-dependent}) holds for any $\vh, \vx$, and $\vp$ (i.e., distribution-independent). It is worth noting that the condition is actually necessary as well under further assumptions, as presented later in Theorem~\ref{thm:tightness}. Defition~\ref{def:trans} generalizes the binary freamwork~\cite{DBLP:conf/icml/AwasthiMM022} by optimizing $\vp$ in a collection of subsets $\mathcal{P}_{\hat{y}}(t)$ to handle multiclass cases. Built upon Defition~\ref{def:trans}, we establish the multiclass distribution-independent bound for zero-one loss as follows. 

\begin{theorem}
[Distribution-independent convex $\ell_{0-1}$ bound, proof in Appendix~\ref{proof: cor: Distribution-independent convex Psi bound}]
\label{cor: Distribution-independent convex Psi bound}
Suppose that $\mathcal{H}$ satisfies that $\{\mathop{\mathrm{argmax}}_{y \in \mathcal{Y}} h_y(\vx) : \boldsymbol{h} \in \mathcal{H}\} = \{1, \dots, K\}$ for any $\vx \in \mathcal{X}$. If there exists a convex function $g: \mathbb{R}_+ \to \mathbb{R}$ with $g(0) = 0$ and $g(t) \leq \mathcal{J}_{\ell}(t)$. Then it holds for any $\boldsymbol{h} \in \mathcal{H}$ and any distribution $\mathcal{D}$ that
\begin{equation*}
    g(R_{\ell_{0-1}}(\vh) - R^*_{\ell_{0-1}, \mathcal{H}} + M_{\ell_{0-1}, \mathcal{H}}) \leq R_{\ell}(\vh) - R^*_{\ell, \mathcal{H}} + M_{\ell, \mathcal{H}}.
\end{equation*}
\end{theorem}

\begin{figure}[t]
\centering
\includegraphics[width=1. \columnwidth]{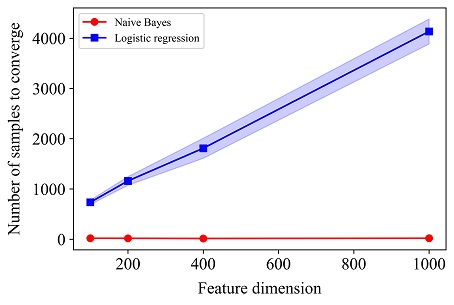}
\caption{Multiclass ($K = 5$) simulation results. Empirically, logistic regression and na\"ive Bayes require $O(n)$ and $O(\log n)$  samples to approach the corresponding asymptotic error respectively. Error bars show the variance estimated by 5 runs.}
\label{figures: multiclass simulation}
\end{figure}

\begin{table*}[t!]
\centering
\caption{Analysis of assumptions on CIFAR10 training dataset.}
\vskip 0.15in
\label{tab: assumptions}
\begin{tabular}{lccccc} 
\toprule
Method   & Backbone & Pre-training data & $\rho_0$ & $\beta$ & $\alpha$  \\ 
\midrule
ViT~\cite{dosovitskiy2020image}      & ViT-B/16 & Image-label  & 2.80E-3                   & 0.004                    & 690                        \\
ResNet~\cite{resnet}   & ResNet50 & Image-label  & 1.70E-3                   & 0.06                     & 11516                      \\
CLIP~\cite{CLIP}     & ResNet50 & Image-text   & 4.78E-3                   & 0.203                    & 6383                       \\
MoCov2~\cite{DBLP:journals/corr/abs-2003-04297MocoV2}   & ResNet50 & Image        & 5.03E-5                   & 0.005                    & 26640                      \\
SimCLRv2~\cite{DBLP:conf/nips/simclrv2} & ResNet50 & Image        & 3.74E-5                   & 0.01                     & 2490                       \\
MAE~\cite{DBLP:conf/cvpr/HeCXLDG22MAE}      & ViT-B/16 & Image        & 6.37E-3                   & 0.032                    & 6919                       \\
SimMIM~\cite{DBLP:conf/cvpr/simmim}   & ViT-B/16 & Image        & 7.86E-3                   & 0.002                    & 5201                       \\
\bottomrule
\end{tabular}
\end{table*}

We present the concave counterpart of it as Theorem~\ref{cor: Distribution-independent concave bound} in Appendix~\ref{sec: Deferred Results: Multiclass H-consistency}. This theorem holds for any hypothesis set $\mathcal{H}$ that can divide any sample $\vx$ into any category, including the linear hypothesis set and hypotheses of neural network. Notably, our multiclass $\mathcal{H}$-consistency result degenerates to the binary one exactly~\cite{DBLP:conf/icml/AwasthiMM022} with $K=2$. In addition, we note that if $\mathcal{J}_{\ell}(t)$ is convex and $\mathcal{J}_{\ell}(0) = 0$, then $\mathcal{J}_{\ell}$ satisfies the condition of $g$ in Theorem~\ref{cor: Distribution-independent convex Psi bound}. In fact, it leads to the tightest multiclass $\mathcal{H}$-consistency bound.
\begin{theorem}[Tightness, proof in Appendix~\ref{Proofs of thm:tightness}]
\label{thm:tightness}
If $\mathcal{J}_{\ell}(t)$ is convex with $\mathcal{J}_{\ell}(0)=0$, then for any $t\in[0,1]$ and $\delta>0$, there exist a distribution $\mathcal{D}$ and a hypothesis $h\in \mathcal{H}$ such that $R_{\ell_{0-1}}(\boldsymbol{h}) - R^*_{\ell_{0-1}, \mathcal{H}} + M_{\ell_{0-1}, \mathcal{H}} = t$ and $\mathcal{J}_{\ell}(t) \leq R_{\ell}(\boldsymbol{h}) - R^*_{\ell, \mathcal{H}} + M_{\ell, \mathcal{H}} \leq \mathcal{J}_{\ell}(t) + \delta$.
\end{theorem}

To establish our main result in Section~\ref{sec: Discriminative vs. Generative: Multiclass Classification}, we have presented an asymptotically distribution-independent multiclass $\mathcal{H}$-consistency bound for the logistic loss in an explicit form in Theorem~\ref{thm: H-consistency bound for log}. We mention that the proof of Theorem~\ref{thm: H-consistency bound for log} is nontrivial because $\mathcal{J}_{\ell}(t)$ in the multiclass case involves a much more complex optimization problem than that in the binary case~\cite{DBLP:conf/icml/AwasthiMM022}.

The proposed framework is not limited to the linear hypothesis class and the logistic loss. In particular, we present a similar result for the hypothesis class of one-hidden-layer neural networks in Theorem~\ref{thm: H-consistency bound for log neural network} of Appendix~\ref{sec: Deferred Results: Multiclass H-consistency}. Besides, the general bound in Theorem~\ref{cor: Distribution-independent convex Psi bound} and the proof idea of Theorem~\ref{thm: H-consistency bound for log} are applicable to hinge loss, exponential loss, $\rho$-margin loss, and so on, which are left for future work. Furthermore, the analysis idea can be used to obtain multiclass Bayes consistency bounds by setting $\mathcal{H}$ to $\mathcal{H}_{all}$.

\section{Implications in Deep Learning}
\label{sec: implications}
In this section, we discuss the implications of our theoretical results in the linear evaluation of pre-trained deep neural networks. 
First, as presented in Section~\ref{sec: Validate the assumptions}, we empirically analyze the main assumptions of our theory in various deep vision models~\cite{dosovitskiy2020image, resnet, CLIP, DBLP:journals/corr/abs-2003-04297MocoV2, DBLP:conf/nips/simclrv2, DBLP:conf/cvpr/HeCXLDG22MAE, DBLP:conf/cvpr/simmim}. Second, we systematically compare logistic regression and na\"ive Bayes on the CIFAR10 and CIFAR100 datasets~\cite{cifar} with various models and sample sizes in Section~\ref{sec: Deep Learning Results}.  Na\"ive Bayes always converges much faster, which agrees with our theory. The ``two regimes'' phenomenon~\cite{DBLP:conf/nips/NgJ01} almost happens with models pre-trained in a supervised manner~\cite{dosovitskiy2020image, resnet}, which is analyzed in detail in Section~\ref{sec: On the difference between supervised pertaining and self-supervised}. Details of experiments can be found in Appendix~\ref{app: details of DL Experiments}.

\subsection{Analyzing the Assumptions}
\label{sec: Validate the assumptions}
We empirically analyze and discuss the main assumptions made in Section~\ref{sec: Theory} on the CIFAR10 dataset. The results are summarized in Table~\ref{tab: assumptions}. We emphasize that the concrete values of the quantities in the table won't affect the asymptotic analyses in Section~\ref{sec: Theory}, i.e., $O(\log n)$ results for na\"ive Bayes, but may affect its performance given a fixed data size.


We consider linear evaluation for transfer learning on top of pre-trained models, whose parameters are frozen. Therefore, it is valid to assume that the features extracted on the target dataset satisfy the $i.i.d.$ assumption.

\subsubsection{Assumption~\ref{Assumption: p(y=k)} and~\ref{Assumption: parmeters bounded}}
Assumption~\ref{Assumption: p(y=k)} holds naturally because the CIFAR10 dataset is class-balanced. For Assumption~\ref{Assumption: parmeters bounded}, we calculate the $\hat{\sigma_i}^2$ for each dimension of the training representations as approximations for $\sigma_i^2$. We present $\rho_0 = \min(\min_i \hat{\sigma_i}^2, \frac{1}{10})$ in Table~\ref{tab: assumptions}, and Figure~\ref{figures: sigmas} in Appendix~\ref{app: Additional Results of Validating the Assumptions} plots the histogram of $\hat{\sigma_i}^2$. Assumption~\ref{Assumption: parmeters bounded}
holds for all models.

\subsubsection{Assumption~\ref{Assumption: multiclass KL} and~\ref{Assumption: multiclass likelihood ratio var}}
It is hard to directly validate the two assumptions in practice. Nevertheless, we estimate $\beta_{k_1,k_2,k}$ and $\alpha_{k_1,k_2,k}$ for all $k_1, k_2 (k_1 \ne k_2)$ and $k \in \mathcal{Y}$ in different models for a comparison. We note that $\beta_{k_1,k_2,k} = \zeta_{k_1,k_2,k}$ in our experiments, because the CIFAR10 dataset is class-balanced. We report the estimated $\beta = \zeta = \min_{k1,k_2,k} \vert\beta_{k_1,k_2,k}\vert$ and $\alpha = \max_{k1,k_2,k} \alpha_{k_1,k_2,k}$ in Table~\ref{tab: assumptions}. We also present the histograms of $\vert\beta_{k_1,k_2,k}\vert$ and $\alpha_{k_1,k_2,k}$  in Figure~\ref{figures: beta} and Figure~\ref{figures: alpha} of Appendix~\ref{app: Additional Results of Validating the Assumptions}, respectively.

\subsubsection{Assumption~\ref{Assumption: Optimal classifier has finite empirical loss}}

Assumption~\ref{Assumption: Optimal classifier has finite empirical loss} is hard to validate in practice because the Bayes-optimal classifier is unknown. However, recent theoretical results in prior works~\cite{DBLP:conf/icml/SaunshiPAKK19, DBLP:conf/nips/LeeLSZ21, DBLP:conf/alt/ToshK021, DBLP:conf/nips/HaoChenWGM21} suggest that it holds  when the number of samples for pre-training is sufficiently large.


\subsection{Empirical Results in Deep Learning}
\label{sec: Deep Learning Results}

We systematically compare logistic regression and naïve Bayes on the CIFAR10 and CIFAR100 datasets in various models, which are trained on image-label pairs~\cite{dosovitskiy2020image, resnet}, image-text pairs~\cite{CLIP}, or pure images~\cite{DBLP:journals/corr/abs-2003-04297MocoV2, DBLP:conf/nips/simclrv2,DBLP:conf/cvpr/HeCXLDG22MAE, DBLP:conf/cvpr/simmim}. 

For a fair comparison, we keep the linear evaluation setting in~\cite{CLIP} throughout the experiments. Specially, we train the logistic regression using scikit-learn's~\cite{scikit-learn} L-BFGS implementation, with a maximum of 1000 iterations. We adjust the weight of $\ell_2$ regularization of logistic regression carefully to reproduce the results reported in~\cite{CLIP} on both datasets with full training data. We then adjust the number of training samples $m$ gradually. For each $m$, we obtain training  samples randomly 5 times and record the mean test error of two models. 

We plot the convergence curves in all settings in Appendix~\ref{app: Additional Deep Learning Results}, which are linked in Table~\ref{tab: visual results sum}. Notably, na\"ive Bayes approaches its asymptotic error much faster than logistic regression in all settings, like that presented in Figure~\ref{figures: resnet cifar100 long}, which is consistent with our theoretical results.

\begin{figure}[t!]
\centering  
\includegraphics[width=0.9\columnwidth]{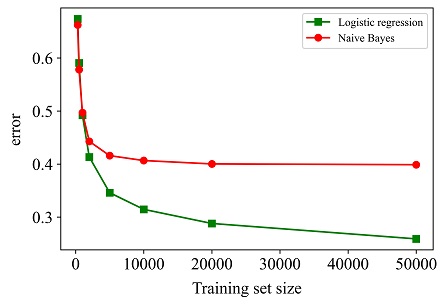}
\caption{Comparison between na\"ive Bayes and logistic regression with the features extracted by ResNet on the CIFAR100 dataset. Na\"ive Bayes approaches its asymptotic error much
faster.} 
\label{figures: resnet cifar100 long}  
\end{figure}

\begin{figure}[t]
\centering  
\includegraphics[width=0.9\columnwidth]{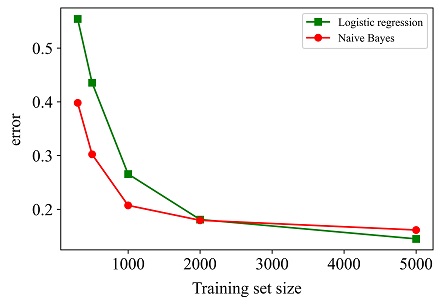}
\caption{Comparison between na\"ive Bayes and logistic regression with the features extracted by ViT on the CIFAR100 dataset. The ``two regimes'' phenomenon is observed.} 
\label{figures: vit cifar100 short}  
\end{figure}

\begin{table}[t]
\centering
\caption{Convergence comparison between multiclass logistic regression and na\"ive Bayes. ``NB faster'' means na\"ive Bayes approaches its asymptotic error faster.}
\vskip 0.15in
\label{tab: visual results sum}
\begin{tabular}{lccc}
\toprule
Method   & Visual results & \multicolumn{2}{c}{NB faster/ Two regimes}  \\
         &                & CIFAR10 & CIFAR100              \\
\midrule
ViT      & Figure~\ref{figures: vit}              & $\surd$ / $\surd$     & $\surd$ / $\surd$                   \\
ResNet   & Figure~\ref{figures: resnet}              & $\surd$ / $\surd$     & $\surd$ / $\surd$                   \\
CLIP     & Figure~\ref{figures: clip}              & $\surd$ / $\surd$     & $\surd$ / $\surd$                   \\
MoCov2   & Figure~\ref{figures: mocov2}              & $\surd$ / $\times$      & $\surd$ / $\times$                   \\
SimCLRv2 & Figure~\ref{figures: simclrv2}              & $\surd$ / $\times$      & $\surd$ / $\surd$                   \\
MAE      & Figure~\ref{figures: mae}              & $\surd$ / $\surd$     &$\surd$ / $\times$                    \\
SimMIM   & Figure~\ref{figures: simmim}              & $\surd$ / $\times$      & $\surd$ / $\times$  \\
\bottomrule
\end{tabular}
\end{table}

\subsection{On the ``Two Regimes'' Phenomenon} 
\label{sec: On the difference between supervised pertaining and self-supervised}





\citet{DBLP:conf/nips/NgJ01} suggests that there can often be two regimes of performance between na\"ive Bayes and logistic regression, that is, though logistic regression enjoys lower asymptotic error, na\"ive Bayes performs better with smaller
training sets because of its fast convergence rate. They observed this phenomenon on many datasets from the UCI Machine Learning repository~\cite{UCI}. These classical datasets are small and the features are mostly low-dimensional. However, nowadays, people prefer to obtain representations by using deep neural networks pre-trained by massive data. The occurrence of the ``two regimes'' phenomenon in this new setting has not been investigated yet.

We summarize the occurrence of the ``two regimes'' phenomenon in Table~\ref{tab: visual results sum}. The ``two regimes'' phenomenon occurs in half of our experiments, which suggests that na\"ive Bayes still shows promise when the training data is limited. We present a typical case in Figure~\ref{figures: vit cifar100 short} and see Appendix~\ref{app: Additional Deep Learning Results} for complete results. Interestingly, the ``two regimes'' phenomenon almost happens when the deep vision model is pre-trained in a supervised manner (ViT, ResNet, and CLIP), which suggests a distinction between representations learned by supervised learning and self-supervised learning. 

We conjecture that representations learned by supervised methods could have some better properties to make na\"ive Bayes converges faster than that learned by self-supervised methods. As validated in Section~\ref{sec: Deep Learning Results}, though our theory could only prove the fast convergence rate of na\"ive Bayes, it does help us to understand this distinction to some extent. Combining the values presented in Table~\ref{tab: assumptions}, we can get some preliminary results. 

\emph{Representations learned by supervised methods could be more robust for each dimension.} As shown in Table~\ref{tab: assumptions}, features learned by supervised methods (ViT, ResNet, CLIP) tend to have larger $\rho_0$. In other words, these representations tend to have larger in-class variance $\sigma_i^2$ than others. Intuitively, it suggests that data in each dimension could be more robust to relieve the over-fitting and boost na\"ive Bayes learning better in the few-shot case. Besides, according to Eq.~(\ref{eq: discrete delata a}-\ref{eq: conti delata a}) in Appendix~\ref{proof: Proof of Theorem Thm: multiclass generalization bound} and the derivation in Appendix~\ref{proof: Proof of Corollary cor: multiclass NB sample complexity}, a larger $\rho_0$ implies faster convergence in a $1/\rho_0^2$ order, which explains it in a certain sense.

\emph{Representations learned by supervised methods could be more separable between different categories.} From Table~\ref{tab: assumptions}, representations learned by supervised methods (ResNet, CLIP) are inclined to have larger $\beta$ than others. Namely, there exists more distinction between the distributions of samples in different classes, which are easier to predict. In addition, by our derivation in Appendix~\ref{proof: Proof of Corollary cor: multiclass NB sample complexity}, a larger $\beta$ implies faster convergence in a $1/\beta^2$ order, which agrees with our observation.




\section{Related Work}
\subsection{Deep Representative Learning}

Deep representation learning aims to learn representations on the raw unlabeled data and transfer them to the downstream tasks. It has made remarkable progress in various machine learning fields~\cite{ren2015faster,he2017mask,chen2020generative,DBLP:conf/cvpr/He0WXG20Moco, DBLP:conf/icml/ChenK0H20SimCLR, DBLP:conf/cvpr/ChenH21SimSiam,grill2020bootstrap,DBLP:conf/cvpr/HeCXLDG22MAE, DBLP:conf/cvpr/simmim, DBLP:conf/naacl/DevlinCLT19Bert, DBLP:conf/nips/BrownMRSKDNSSAA20GPT3,raffel2020exploring}. 
In particular, the promise of \emph{linear evaluation}~\cite{chen2020generative,DBLP:conf/cvpr/He0WXG20Moco, DBLP:conf/icml/ChenK0H20SimCLR,grill2020bootstrap,CLIP} suggests that representations extracted by pre-trained models are near to linear separable. 
Besides, the performance of such representations in linear evaluation is 
guaranteed in recent theoretical works~\cite{DBLP:conf/icml/SaunshiPAKK19, DBLP:conf/nips/LeeLSZ21, DBLP:conf/alt/ToshK021, DBLP:conf/nips/HaoChenWGM21}.  
All of these empirical and theoretical works encourage us to rethink the role of linear classifiers.

\subsection{Discriminative vs. Generative Learning}
Comparing discriminative with generative classifiers has long been an interesting topic~\cite{efron1975efficiency, DBLP:conf/kdd/RubinsteinH97, DBLP:conf/nips/NgJ01}. \citet{efron1975efficiency} compared the logistic regression and normal discriminant analysis and claimed that the latter is only slightly more efficient. \citet{DBLP:conf/nips/NgJ01} simplified the normal discriminant analysis to na\"ive Bayes and concluded that the discriminative model has lower asymptotic error while the generative classifier may approach its higher asymptotic error much faster. 
\citet{DBLP:conf/nips/NgJ01} assume that one can directly optimize on zero-one loss. Instead, we weaken the assumption and introduce the theoretical tools from $\mathcal{H}$-consistency to obtain more reliable results.

\subsection{$\mathcal{H}$-consistency}
Most machine learning algorithms depend on optimizing a surrogate loss function rather than the target loss function. To find the favorable property of surrogate loss, consistency has been studied broadly in the last two decades. Classical Bayes consistency~\cite{DBLP:journals/jmlr/Zhang04a, zhang2004statistical,bartlett2006convexity,tewari2007consistency} analyzes the relationship between the excess error of zero-one loss and that of a surrogate loss. Instead, $\mathcal{H}$-consistency~\cite{long2013consistency} considers the estimation error w.r.t. a hypothesis set $\mathcal{H}$. It includes the classical Bayes consistency as a special case by setting $\mathcal{H}$ to $\mathcal{H}_{all}$. Most recently, \citet{DBLP:conf/icml/AwasthiMM022} proposed a novel and solid framework named $\mathcal{H}$-consistency bounds, which consider the upper bounds on the target estimation error expressed by surrogate estimation error. 

We proposed a novel multiclass $\mathcal{H}$-consistency framework, which includes the framework in~\cite{DBLP:conf/icml/AwasthiMM022} as a special case. We notice that the independent work of~\cite{awasthimulti} also proposed a multiclass $\mathcal{H}$-consistency framework from the same general theorem (Proposition~\ref{Thm: Distribution-dependent convex bound}). We highlight the following comparison that distinguishes our work. 
First, the proof ideas are totally different. In particular, we directly generalize the binary framework in~\cite{DBLP:conf/icml/AwasthiMM022} to the multiclass case in Theorem~\ref{cor: Distribution-independent convex Psi bound}, which is general and tight (Theorem~\ref{thm:tightness}). In contrast, \citet{awasthimulti} argues that generalizing the binary framework is nontrivial and instead provides a case-by-case analysis for different losses, which does not enjoy the tightness guarantee. Second, we provide an explicit bound for logistic loss (Theorem~\ref{thm: H-consistency bound for log}), which is necessary for our subsequent analysis, while it is unclear how to derive such a bound by the prior work~\cite{awasthimulti}. 

\textbf{Concurrent work.} The concurrent and independent work of~\citet{DBLP:journals/corr/maoanqi} also obtains $\mathcal{H}$-consistency bounds of the multiclass logistic loss under a little stronger assumption. The multiclass $\mathcal{H}$-estimation error transformation $\mathcal{J}_{\ell}(t)$ derived by them (Theorem 1 of~\citet{DBLP:journals/corr/maoanqi}) is actually the same as ours in Theorem~\ref{thm: H-consistency bound for log}, and their bounds also enjoy the tightness guarantee. However, they assume that the hypothesis set $\mathcal{H}$ is complete, that is, $\{h_y(\vx): \boldsymbol{h} \in \mathcal{H}\} = \mathbb{R}$ for any $(\vx, y) \in \mathcal{X} \times \mathcal{Y}$, which does not hold for bounded linear hypotheses ($W, B < +\infty$) considered by this paper.



\section{Conclusion}
We revisit the classical topic of discriminative vs. generative classifiers~\cite{DBLP:conf/nips/NgJ01}. Specially, we weaken the assumption in the previous work and extend the analysis to multiclass cases. As result, under some assumptions, we prove that multiclass na\"ive Bayes requires $O(\log n)$ samples to approach its asymptotic error while the logistic regression needs $O(n)$ samples. Technically, we proposed a multiclass $\mathcal{H}$-consistency framework, which is of independent interest. Experiments with various pre-trained deep vision models verify our theory and show the potential of the generative linear head in the few-shot cases. Finally, our experiments suggest differences between representations learned by supervised and self-supervised methods.

\textbf{Social Impact:} This is mainly theoretical work and we do not see a direct social impact of our theory. The experiments on Na\"ive Bayes may benefit applications with a few training data such as medical analysis.

\section*{Acknowledgements}

This work was supported by NSF of China (NO. 62076145, 62206159); Beijing Outstanding Young Scientist Program (NO. BJJWZYJH012019100020098); Shandong Provincial Natural Science
Foundation (NO. ZR2022QF117); Major Innovation \& Planning Interdisciplinary Platform for the ``Double-First Class" Initiative, Renmin University of China; the Fundamental Research Funds for the Central Universities, and the Research Funds of Renmin University of China (22XNKJ13); the Fundamental Research Funds of Shandong University. C. Li was also sponsored by Beijing Nova Program.

\bibliography{ref}
\bibliographystyle{icml2023}

\newpage
\appendix
\onecolumn

\renewcommand{\contentsname}{Contents of Appendix}
\tableofcontents
\addtocontents{toc}{\protect\setcounter{tocdepth}{3}} 
\clearpage

\section{Detailed Notations and Definitions}
\label{notations}
 Let lower, boldface lower and capital case letters denote scalers (e.g., a), vectors (e.g., $\boldsymbol{a}$), and matrices (e.g., $\boldsymbol{A}$) respectively. For a matrix $\boldsymbol{A}$, $\boldsymbol{A}_i$ and $A_{ij}$ denote its $i$-th row and $(i, j)$-th element. For a vector $\boldsymbol{a}$, $a_i$ denotes its $i$-th element. Similarly, for a vector function $\boldsymbol{f}$, $f_i(\vx)$ denotes the $i$-th element of $\boldsymbol{f}(\vx)$. Let $\mathcal{X}$ denote the domain set and $\mathcal{Y}$ denote the label set. For simplicity, we assume $\mathcal{X} = \{0,1\}^n$ when inputs are discrete and $\mathcal{X} = [0,1]^n$ otherwise, where $n$ is the feature dimension. Let $\mathcal{Y} = \{0, 1\}$ be the binary label space and $\mathcal{Y} = \{1,\dots,K\}$ be the multiclass label space, where $K$ is the number of classes. $\mathcal{D}$ denotes the distribution on $\mathcal{X} \times \mathcal{Y}$ and $\mathcal{P}$ denotes set of distribution. We denote the KL Divergence between two distributions $p$ and $q$ by $D(p \Vert q)$. We use $\mathbb{E}$ and $\mathbb{V}$ to represent expectation and variance, respectively.

For the binary case, let $\mathcal{H}$ be a hypothesis set of functions mapping from $\mathcal{X}$ to $\mathbb{R}$. The prediction associated by a hypothesis $h \in \mathcal{H}$ and $\vx \in \mathcal{X}$ is $\sign(h(x))$. In this paper, we mainly focus on the family of constrained binary linear hypotheses $\mathcal{H}_{lin} = \{x \to \vw^Tx + b: \Vert \vw \Vert_2 \leq W, \vert b\vert \leq B\}$, where $W, B \in \mathbb{R}^+$. The generalization error and minimal generalization error of a hypothesis $h$ w.r.t. the loss function $\ell: \mathbb{R} \times \mathcal{Y} \rightarrow \mathbb{R}$ are defined as $R_{\ell}(h) = \mathbb{E}_{(\vx,y) \sim \mathcal{D}} [\ell(h(\vx), y))]$ and $R_{\ell, \mathcal{H}}^* = \inf_{h \in \mathcal{H}} R_{\ell}(h)$, where $\mathcal{H}$ is a hypothesis set and $\mathcal{D}$ is data distribution. We denote the empirical generalization error by $\hat{R}_{\ell}(h)$. Furthermore, given a family of functions $\mathcal{G}$ mapping from $\mathcal{Z}$ to $\mathbb{R}$, the empirical Rademacher complexity of $\mathcal{G}$ for a sample $S = (z_1, \dots, z_m)$ is defined by $\hat{\mathcal{R}}_S(\mathcal{G}) = \mathbb{E}_{ \sigma}[\frac{1}{m} \sup_{g \in \mathcal{G}} \sum_{i=1}^m \sigma_i g(z_i)]$, where $\sigma = (\sigma_1, \dots, \sigma_m)$ is a vector of $i.i.d.$ independent uniform random variables taking values in $\{-1,+1\}$. The Rademacher complexity of $\mathcal{G}$ is defined as $\mathcal{R}_m(\mathcal{G}) = \mathbb{E}_S[\hat{\mathcal{R}}_S(\mathcal{G})]$.

Notations listed in the following will be useful to analyze the $\mathcal{H}$-consistency bounds. For binary label space, let $\eta(x)$ denote the conditional distribution $\mathbb{P}(Y = 1 \vert X = x)$ and $\Delta \eta(x)$ the $\eta(x) - \frac{1}{2}$. We rewrite the generalization error as $R_{\ell}(h) = \mathbb{E}_{\vx} [\mathscr{C}_{\ell}(h, x)]$, where $\mathscr{C}_{\ell}(h, x) = \eta(x)\ell(h, (x, 1))  + (1-\eta(x))\ell(h, (x, 0))$ is called as conditional risk. We can also define the minimal conditional risk as $\mathscr{C}_{\ell, \mathcal{H}}^*(x) = \inf_{h \in \mathcal{H}} \mathscr{C}_{\ell}(h, x)$. We use the shorthand for the gap $\Delta \mathscr{C}_{\ell, \mathcal{H}}(h, x) = \mathscr{C}_{\ell}(h, x) - \mathscr{C}_{\ell, \mathcal{H}}^*(x)$ and conditional $\epsilon$-regret of $\ell$ $\langle \Delta \mathscr{C}_{\ell, \mathcal{H}}(h, x) \rangle_\epsilon = \Delta \mathscr{C}_{\ell, \mathcal{H}}(h, x) \mathbbm{1}_{\mathscr{C}_{\ell, \mathcal{H}}(h, x) > \epsilon}$. For any $t \in [0,1]$, we also define $\mathscr{C}_{\ell}(h, x, t) = t\ell(h, (x, 1))  + (1-t)\ell(h, (x, 0))$ and $\Delta \mathscr{C}_{\ell, \mathcal{H}}(h, x, t) = \mathscr{C}_{\ell}(h, x, t) - \inf_{h \in \mathcal{H}} \mathscr{C}_{\ell}(h, x, t)$. It is worthwhile to note that a key quantity appears in the article is the $M_{\ell, \mathcal{H}} = R_{\ell, \mathcal{H}}^* - \mathbb{E}_{\vx}(\mathscr{C}_{\ell, \mathcal{H}}^*(x))$, which is hard to estimate.

For the multiclass case, let $\mathcal{H}$ be a hypothesis set of functions mapping from $\mathcal{X} \times \mathcal{Y}$ to $\mathbb{R}^K$. The prediction associated by a hypothesis $\vh \in \mathcal{H}$ and $\vx \in \mathcal{X}$ is $\mathop{\mathrm{argmax}}_{y \in \mathcal{Y}} h_y(\vx)$. In the main paper, we mainly focus on the family of constrained linear hypotheses $\mathcal{H}_{lin} = \{x \to \boldsymbol{h}(x): h_y(x) = \vw_y^Tx + b_y, \Vert \vw_y \Vert_2 \leq W, \vert b_y\vert \leq B, y \in \mathcal{Y}\}$, where $W, B \in \mathbb{R}^+$. We also give $\mathcal{H}$-consistency bound for family of one-hidden-layer neural network hypotheses with ReLU activation function $(\cdot)_+$ $\mathcal{H}_{NN} = \{\vx \to \boldsymbol{\vh}(\vx): h_y(\vx) = \sum_{j=1}^n U_{yj}(\langle \vw_j, \vx\rangle + b)_+\}$, where $\boldsymbol{U} \in \mathbb{R}^{K \times n}$, $\vw_j \in \mathbb{R}^n$ and  $b \in \mathbb{R}$. The generalization error and minimal generalization error of a hypothesis $\vh$ w.r.t. the loss function $\ell: \mathbb{R}^K \times \mathcal{Y} \rightarrow \mathbb{R}$ are defined as $R_{\ell}(\vh) = \mathbb{E}_{(\vx,y) \sim \mathcal{D}} [\ell(\vh(\vx), y)]$ and $R_{\ell, \mathcal{H}}^* = \inf_{\vh \in \mathcal{H}} R_{\ell}(\vh)$, where $\mathcal{H}$ is a hypothesis set and $\mathcal{D}$ is data distribution. We denote by $\vp(\vx)$ the conditional distribution of $y$ when given $\vx$, i.e., $p_y(\vx) = \mathbb{P}(Y=y\vert X=\vx)$. Similarly to the binary classification, we have $\mathscr{C}_{\ell}(\boldsymbol{h}, \vx) = \sum_{y=1}^K p_y(\vx) \ell(\vh(\vx), y)$, $\mathscr{C}_{\ell, \mathcal{H}}^*(\vx) = \inf_{\boldsymbol{h} \in \mathcal{H}} \mathscr{C}_{\ell}(\boldsymbol{h}, \vx)$, $\Delta \mathscr{C}_{\ell, \mathcal{H}}(\boldsymbol{h}, \vx) = \mathscr{C}_{\ell}(\boldsymbol{h}, \vx) - \mathscr{C}_{\ell, \mathcal{H}}^*(\vx)$ and $M_{\ell, \mathcal{H}} = R_{\ell, \mathcal{H}}^* - \mathbb{E}_{\vx}(\mathscr{C}_{\ell, \mathcal{H}}^* (\vx))$. Furthermore, for any $\vp$ in probability simplex $\Delta_K$, we can define $\mathscr{C}_{\ell}(\boldsymbol{h}, \vx, \vp) = \sum_{y=1}^K p_y \ell(\boldsymbol{h}(\vx), y)$ and $\Delta \mathscr{C}_{\ell, \mathcal{H}}(\boldsymbol{h}, \vx, \vp) = \mathscr{C}_{\ell}(\boldsymbol{h}, \vx, \vp) - \inf_{\boldsymbol{h} \in \mathcal{H}} \mathscr{C}_{\ell}(\boldsymbol{h}, \vx, \vp)$.

\section{On Binary Discriminative vs. Generative  Linear Classifiers}
\label{sec: Discriminative vs. Generative: Binary Classification}

In this section, we focus on the binary case and obtain results that are similar to~\cite{DBLP:conf/nips/NgJ01}, under weaker assumptions. Let $h_{Gen, m}$ and $h_{Dis, m}$ be logistic regression and na\"ive Bayes trained with $m$ $i.i.d$ samples, $h_{Gen, \infty}$ and $h_{Dis, \infty}$ be their asymptotic/population versions. Proofs of this section can be found in Appendix~\ref{Proofs of sec: Discriminative vs. Generative: Binary Classification}. 

We will compare the sample complexity of logistic regression with that of na\"ive Bayes. Consider optimizing the practicable logistic loss rather than zero-one loss, the estimation error of the logistic regression can be bounded by making use of the definition of Rademacher complexity from classical statistical learning techniques.

\begin{prop}[Proof in Appendix~\ref{proof: Proposition logisticbound}]
\label{Prop: logisticbound}
With a high probability of at least $1 - \delta_0$, the following holds
\begin{equation*}
\label{Eq: logisticboundO}
    R_{\ell_{log}}(h_{Dis, m}) \leq R_{\ell_{log}}(h_{Dis, \infty}) + O(\sqrt{\frac{n}{m}}).
\end{equation*}
\end{prop}
Theorem~\ref{lemma: binary H consistency bounds} means that we can bound the estimation error of the zero-one loss by the estimation error of the logistic loss, which makes it possible to obtain an upper bound of the sample complexity with respect to zero-one loss.

\begin{theorem}[Proof in Appendix~\ref{proof: Proof of Corollary cor: sample complexity of binary lr}]
\label{cor: sample complexity of binary lr}
Suppose that Assumption~\ref{Assumption: Optimal classifier has finite empirical loss} is valid. Then, it suffices to pick $m = O(n)$ training samples such that $R_{\ell_{0-1}}(h_{Dis, m}) \leq R_{\ell_{0-1}}(h_{Dis, \infty}) +\epsilon_0$  hold with probability $1 - \delta_0$, for any $\epsilon_0 \in [\sqrt{2\nu}, {\frac{e^{B}-1}{e^{B}+ 1}}]$ and $\delta_0 \in (0, 1)$.
\end{theorem}

By further using the Theorem 9.3 in~\cite{shalev2014understanding} and binary $\mathcal{H}$-consistency bound Theorem~\ref{lemma: binary H consistency bounds}, which states that for $n$-dimension logistic regression, it needs at least $\Omega(n)$ training samples to guarantee the estimation error is small enough with high probability, we know the result in Theorem~\ref{cor: sample complexity of binary lr} is tight.

In the rest of this subsection, we will discuss the sample complexity of na\"ive Bayes. The sketch of proofs has been adopted by~\cite{DBLP:conf/nips/NgJ01}. However, their results are somewhat ambiguous and without detailed derivation, which is very important to the extended analysis in Section~\ref{sec: Discriminative vs. Generative: Multiclass Classification} for multiclass classification. Thus, we present the proof for completeness.
\begin{mydef}
\label{Def :G}
We define the $G(\tau)$ which will be useful to bound the generalization error of binary na\"ive Bayes as
\begin{align*}
G(\tau) = \mathbb{P}_{(\vx,y) \sim \mathcal{D}}(\vert \Delta a_{Gen, \infty}(\vx, 1, 0)\vert \leq \tau n).
\end{align*}
\end{mydef}

\begin{theorem}[Proof in Appendix~\ref{proof: thm generalization bound}]
\label{Thm: generalization bound}
Suppose that Assumption~\ref{Assumption: p(y=k)} and~\ref{Assumption: parmeters bounded} hold. Then with probability at least $1-\delta$:
\begin{equation*}
    R_{\ell_{0-1}}(h_{Gen, m}) \leq R_{\ell_{0-1}}(h_{Gen, \infty}) + G\bigl(O(\sqrt{\frac{1}{m} \log(\frac{n}{\delta})})\bigr) + \delta.
\end{equation*}
\end{theorem}
The key quantity in this Theorem is the $G(\tau)$ , which must be small when $\tau$ is
small in order to bound $R_{\ell_{0-1}}(h_{Gen, m}) - R_{\ell_{0-1}}(h_{Gen, \infty})$. This property holds when we introduce the Assumption~\ref{Assumption: binary KL} and~\ref{Assumption: binary likelihood ratio var}.

\begin{assumption}
\label{Assumption: binary KL}
For $k_1, k_2 \in \{0,1\} (k_1 \ne k_2)$, it holds that $\sum_{i=1}^n D(p(x_i \vert y=k_1) \Vert p(x_i \vert y=k_2)) = \beta_{k_1, k_2}n = \Omega(n)$.
\end{assumption}

It means that samples from different classes ($y = 0$ and $y = 1$) should have different distributions on at least $\Omega(1)$ fraction of their features.

\begin{assumption}
\label{Assumption: binary likelihood ratio var}
    For all $k \in \{0,1\}$, it holds that $\mathbb{V}_{\vx}[\sum_{i=1}^n \log \frac{{p}(x_i\vert y=1) }{{p}(x_i\vert y=0) } \vert y = k] = \alpha_{k}n = O(n^r)$, where $r < 2$..
\end{assumption}

\begin{prop}[Proof in Appendix~\ref{proof: Prop: binary ploy -n}]
\label{Prop: binary ploy -n}
Suppose that Assumption~\ref{Assumption: p(y=k)},~\ref{Assumption: binary KL} and~\ref{Assumption: binary likelihood ratio var} hold, then $G(\tau)$ is polynomially small in $n$:
\begin{equation*}
    G(\tau) \leq \frac{\alpha}{(\tau - \zeta)^2 n},
\end{equation*}
where $\alpha = \min_{k} \vert\alpha_{k}\vert = O(n^{r-1})$, $\mathbb{E}_{\vx}[\Delta a_{Gen, \infty}(\vx, 1, 0)\vert y=k] = \zeta_{k} n$, $\zeta = \min_{k} \vert\zeta_{k}\vert = \Omega(1)$ and $\tau < \zeta$.
\end{prop}

Indeed, if the na\"ive Bayes assumption really holds, that is, feature values are independent given the label, we can obtain a much stronger guarantee for $G(\tau)$.

\begin{prop}[Proof in Appendix~\ref{proof: Prop: binary exp -n}]
\label{Prop: binary exp -n}
Suppose that Assumption~\ref{Assumption: p(y=k)},~\ref{Assumption: parmeters bounded},~\ref{Assumption: binary KL} and the na\"ive Bayes assumption hold, then $G(\tau)$ is exponentially small in $n$, that is,
\begin{equation*}
    G(\tau) \leq \exp{-O((\tau - \beta)^2n)},
\end{equation*}
where $\mathbb{E}_{\vx}[\Delta a_{Gen, \infty}(\vx, 1, 0)\vert y=k] = \zeta_{k} n$, $\zeta = \min_{k} \vert\zeta_{k}\vert = \Omega(1)$ and $\tau < \zeta$.
\end{prop}
Using the results from Theorem~\ref{Thm: generalization bound}, we can obtain the sample complexity of na\"ive Bayes as follows.

\begin{theorem}[Proof in Appendix~\ref{Proof of Corollary cor: binary gen logn}]
\label{cor: binary gen logn}
Suppose that either precondition of Proposition~\ref{Prop: binary ploy -n} or Proposition~\ref{Prop: binary exp -n} holds. Then, it suffices to pick $m = O(\log n)$ training samples such that $R_{\ell_{0-1}}(h_{Gen, m}) \leq R_{\ell_{0-1}}(h_{Gen, \infty}) +\epsilon_0$  hold with probability $1 - \delta_0$, for any $\epsilon_0 \in (0,1)$ and $\delta_0 \in (0, \frac{\epsilon_0}{2}]$.
\end{theorem}

Compare Corollary~\ref{cor: sample complexity of binary lr} with~\ref{cor: binary gen logn}, we revisit the results in~\cite{DBLP:conf/nips/NgJ01}. But we highlight that our results are obtained based on different assumptions and novel $\mathcal{H}$-consistency bound.

\section{Deferred Results}
\label{sec: Deferred Results: Multiclass H-consistency}
Proofs of results in this section can be found in Section~\ref{sec: proof Deferred Results: Multiclass H-consistency}.

\begin{prop}[Distribution-dependent concave bound, proof in~\ref{proof: Theorem Thm: Distribution-dependent concave bound}]
\label{Thm: Distribution-dependent concave bound}
For a fixed distribution, if there exists a concave function $s: \mathbb{R}_+ \to \mathbb{R}$ and $\epsilon \ge 0$ such that the following holds for any $\boldsymbol{h} \in \mathcal{H}$ and $x \in \mathcal{X}$:
\begin{equation*}
    \langle \Delta \mathscr{C}_{\ell_2, \mathcal{H}}(\boldsymbol{h}, \vx) \rangle_\epsilon \leq s(\Delta \mathscr{C}_{\ell_1, \mathcal{H}}(\boldsymbol{h}, \vx)).
\end{equation*}
Then it holds for all $\boldsymbol{h} \in \mathcal{H}$ that
\begin{equation*}
    R_{\ell_2}(\boldsymbol{h}) - R^*_{\ell_2, \mathcal{H}} + M_{\ell_2, \mathcal{H}} \leq s(R_{\ell_1}(\boldsymbol{h}) - R^*_{\ell_1, \mathcal{H}} + M_{\ell_1, \mathcal{H}}) + \epsilon.
\end{equation*}
\end{prop}


\begin{theorem}
[Distribution-independent concave $\ell_{0-1}$ bound, proof in Appendix~\ref{proof: Proof of Theorem cor: Distribution-independent concave bound}]
\label{cor: Distribution-independent concave bound}
Suppose that $\mathcal{H}$ satisfies that $\{\mathop{\mathrm{argmax}}_{y \in \mathcal{Y}} h_y(\vx) : \boldsymbol{h} \in \mathcal{H}\} = \{1, \dots, K\}$ for any $\vx \in \mathcal{X}$. If there exists a non-decreasing concave function $s: \mathbb{R}_+ \to \mathbb{R}_+$ with $t \leq s(\mathcal{J}_{\ell}(t))$. Then it holds for all $\boldsymbol{h} \in \mathcal{H}$ and any distribution $\mathcal{D}$ that
\begin{equation*}
    R_{\ell_{0-1}}(\boldsymbol{h}) - R^*_{\ell_{0-1}, \mathcal{H}} + M_{\ell_{0-1}, \mathcal{H}} \leq s(R_{\ell(\boldsymbol{h})} - R^*_{\ell, \mathcal{H}} + M_{\ell, \mathcal{H}}).
\end{equation*}
\end{theorem}

\begin{theorem}
[Multiclass $\mathcal{H}$-consistency bound for $\ell_{log}$ with one-hidden-layer neural network, proof in Appendix~\ref{Proofs of thm:H-consistency bound for log neural network}]
\label{thm: H-consistency bound for log neural network}
Given family of one-hidden-layer neural network hypotheses with ReLU activation function $(\cdot)_+$ $\mathcal{H}_{NN} = \{\vx \to \boldsymbol{\vh}(\vx): h_y(\vx) = \sum_{j=1}^n U_{yj}(\langle \vw_j, \vx\rangle + b)_+\}$, where $\boldsymbol{U} \in \mathbb{R}^{K \times n}$, $\vw_j \in \mathbb{R}^n$ and  $b \in \mathbb{R}$, then it holds for any distribution that $R_{\ell_{0-1}}(\boldsymbol{h}) - R^*_{\ell_{0-1}, \mathcal{H}_{NN}} + M_{\ell_{0-1}, \mathcal{H}_{NN}} \leq \sqrt{2}(R_{\ell_{log}}(\boldsymbol{h}) - R^*_{\ell_{log}, \mathcal{H}_{NN}} + M_{\ell_{log}, \mathcal{H}_{NN}})^{\frac{1}{2}}$.
\end{theorem}

 \begin{prop}[Proof in Appendix~\ref{proof: Proposition Prop: multiclass exp -n}]
\label{Prop: multiclass exp -n}
Suppose that Assumption~\ref{Assumption: p(y=k)},~\ref{Assumption: parmeters bounded},\ref{Assumption: multiclass KL} and na\"ive Bayes assumption hold, then $\widetilde{G}(\tau)$ is exponentially small in $n$:
\begin{equation*}
    \widetilde{G}(\tau) \leq \exp{-O((\tau - \zeta)^2n)},
\end{equation*}
where $\mathbb{E}_{\vx}[\Delta a_{Gen, \infty}(\vx, k_1, k_2)\vert y=k] = \zeta_{k_1,k_2,k} n$, $\zeta = \min_{k_1, k_2, k} \vert\zeta_{k_1,k_2,k}\vert = \Omega(1)$ and $\tau < \zeta$.
\end{prop}

\section{Proofs of Section~\ref{sec: Discriminative vs. Generative: Multiclass Classification}}
\label{Proofs of sec: Discriminative vs. Generative: Multiclass Classification}

\subsection{Proof of Theorem~\ref{Thm: multiclass generalization bound}}
\label{proof: Proof of Theorem Thm: multiclass generalization bound}
The proof is very similar to the proof of binary case (Theorem~\ref{Thm: generalization bound}). Similarly, there are some lemmas to bound the  $\vert \Delta a_{Gen}(\vx, k_1, k_2) - \Delta a_{Gen,\infty}(\vx, k_1, k_2) \vert$ with high probability. 

\begin{lemma}
\label{lemma: multiclass discrete delta a}
In case of discrete inputs, and suppose that Assumption~\ref{Assumption: parmeters bounded} holds, then with probability at least $1 - \delta$, for every fixed $k_1, k_2$ the following holds:
\begin{align}
\label{eq: discrete delata a}
    \vert \Delta a_{Gen}(\vx, k_1, k_2) - \Delta a_{Gen,\infty}(\vx, k_1, k_2) \vert \leq \frac{4(n+1)}{\rho_0} \sqrt{\frac{1}{\rho_0m} \log(\frac{2(4n+2)}{\delta})} = O\bigl(n\sqrt{\frac{1}{m}\log(\frac{n}{\delta})}\bigr).
\end{align}
\end{lemma}
\begin{proof}
The proof is almost the same as the proof of the binary case (Lemma~\ref{lemma: discrete delta a}). Just replace the label $\{0,1\}$ with $\{k_1, k_2\}$ and notice that $\vert\log \hat{p}(y=k_1) - \log {p}(y=k_1)\vert \leq \epsilon$ no longer implies that $\vert\log \hat{p}(y=k_2) - \log {p}(y=k_2)\vert \leq \epsilon$.
\end{proof}

\begin{lemma}
\label{lemma: multiclass conti delta a}
In case of continuous inputs, and suppose that Assumption~\ref{Assumption: parmeters bounded} holds, then with probability at least $1 - \delta$, the following holds:
\begin{align}
   \vert \Delta a_{Gen}(\vx, k_1, k_2) - \Delta a_{Gen,\infty}(\vx, k_1, k_2) \vert &\leq 4(\frac{n}{3\rho_0}(\frac{4}{\rho^2_0} + \frac{3}{\rho_0} + \sqrt{\frac{2}{\rho_0}}) + \frac{1}{\rho_0}) \sqrt{\frac{1}{\rho_0 m} \log(\frac{2(5n+2)}{\delta})} \label{eq: conti delata a}\\
   &= O\bigl(n\sqrt{\frac{1}{m}\log(\frac{n}{\delta})}\bigr).
\end{align}
\end{lemma}
\begin{proof}
The proof is almost the same as the proof of the binary case (Lemma~\ref{lemma: conti delta a}). Just replace the label $\{0,1\}$ with $\{k_1, k_2\}$ and notice that  $\vert\log \hat{p}(y=k_1) - \log {p}(y=k_1)\vert \leq \epsilon$ no longer implies that $\vert\log \hat{p}(y=k_2) - \log {p}(y=k_2)\vert \leq \epsilon$.
\end{proof}

Based on Lemma~\ref{lemma: multiclass discrete delta a} and~\ref{lemma: multiclass conti delta a}, we are ready to prove Theorem~\ref{Thm: multiclass generalization bound}. 
\begin{proof}
Let $\delta$ and $\epsilon = O\bigl(n\sqrt{\frac{1}{m}\log(\frac{n}{\delta})}\bigr)$ are what claimed in the Lemma~\ref{lemma: multiclass discrete delta a} for discrete case and  Lemma~\ref{lemma: multiclass conti delta a} for the continuous case.We calculate the $\vert R_{\ell_{0-1}}(\vh_{Gen, m}) - R_{\ell_{0-1}}(\vh_{Gen, \infty})\vert$ for multiclass na\"ive Bayes as follows:
\begin{align*}
    &\vert R_{\ell_{0-1}}(\vh_{Gen, m}) - R_{\ell_{0-1}}(\vh_{Gen, \infty})\vert \\
    &= \vert \mathbb{E}_{(\vx,y) \sim \mathcal{D}} [\ell_{0-1}(\vh_{Gen, m}, (\vx,y)) - \ell_{0-1}(\vh_{Gen, \infty}, (\vx,y))]\vert \\
    & \leq \mathbb{E}_{(\vx,y) \sim \mathcal{D}} \vert \ell_{0-1}(\vh_{Gen, m}, (\vx,y)) - \ell_{0-1}(\vh_{Gen, \infty}, (\vx,y))\vert \\
    &=\mathbb{P}_{(\vx,y) \sim \mathcal{D}} (\mathop{\mathrm{argmax}}\limits_{k}{a_{Gen}(\vx, k)} \neq \mathop{\mathrm{argmax}}\limits_{k}{a_{Gen, \infty}(\vx, k)}) \\
    &\leq  \mathbb{P}_{(\vx,y) \sim \mathcal{D}} (\cup_{k_1,k_2} \Delta a_{Gen}(\vx, k_1, k_2) \Delta a_{Gen,\infty}(\vx, k_1, k_2) < 0) \\
    &\leq \sum_{k_1 \ne k_2}\mathbb{P}_{(\vx,y) \sim \mathcal{D}} (\Delta a_{Gen}(\vx, k_1, k_2) \Delta a_{Gen,\infty}(\vx, k_1, k_2) < 0) \\
    &\leq \frac{K(K-1)}{2} \max_{k_1,k_2}\mathbb{P}_{(\vx,y) \sim \mathcal{D}} (\Delta a_{Gen}(\vx, k_1, k_2) \Delta a_{Gen,\infty}(\vx, k_1, k_2) < 0) \\
    &\leq \frac{K(K-1)}{2} \max_{k_1,k_2}\bigl( \mathbb{P} (\Delta a_{Gen}(\vx, k_1, k_2) \Delta a_{Gen,\infty}(\vx, k_1, k_2) < 0 \vert \vert \Delta a_{Gen}(\vx, k_1, k_2) - \Delta a_{Gen,\infty}(\vx, k_1, k_2) \vert \leq \epsilon) + \delta  \bigr)\\
    &\leq \frac{K(K-1)}{2} \max_{k_1,k_2}\bigl( \mathbb{P} (\vert \Delta a_{Gen,\infty}(\vx, k_1, k_2) \vert \leq O\bigl(n\sqrt{\frac{1}{m}\log(\frac{n}{\delta})}\bigr))) + \delta  \bigr)\\
    &=  \frac{K(K-1)}{2} \biggl(\widetilde{G}\bigl(O(\sqrt{\frac{1}{m} \log(\frac{n}{\delta})})\bigr) + \delta\biggr).
\end{align*}
The proof of Theorem~\ref{Thm: multiclass generalization bound} is complete.
\end{proof}

\subsection{Proof of Proposition~\ref{Prop: multiclass ploy -n}}
\label{proof: Proposition Prop: multiclass poly -n}

The following lemma states that the expectation of $\Delta a_{Gen, \infty}(\vx, k_1, k_2)$ condition on $y$ is always large, which is essential to the proof of Proposition~\ref{Prop: multiclass ploy -n}.

\begin{lemma}
\label{Prop: multiclass high E}
Suppose that Assumption~\ref{Assumption: multiclass KL} holds, then for every $k_1, k_2$ and $k \in \mathcal{Y}$, it holds that $\vert \mathbb{E}_{\vx}[\sum_{i=1}^n \log \frac{{p}(x_i\vert y=k_1) }{{p}(x_i\vert y=k_2) } \vert y = k] \vert = \Omega(n)$, which implies that $\vert \mathbb{E}[\Delta a_{Gen, \infty}(\vx, k_1, k_2)\vert y=k]\vert = \Omega(n)$.
\end{lemma}

\begin{proof}
We calculate $\vert \mathbb{E}_{\vx}[\sum_{i=1}^n \log \frac{{p}(x_i\vert y=k_1) }{{p}(x_i\vert y=k_2) } \vert y = k]\vert$ directly:
\begin{align*}
    &\vert \mathbb{E}_{\vx}[\sum_{i=1}^n \log \frac{{p}(x_i\vert y=k_1) }{{p}(x_i\vert y=k_2) } \vert y = k]\vert \\
    & = \vert\sum_{i=1}^n \mathbb{E}_{x_i}[ \log \frac{{p}(x_i\vert y=k_1) }{{p}(x_i\vert y=k_2) }\vert y = k] \vert \\
    & = \vert\sum_{i=1}^n \sum_{x_i}( {p}(x_i\vert y=k)\log \frac{{p}(x_i\vert y=k_1) }{{p}(x_i\vert y=k_2) })\vert \\
    & =\vert \sum_{i=1}^n \sum_{x_i}( {p}(x_i\vert y=k)\log \frac{{p}(x_i\vert y=k) }{{p}(x_i\vert y=k_2) } - {p}(x_i\vert y=k)\log \frac{{p}(x_i\vert y=k) }{{p}(x_i\vert y=k_1) }) \vert\\
    & = \vert\sum_{i=1}^n (D(p(x_i \vert y=k) \Vert p(x_i \vert y=k_2)) - D(p(x_i \vert y=k) \Vert p(x_i \vert y=k_1)))\vert \\
    &= \beta_{k_2, k_1, k}n= \Omega(n). & \text{(Assumption~\ref{Assumption: multiclass KL})}
\end{align*}
Furthermore, we can obtain
\begin{align*}
    &\vert\mathbb{E}_{\vx}[\Delta a_{Gen, \infty}(\vx, k_1, k_2)\vert y=k]\vert\\
    &=  \vert\mathbb{E}_{\vx}[\sum_{i=1}^n \log \frac{{p}(x_i\vert y=k_1) }{{p}(x_i\vert y=k_2) } + \log\frac{{p}(y=k_1) }{{p}(y=k_2) } \vert y = k]\vert \\
    & = \vert\sum_{i=1}^n \mathbb{E}_{x_i}[ \log \frac{{p}(x_i\vert y=k_1) }{{p}(x_i\vert y=k_2) }\vert y = k] +  \log\frac{{p}(y=k_1) }{{p}(y=k_2) }\vert \\
    &\ge \beta_{k_2, k_1, k}n - \vert  \log\frac{{p}(y=k_1) }{{p}(y=k_2)}\vert\\
    & \ge \beta_{k_2, k_1, k}n - \vert  \log\frac{\rho_0}{1 - \rho_0}\vert & (\text{Assumption~\ref{Assumption: p(y=k)}})\\
    &= \Omega(n),
\end{align*}
which implies that $\vert\mathbb{E}_{\vx}[\Delta a_{Gen, \infty}(\vx, k_1, k_2)\vert y=k]\vert = \Omega(n)$. Then the lemma is proved.
\end{proof}

Built upon Lemma~\ref{Prop: multiclass high E}, we prove Theorem~\ref{Prop: binary ploy -n} as follows.
\begin{proof}
For $k_1, k_2$ and $k$ which satisfies $\zeta_{k_1, k_2, k} > 0$, to bound $\mathbb{P}(\vert \Delta a_{Gen, \infty}(\vx, k_1, k_2) \vert \leq \tau n \vert y=k)$ with $\tau \in (0,\zeta)$. we can write:
\begin{align*}
    &\mathbb{P}(\vert \Delta a_{Gen, \infty}(\vx, k_1, k_2) \vert \leq \tau n \vert y=k) \\
    &\leq \mathbb{P}( \Delta a_{Gen, \infty}(\vx, k_1, k_2) \leq \tau n \vert y=k) \\
    &=\mathbb{P}(  \Delta a_{Gen, \infty}(\vx, k_1, k_2) - \zeta_{k_1, k_2, k} n \leq \tau n - \zeta_{k_1, k_2, k} n\vert y=k) \\
    &=\mathbb{P}(\sum_{i=1}^n \log \frac{{p}(x_i\vert y=k_1) }{{p}(x_i\vert y=k_2) } - \mathbb{E}_{\vx}(\sum_{i=1}^n \log \frac{{p}(x_i\vert y=k_1) }{{p}(x_i\vert y=k_2) }) \leq (\tau  - \zeta_{k_1, k_2, k}) n \vert y=k)\\
    &\leq \mathbb{P}(\vert \sum_{i=1}^n \log \frac{{p}(x_i\vert y=k_1) }{{p}(x_i\vert y=k_2) } - \mathbb{E}_{\vx}(\sum_{i=1}^n \log \frac{{p}(x_i\vert y=k_1) }{{p}(x_i\vert y=k_2) })\vert \ge (\zeta_{k_1, k_2, k} - \tau) n \vert y=k)\\
    &\leq \frac{\mathbb{V} [\sum_{i=1}^n \log \frac{{p}(x_i\vert y=k_1) }{{p}(x_i\vert y=k_2) } \vert y = k]}{(\tau -  \zeta_{k_1, k_2, k})^2 n^2} & \text{(Chebyshev inequality)}\\
    & = \frac{\alpha_{k_1, k_2, k} n }{(\tau -  \zeta_{k_1, k_2, k})^2 n^2} & \text{(Assumption~\ref{Assumption: multiclass likelihood ratio var})}\\ 
    &= \frac{\alpha_{k_1, k_2, k} }{(\tau -  \zeta_{k_1, k_2, k})^2 n}.
\end{align*}
Similar to the above discussion, we have $\mathbb{P}(\vert \Delta a_{Gen, \infty}(\vx, k_1, k_2) \vert \leq \tau n \vert y=k) \leq \frac{\alpha_{k_1, k_2, k} }{(\tau - \vert\zeta_{k_1, k_2, k}\vert)^2 n}$ for $k_1, k_2$ and $k$ which satisfies $\zeta_{k_1, k_2, k} < 0$. Finally, we can conclude that:
\begin{align*}
    \widetilde{G}(\tau) &= \max_{k_1,k_2}\sum_{k=1}^K p(y=k)\mathbb{P}(\vert \Delta a_{Gen, \infty}(\vx, k_1, k_2) \vert \leq \tau n \vert y=k)\\
    &\leq \max_{k_1,k_2}\sum_{k=1}^K p(y=k)\frac{\alpha_{k_1, k_2, k} }{(\tau - \vert\zeta_{k_1, k_2, k}\vert)^2 n}\\
    &\leq \max_{k_1,k_2}\frac{\max_{k}\alpha_{k_1, k_2, k} }{(\tau - \min_{k}\vert\zeta_{k_1, k_2, k}\vert)^2 n}\\
    &= \frac{\max_{k_1,k_2,k}\alpha_{k_1, k_2, k} }{(\tau - \min_{k_1,k_2,k}\vert\zeta_{k_1, k_2, k}\vert)^2 n} = \frac{\alpha }{(\tau - \zeta)^2 n}.
\end{align*}
\end{proof}

\subsection{Proof of Theorem~\ref{cor: multiclass NB sample complexity}}
\label{proof: Proof of Corollary cor: multiclass NB sample complexity}

\begin{proof}
In the case that precondition of Proposition~\ref{Prop: multiclass ploy -n} holds, combining Theorem~\ref{Thm: multiclass generalization bound} and Proposition~\ref{Prop: multiclass ploy -n}, we know that there exist positive $c = \Theta(1)$ and large enough $m$ such that when $c \sqrt{\frac{1}{m} \log(\frac{n}{\delta})} < \zeta$, with probability at least $1 - \delta$, we have
\begin{align*}
    R_{\ell_{0-1}}(h_{Gen, m}) &\leq R_{\ell_{0-1}}(h_{Gen, \infty}) + \frac{K(K-1)}{2} \biggl( \frac{\alpha}{(c\sqrt{\frac{1}{m} \log(\frac{n}{\delta})} - \zeta)^2n}  + \delta\biggr)\\
    &\leq R_{\ell_{0-1}}(h_{Gen, \infty}) + \frac{K^2}{2} \biggl( \frac{\alpha}{(c\sqrt{\frac{1}{m} \log(\frac{n}{\delta})} - \zeta)^2n}  + \delta\biggr).
\end{align*}
For fixed $\epsilon_0 \in (0,1)$, the logical relations listed in the following is correct:
\begin{align*}
    &R_{\ell_{0-1}}(h_{Gen, m}) \leq R_{\ell_{0-1}}(h_{Gen, \infty}) + \epsilon_0 \text{ with probability at least 1 - $\delta$} \\
    & \Leftarrow c \sqrt{\frac{1}{m} \log(\frac{n}{\delta})} < \zeta \wedge 0 < \delta < 1 \wedge \frac{K^2}{2} \biggl( \frac{\alpha}{(c\sqrt{\frac{1}{m} \log(\frac{n}{\delta})} - \zeta)^2n}  + \delta\biggr) \leq \epsilon_0\\
    & \Leftrightarrow c \sqrt{\frac{1}{m} \log(\frac{n}{\delta})} < \zeta \wedge  0 < \delta < 1  \wedge \frac{\alpha}{(c\sqrt{\frac{1}{m} \log(\frac{n}{\delta})} - \zeta)^2n} \leq \frac{2\epsilon_0}{K^2} - \delta\\
    & \Leftrightarrow c \sqrt{\frac{1}{m} \log(\frac{n}{\delta})} < \zeta \wedge 0 < \delta < 1 \wedge \frac{2\epsilon_0}{K^2} - \delta > 0 \wedge (c\sqrt{\frac{1}{m} \log(\frac{n}{\delta})} - \zeta)^2 \ge \frac{\alpha}{(\frac{2\epsilon_0}{K^2} - \delta)n}\\
    & \Leftrightarrow c \sqrt{\frac{1}{m} \log(\frac{n}{\delta})} < \zeta \wedge 0 < \delta < \frac{2\epsilon_0}{K^2} \wedge (c \sqrt{\frac{1}{m} \log(\frac{n}{\delta})} - \zeta)^2 \ge \frac{\alpha}{(\frac{2\epsilon_0}{K^2} - \delta)n}\\
    & \Leftrightarrow 0 < \delta < \frac{2\epsilon_0}{K^2} \wedge \zeta - c \sqrt{\frac{1}{m} \log(\frac{n}{\delta})} \ge \sqrt{\frac{\alpha}{(\frac{2\epsilon_0}{K^2} - \delta)n}}\\
    & \Leftrightarrow 0 < \delta < \frac{2\epsilon_0}{K^2} \wedge \zeta - \sqrt{\frac{\alpha}{(\frac{2\epsilon_0}{K^2} - \delta)n}} > 0 \wedge (\zeta - \sqrt{\frac{\alpha}{(\frac{2\epsilon_0}{K^2} - \delta)n}})^2 \ge c^2 \frac{1}{m} \log(\frac{n}{\delta})\\
    & \Leftarrow 0 < \delta < \frac{2\epsilon_0}{K^2} - \frac{\alpha}{\zeta^2 n} \wedge \frac{2\epsilon_0}{K^2} - \frac{\alpha}{\zeta^2 n} > 0 \wedge m  \ge \frac{c^2}{(\zeta - \sqrt{\frac{\alpha}{(\frac{2\epsilon_0}{K^2} - \delta)n}})^2} \log(\frac{n}{\delta})\\
    & \Leftarrow 0 < \delta \leq \frac{\epsilon_0}{K^2} \wedge  \frac{2\epsilon_0}{K^2} - \frac{\alpha}{\zeta^2 n} > \frac{\epsilon_0}{K^2} \wedge m  \ge \frac{c^2}{(\zeta - \sqrt{\frac{\alpha}{(\frac{2\epsilon_0}{K^2} - \delta)n}})^2} \log(\frac{n}{\delta})\\
    & \Leftrightarrow 0 < \delta \leq \frac{\epsilon_0}{K^2} \wedge K < \sqrt{\frac{\epsilon_0 \zeta^2 n}{\alpha}} \wedge m  \ge \frac{c^2}{(\zeta - \sqrt{\frac{\alpha}{(\frac{2\epsilon_0}{K^2} - \delta)n}})^2} \log(\frac{n}{\delta})\\
    & \Leftarrow 0 < \delta \leq \frac{\epsilon_0}{K^2} \wedge 2K < \sqrt{\frac{\epsilon_0 \zeta^2 n}{\alpha}} \wedge  m  \ge \frac{c^2}{(\zeta - \sqrt{\frac{\alpha K^2}{\epsilon_0 n}} )^2}\log(\frac{n}{\delta})\\
    & \Leftarrow 0 < \delta \leq \frac{\epsilon_0}{K^2} \wedge2K < \sqrt{\frac{\epsilon_0 \zeta^2 n}{\alpha}} \wedge  m  \ge \frac{c^2}{(\zeta - \frac{\zeta}{2} )^2}\log(\frac{n}{\delta})\\
    & \Leftarrow 0 < \delta \leq \frac{\epsilon_0}{K^2} \wedge2K < \sqrt{\frac{\epsilon_0 \zeta^2 n}{\alpha}} \wedge m = O(\log(n)).
\end{align*}

We note that in the case that precondition of Proposition~\ref{Prop: multiclass exp -n} holds, the $O(\log(n))$ result is correct as well. Combining Theorem~\ref{Thm: multiclass generalization bound} and Proposition~\ref{Prop: multiclass exp -n}, we know that there exist positive $b, c = \Theta(1)$ and large enough $m$ such that when $c \sqrt{\frac{1}{m} \log(\frac{n}{\delta})} < \zeta$, with probability at least $1 - \delta$, we have
\begin{align*}
    R_{\ell_{0-1}}(h_{Gen, m}) &\leq R_{\ell_{0-1}}(h_{Gen, \infty}) + \frac{K(K-1)}{2} \biggl(\exp(-b(c \sqrt{\frac{1}{m} \log(\frac{n}{\delta})} - \zeta)^2n) + \delta\biggr)\\
    &\leq R_{\ell_{0-1}}(h_{Gen, \infty}) + \frac{K^2}{2} \biggl(\exp(-b(c \sqrt{\frac{1}{m} \log(\frac{n}{\delta})} - \zeta)^2n) + \delta\biggr).
\end{align*}
For fixed $\epsilon_0 \in (0,1)$, the logical relations listed in the following is correct:
\begin{align*}
    &R_{\ell_{0-1}}(h_{Gen, m}) \leq R_{\ell_{0-1}}(h_{Gen, \infty}) + \epsilon_0 \text{ with probability at least 1 - $\delta$} \\
    & \Leftarrow c \sqrt{\frac{1}{m} \log(\frac{n}{\delta})} < \zeta \wedge 0 < \delta < 1 \wedge \frac{K^2}{2} \biggl(\exp(-b(c \sqrt{\frac{1}{m} \log(\frac{n}{\delta})} - \zeta)^2n) + \delta\biggr) \leq \epsilon_0\\
    & \Leftrightarrow c \sqrt{\frac{1}{m} \log(\frac{n}{\delta})} < \zeta \wedge  0 < \delta < 1  \wedge \exp(-b(c \sqrt{\frac{1}{m} \log(\frac{n}{\delta})} - \zeta)^2n) \leq \frac{2\epsilon_0}{K^2} - \delta\\
    & \Leftrightarrow c \sqrt{\frac{1}{m} \log(\frac{n}{\delta})} < \zeta \wedge 0 < \delta < 1 \wedge \frac{2\epsilon_0}{K^2} - \delta > 0 \wedge -b(c \sqrt{\frac{1}{m} \log(\frac{n}{\delta})} - \zeta)^2n \leq \log(\frac{2\epsilon_0}{K^2} - \delta)\\
    & \Leftrightarrow c \sqrt{\frac{1}{m} \log(\frac{n}{\delta})} < \zeta \wedge 0 < \delta < \frac{2\epsilon_0}{K^2} \wedge (c \sqrt{\frac{1}{m} \log(\frac{n}{\delta})} - \zeta)^2 \ge \frac{1}{bn} \log(\frac{1}{\frac{2\epsilon_0}{K^2} - \delta})\\
    & \Leftarrow c \sqrt{\frac{1}{m} \log(\frac{n}{\delta})} < \zeta \wedge 0 < \delta < \frac{2\epsilon_0}{K^2} \wedge \zeta - c \sqrt{\frac{1}{m} \log(\frac{n}{\delta})} \ge \sqrt{\frac{1}{bn} \log(\frac{1}{\frac{2\epsilon_0}{K^2} - \delta})}\\
    & \Leftrightarrow 0 < \delta < \frac{2\epsilon_0}{K^2} \wedge \zeta - \sqrt{\frac{1}{bn} \log(\frac{1}{\frac{2\epsilon_0}{K^2} - \delta})} \ge c \sqrt{\frac{1}{m} \log(\frac{n}{\delta})}\\
    & \Leftrightarrow 0 < \delta < \frac{2\epsilon_0}{K^2} \wedge \zeta - \sqrt{\frac{1}{bn} \log(\frac{1}{\frac{2\epsilon_0}{K^2} - \delta})} > 0 \wedge (\zeta - \sqrt{\frac{1}{bn} \log(\frac{1}{\frac{2\epsilon_0}{K^2} - \delta})})^2 \ge c^2 \frac{1}{m} \log(\frac{n}{\delta})\\
    & \Leftarrow 0 < \delta < \frac{2\epsilon_0}{K^2} - \exp(-b\zeta^2n) \wedge \frac{2\epsilon_0}{K^2} - \exp(-b\zeta^2n) > 0 \wedge m  \ge \frac{c^2}{(\zeta - \sqrt{\frac{1}{bn} \log(\frac{1}{\frac{2\epsilon_0}{K^2} - \delta})})^2} \log(\frac{n}{\delta})\\
    & \Leftarrow 0 < \delta \leq \frac{\epsilon_0}{K^2} \wedge \frac{2\epsilon_0}{K^2} - \exp(-b\zeta^2n) > \frac{\epsilon_0}{K^2} \wedge m  \ge \frac{c^2}{(\zeta - \sqrt{\frac{1}{bn} \log(\frac{1}{\frac{2\epsilon_0}{K^2} - \delta})})^2} \log(\frac{n}{\delta})\\
    & \Leftrightarrow 0 < \delta \leq \frac{\epsilon_0}{K^2} \wedge K < \sqrt{\epsilon_0} \exp(\frac{bn\zeta^2}{2}) \wedge m  \ge \frac{c^2}{(\zeta - \sqrt{\frac{1}{bn} \log(\frac{1}{\frac{2\epsilon_0}{K^2} - \delta})})^2} \log(\frac{n}{\delta})\\
    & \Leftarrow 0 < \delta \leq \frac{\epsilon_0}{K^2} \wedge 2K < \sqrt{\epsilon_0} \exp(\frac{bn\zeta^2}{2}) \wedge m  \ge \frac{c^2}{(\zeta - \sqrt{\frac{1}{bn} \log(\frac{K^2}{\epsilon_0})})^2} \log(\frac{K^2n}{\epsilon_0})\\
    & \Leftarrow 0 < \delta \leq \frac{\epsilon_0}{K^2} \wedge 2K < \sqrt{\epsilon_0} \exp(\frac{bn\zeta^2}{2}) \wedge m  \ge \frac{c^2}{\zeta^2(1 - \frac{\log(K^2/\epsilon_0)}{\log(4K^2/\epsilon_0)})^2} \log(\frac{K^2n}{\epsilon_0})\\
    & \Leftrightarrow 0 < \delta \leq \frac{\epsilon_0}{K^2} \wedge 2K < \sqrt{\epsilon_0} \exp(\frac{bn\zeta^2}{2}) \wedge m = O(\log(n)).
\end{align*}
\end{proof}

\subsection{Proof of Proposition~\ref{Prop: multiclass logisticbound}}
 \label{proof: Prop: multiclass logisticbound}
We first present the following lemmas to show Proposition~\ref{Prop: multiclass logisticbound}.

\begin{lemma}[\cite{mohri2018foundations}, Theorem 3.3]
\label{lemma: rademacher}
Let $\mathcal{G}$ be a family of functions mapping from $\mathcal{Z}$ to $[0, c]$. Then, for any $\delta$ > 0, with probability at least $1 - \delta$ over the draw of an $i.i.d.$ sample $S$ of size $m$, the following holds for all $g \in \mathcal{G}$:
\begin{align*}
    \mathbb{E}[g(z)] \leq \frac{1}{m} \sum_{i = 1}^m g(z_i) + 2\mathcal{R}_m(\mathcal{G}) + c\sqrt{\frac{1}{2m} \log(\frac{2}{\delta})},
\end{align*}
where $\mathcal{R}_m(\mathcal{G})$ is the Rademacher complexity of $\mathcal{G}$.
\end{lemma}

\begin{lemma}[\cite{maurer2016vector}, Corollary 4]
\label{lemma: vector concentration}
Let $\mathcal{X}$ be any set, $S = (\vx_1, ..., \vx_m) \in \mathcal{X}^m$, $\sigma_1, \dots, \sigma_m$ be Rademacher random variables, $\mathcal{H}$ be a class of functions $\boldsymbol{h} : \mathcal{X} \to \ell_2$ and let $\Phi : \ell_2 \to \mathbb{R}$ have Lipschitz norm $L$, where $\ell_2$ is Hilbert space of square summable sequences of real numbers. Then we have
\begin{equation*}
    \mathcal{R}_m(\Phi \circ \mathcal{H}) =  \frac{1}{m} \mathbb{E}_{S, \sigma} \sup_{\boldsymbol{h}} \sum_{i} \sigma_{i} \Phi(\boldsymbol{h}(\vx_i)) \leq \sqrt{2}L \frac{1}{m} \mathbb{E}_{S, \sigma} \sup_{\boldsymbol{h}} \sum_{i, k} \sigma_{ik} h_k(\vx_i).
\end{equation*}
where $\sigma_{ik}$ is an independent doubly indexed Rademacher sequence and $h_k(x_i)$ is
the $k$-th component of $\vh(x_i)$.
\end{lemma}

\begin{lemma}
\label{lemma: rademacher complexity of multiclass constrained linear hypothesis}
Let $\mathcal{X} = [0,1]^n$, $S = (\vx_1, ..., \vx_m) \in \mathcal{X}^m$, $\mathcal{H} = \{\vx \to \boldsymbol{h}(\vx): h_y(\vx) = \vw_y^Tx + b_y, \Vert \vw_y \Vert_2 \leq W, \vert b_y\vert \leq B, y \in \mathcal{Y}\}$ and $\sigma_{ik}$ be independent doubly indexed Rademacher sequence. Then we have
\begin{equation*}
    \frac{1}{m} \mathbb{E}_{S, \sigma} \sup_{\boldsymbol{h}} \sum_{i, k} \sigma_{ik} h_k(\vx_i) \leq WK \sqrt{\frac{n}{m}}.
\end{equation*}
\end{lemma}
\begin{proof}
\begin{align*}
     \frac{1}{m} \mathbb{E}_{S, \sigma} \sup_{\boldsymbol{h}} \sum_{i, k} \sigma_{ik} h_k(\vx_i) &= \frac{1}{m} \mathbb{E}_{S, \sigma} \sup_{\boldsymbol{h}} \sum_{i, k} \sigma_{ik} (\langle \vw_k, \vx_i\rangle + b_k)\\
    &= \frac{1}{m} \mathbb{E}_{S, \sigma} \sup_{\boldsymbol{h}} \sum_{i, k} \sigma_{ik} \langle \vw_k, \vx_i\rangle\\
    &\leq \sum_{k=1}^K \frac{1}{m} \mathbb{E}_{S, \sigma} \sup_{h_k} \sum_{i=1}^m \sigma_{ik} \langle \vw_k, \vx_i\rangle\\
    &= \sum_{k=1}^K  \frac{1}{m} \mathbb{E}_{S, \sigma} \sup_{h_k} \langle \vw_k,\sum_{i=1}^m \sigma_{ik}  \vx_i\rangle\\
    &\leq \sum_{k=1}^K  \frac{1}{m} \mathbb{E}_{S, \sigma} \sup_{h_k} \Vert \vw_k\Vert_2 \Vert\sum_{i=1}^m \sigma_{ik}  \vx_i\Vert_2\\
    &\leq \sum_{k=1}^K \frac{W}{m} \mathbb{E}_{S, \sigma} \Vert\sum_{i=1}^m \sigma_{ik}  \vx_i\Vert_2\\
    &\leq \sum_{k=1}^K \frac{W}{m} \sqrt{\mathbb{E}_{S, \sigma} \Vert\sum_{i=1}^m \sigma_{ik}  \vx_i\Vert_2^2}\\
    &= \sum_{k=1}^K \frac{W}{m} \sqrt{\sum_{i=1}^m\Vert  \vx_i\Vert_2^2}\\
    &\leq  \frac{WK}{m} \sqrt{m \times n} = WK\sqrt{\frac{n}{m}}
\end{align*}
\end{proof}

\begin{lemma}
\label{lemma: rademacher complexity of multiclass constrained linear hypothesis 2}
Let $\mathcal{X} = [0,1]^n$, $\mathcal{Y} = \{1, \dots, K\}$, $S = ((\vx_1, y_1), ..., (\vx_m,y_m)) \in (\mathcal{X}, \mathcal{Y})^m$, $\mathcal{H} = \{\vx \to \boldsymbol{h}(\vx): h_y(\vx) = \vw_y^T\vx + b_y, \Vert \vw_y \Vert_2 \leq W, \vert b_y\vert \leq B, y \in \mathcal{Y}\}$ and $\Pi_1(\mathcal{H}) = \{(\vx,y) \to h_y(\vx), y \in \mathcal{Y}, \boldsymbol{h} \in \mathcal{H}\}$. Then we have
\begin{equation*}
    \frac{1}{m} \mathbb{E}_{S, \sigma}[\sup_{\boldsymbol{h}} \sum_{i=1}^m \sigma_i h_{y_i}(\vx_i)] \leq K \mathcal{R}_m(\Pi_1(\mathcal{H})).
\end{equation*}
\end{lemma}
\begin{proof}
\begin{align*}
    &\frac{1}{m} \mathbb{E}_{S, \sigma}[\sup_{\boldsymbol{h}} \sum_{i=1}^m \sigma_i h_{y_i}(\vx_i)] \\
    &= \frac{1}{m} \mathbb{E}_{S, \sigma}[\sup_{\boldsymbol{h}} \sum_{i=1}^m \sigma_i \sum_{y \in \mathcal{Y}} h_{y}(\vx_i) \mathbbm{1}_{y_i = y}]\\
    &\leq \frac{1}{m} \sum_{y \in \mathcal{Y}} \mathbb{E}_{S, \sigma}[\sup_{\boldsymbol{h}} \sum_{i=1}^m \sigma_i h_{y}(\vx_i) \mathbbm{1}_{y_i = y}]\\
    &=\sum_{y \in \mathcal{Y}}\frac{1}{m} \mathbb{E}_{S, \sigma}[\sup_{\boldsymbol{h}} \sum_{i=1}^m \sigma_i h_{y}(\vx_i) (\frac{\epsilon_i}{2} + \frac{1}{2})] & (\epsilon_i = 2\times\mathbbm{1}_{y_i = y} - 1 \in \{-1, +1\})\\
    &\leq \sum_{y \in \mathcal{Y}}\frac{1}{2m} \mathbb{E}_{S, \sigma}[\sup_{\boldsymbol{h}} \sum_{i=1}^m \sigma_i h_{y}(\vx_i){\epsilon_i}] + \frac{1}{2m} \mathbb{E}_{S, \sigma}[\sup_{\boldsymbol{h}} \sum_{i=1}^m \sigma_i h_{y}(\vx_i)]\\
    &= \sum_{y \in \mathcal{Y}} \frac{1}{m} \mathbb{E}_{S, \sigma}[\sup_{\boldsymbol{h}} \sum_{i=1}^m \sigma_i h_{y}(\vx_i)] \\
    &\leq K \mathcal{R}_m(\Pi_1(\mathcal{H})).
\end{align*}
\end{proof}

\begin{lemma}
\label{lemma: rademacher complexity of multiclass constrained linear hypothesis 3}
Let $\mathcal{X} = [0,1]^n$, $\mathcal{Y} = \{1, \dots, K\}$, $S = ((\vx_1, y_1), ..., (\vx_m,y_m)) \in (\mathcal{X}, \mathcal{Y})^m$, $\mathcal{H} = \{x \to \boldsymbol{h}(\vx): h_y(\vx) =\vw_y^T\vx + b_y, \Vert \vw_y \Vert_2 \leq W, \vert b_y\vert \leq B, y \in \mathcal{Y}\}$ and $\Pi_1(\mathcal{H}) = \{(\vx,y) \to h_y(\vx), y \in \mathcal{Y}, \boldsymbol{h} \in \mathcal{H}\}$. Then we have
\begin{equation*}
   \mathcal{R}_m(\Pi_1(\mathcal{H})) \leq W\sqrt{\frac{n}{m}}.
\end{equation*}
\end{lemma}
\begin{proof}
\begin{align*}
    \mathcal{R}_m(\Pi_1(\mathcal{H})) &=  \frac{1}{m} \mathbb{E}_{S, \sigma}[\sup_{\boldsymbol{h}, y} \sum_{i=1}^m \sigma_i h_{y}(\vx_i)]
    = \frac{1}{m} \mathbb{E}_{S, \sigma}[ \sup_{\boldsymbol{h}, y} \sum_{i = 1}^m \sigma_{i} (\langle \vw_y, \vx_i\rangle + b_y)]\\
    &= \frac{1}{m}  \mathbb{E}_{S, \sigma}[ \sup_{\boldsymbol{h}, y} \sum_{i = 1}^m \sigma_{i} \langle \vw_y, \vx_i\rangle]
    = \frac{1}{m} \mathbb{E}_{S, \sigma}[ \sup_{\boldsymbol{h}, y} \langle \vw_y,\sum_{i=1}^m \sigma_{i}  \vx_i\rangle]\\
    &\leq \frac{W}{m} \mathbb{E}_{S, \sigma}[ \Vert\sum_{i=1}^m \sigma_{i}  \vx_i\Vert_2]
    \leq \frac{W}{m} \sqrt{\mathbb{E}_{S, \sigma}[ \Vert\sum_{i=1}^m \sigma_{i}  \vx_i\Vert_2^2}] \\
    &= \frac{W}{m} \sqrt{\sum_{i=1}^m\Vert  \vx_i\Vert_2^2} \leq \frac{W}{m} \sqrt{m \times n} = W\sqrt{\frac{n}{m}}.
\end{align*}
\end{proof}

\begin{lemma}[\cite{shalev2014understanding}, Lemma B.6]
\label{lemma: hoeffding}
Let $Z_1, \dots, Z_m$ be a sequence of i.i.d. random variables and let $\bar{Z} = \frac{1}{m} \sum_{i=1}^m Z_i$. Assume that $\mathbb{E}[\bar{Z}] = \mu$ and $P[a \leq Z_i \leq b] = 1$ for every $i$. Then, for any $\epsilon > 0$:
\begin{equation*}
    \mathbb{P}[\vert \bar{Z} - \mu\vert > \epsilon] \leq 2\exp({-\frac{2m\epsilon^2}{(b-a)^2}}).
\end{equation*}
\end{lemma}

Based on the above lemmas, we now prove Proposition~\ref{Prop: multiclass logisticbound} as follows.
\begin{proof}
We first rewrite $R_{\ell_{log}}(\boldsymbol{h}_{Dis, m}) - R_{\ell_{log}}(\boldsymbol{h}_{Dis, \infty})$.
\begin{align*}
     & R_{\ell_{log}}(\boldsymbol{h}_{Dis, m}) - R_{\ell_{log}}(\boldsymbol{h}_{Dis, \infty}) \\
     &= R_{\ell_{log}}(\boldsymbol{h}_{Dis, m}) - \hat{R}_{\ell_{log}, S}(\boldsymbol{h}_{Dis, m}) + \hat{R}_{\ell_{log}, S}(\boldsymbol{h}_{Dis, m}) -\hat{R}_{\ell_{log}, S}(\boldsymbol{h}_{Dis, \infty}) +  \hat{R}_{\ell_{log}, S}(\boldsymbol{h}_{Dis, \infty}) - R_{\ell_{log}}(\boldsymbol{h}_{Dis, \infty})  \\
    & \leq  (R_{\ell_{log}}(\boldsymbol{h}_{Dis, m}) - \hat{R}_{\ell_{log}, S}(\boldsymbol{h}_{Dis, m})) +  (\hat{R}_{\ell_{log}, S}(\boldsymbol{h}^*) - R_{\ell_{log}}(\boldsymbol{h}_{Dis, \infty})).
\end{align*}
We consider the first summand now. By Lemma~\ref{lemma: rademacher}, with probability of at least $1 - \delta$, we have:
\begin{align*}
     &R_{\ell_{log}}(\boldsymbol{h}_{Dis, m}) - \hat{R}_{\ell_{log}, S}(\boldsymbol{h}_{Dis, m}) \\
     &\leq 2\mathcal{R}_m(\ell_{log} \circ (\mathcal{H}, \mathcal{Y})) + \log(1 + (K-1)\exp{2(W\sqrt{n} + B)}) \sqrt{\frac{1}{2m} \log(\frac{2}{\delta})}
\end{align*}
We define $\Pi_1(\mathcal{H}) = \{(\vx,y) \to h_y(\vx), y \in \mathcal{Y}, \boldsymbol{h} \in \mathcal{H}\}$ and $\Phi = \{\boldsymbol{h} \to \log(\sum_{y=1}^K \exp{(h_y))}, \boldsymbol{h} \in \mathcal{H}(\vx) \subseteq \mathbb{R}^K\}$. We can bound $\mathcal{R}_m(\ell_{log} \circ (\mathcal{H}, \mathcal{Y}))$ as follows:
\begin{align*}
     \mathcal{R}_m(\ell_{log} \circ (\mathcal{H}, \mathcal{Y})) &=\frac{1}{m} \mathbb{E}_{S, \sigma}[\sup_{\boldsymbol{h}} \sum_{i=1}^m \sigma_i \ell_{log}(\boldsymbol{h}(\vx_i), y_i)] \\
     &=\frac{1}{m} \mathbb{E}_{S, \sigma}[\sup_{\boldsymbol{h}} \sum_{i=1}^m \sigma_i (\log(\sum_{i=1}^K \exp{(h_k(\vx_i))}) - h_{y_i}(\vx_i))] \\
     &\leq \frac{1}{m} \mathbb{E}_{S, \sigma}[\sup_{\boldsymbol{h}} \sum_{i=1}^m \sigma_i \log(\sum_{i=1}^K \exp{(h_k(\vx_i))})] + \frac{1}{m} \mathbb{E}_{S, \sigma}[\sup_{\boldsymbol{h}} \sum_{i=1}^m \sigma_i h_{y_i}(\vx_i)] \\
     &=  \mathcal{R}_m(\Phi \circ \mathcal{H}) + \frac{1}{m} \mathbb{E}_{S, \sigma}[\sup_{\boldsymbol{h}} \sum_{i=1}^m \sigma_i h_{y_i}(\vx_i)].
\end{align*}
We will bound $\mathcal{R}_m(\Phi \circ \mathcal{H})$ by using Lemma~\ref{lemma: vector concentration}. Before that, we note $\Phi$ has Lipschitz norm $\sqrt{K}$. Because $\frac{\partial \Phi}{\partial h_i} \leq 1$ for any $i \in \{1, \dots, K\}$. Then, for any $\vh_1, \vh_2 \in \mathbb{R}^K$, we have
\begin{equation*}
    \vert \Phi(\boldsymbol{h}_1) - \Phi(\boldsymbol{h}_2)\vert \leq \vert \sum_{k=1}^K \vert \boldsymbol{h}_{1k} - \boldsymbol{h}_{2k}\vert \vert \leq \sqrt{K} \Vert \boldsymbol{h}_1 - \boldsymbol{h}_2\Vert_2.
\end{equation*}
Then we can bound $\mathcal{R}_m(\Phi \circ \mathcal{H})$ as the following
\begin{align*}
    \mathcal{R}_m(\Phi \circ \mathcal{H}) &\leq \sqrt{2} \sqrt{K} \frac{1}{m} \mathbb{E}_{S, \sigma} [\sup_{\boldsymbol{h}} \sum_{i, k} \sigma_{ik} h_k(\vx_i)] &(\text{by Lemma~\ref{lemma: vector concentration}})\\
    &\leq \sqrt{2K} WK \sqrt{\frac{n}{m}} = W\sqrt{\frac{2K^3n}{m}}. &(\text{by Lemma~\ref{lemma: rademacher complexity of multiclass constrained linear hypothesis}})
\end{align*}
We can also bound $\frac{1}{m} \mathbb{E}_{S, \sigma}[\sup_{\boldsymbol{h}} \sum_{i=1}^m \sigma_i h_{y_i}(\vx_i)]$ as follows
\begin{align*}
    \frac{1}{m} \mathbb{E}_{S, \sigma}[\sup_{\boldsymbol{h}} \sum_{i=1}^m \sigma_i h_{y_i}(\vx_i)]
    &\leq K \mathcal{R}_m(\Pi_1(\mathcal{H})) &(\text{by Lemma~\ref{lemma: rademacher complexity of multiclass constrained linear hypothesis 2}})\\
    &\leq KW \sqrt{\frac{n}{m}} = W\sqrt{\frac{K^2n}{m}}. &(\text{by Lemma~\ref{lemma: rademacher complexity of multiclass constrained linear hypothesis 3}})
\end{align*}
Therefore, we can obtain
\begin{equation}
\label{eq:logisitc-bound-estimation-error}
    R_{\ell_{log}}(\boldsymbol{h}_{Dis, m}) - \hat{R}_{\ell_{log}, S}(\boldsymbol{h}_{Dis, m}) \leq 2W(\sqrt{\frac{2K^3n}{m}} + \sqrt{\frac{K^2n}{m}}) + \log(1 + (K-1)\exp{2(W\sqrt{n} + B)}) \sqrt{\frac{1}{2m} \log(\frac{2}{\delta})}.
\end{equation}
For the second summand, we use the fact that $R_{\ell_{log}}(\boldsymbol{h}_{Dis, \infty})$ does not depend on $\mathcal{S}$; hence by Lemma~\ref{lemma: hoeffding}, we obtain its bound:
\begin{align*}
    \mathbb{P}(\vert \hat{R}_{\ell_{log}, S}(\boldsymbol{h}_{Dis, \infty}) - R_{\ell_{log}}(\boldsymbol{h}_{Dis, \infty}) \vert > \epsilon) \leq 2\exp({-\frac{2m\epsilon^2}{(c-0)^2}}) = 2\exp({-\frac{2m\epsilon^2}{c^2}}),
\end{align*}
where $c = \log(1 + (K-1)\exp{2(W\sqrt{n} + B)})$. It implies that with the probability of at least $1-\delta$, we have:
\begin{equation}
\label{eq:logisitc-bound-approx-error}
    \hat{R}_{\ell_{log}, S}(\boldsymbol{h}_{Dis, \infty}) - R_{\ell_{log}}(\boldsymbol{h}_{Dis, \infty}) \leq c\sqrt{\frac{1}{2m} \log(\frac{2}{\delta})}.
\end{equation}
At last, we make use of the union bound for Eq.~(\ref{eq:logisitc-bound-estimation-error}) and~(\ref{eq:logisitc-bound-approx-error}) to get the final result. With probability at least $1-\delta$, the following holds:
\begin{align*}
     &R_{\ell_{log}}(\boldsymbol{h}_{Dis, m}) - R_{\ell_{log}}(\boldsymbol{h}_{Dis, \infty}) \\
     &\leq 2W(\sqrt{\frac{2K^3n}{m}} + \sqrt{\frac{K^2n}{m}}) + \log(1 + (K-1)\exp{2(W\sqrt{n} + B)}) \sqrt{\frac{1}{2m} \log(\frac{4}{\delta})} + c\sqrt{\frac{1}{2m} \log(\frac{4}{\delta})}\\
    &= 2W(\sqrt{\frac{2K^3n}{m}} + \sqrt{\frac{K^2n}{m}}) + 2\log(1 + (K-1)\exp{2(W\sqrt{n} + B)})\sqrt{\frac{1}{2m} \log(\frac{4}{\delta})}\\
    & = O(\sqrt{\frac{K^3n}{m}}).
\end{align*}
Therefore, for $R_{\ell_{log}}(h_{Dis, m}) \leq R_{\ell_{log}}(\boldsymbol{h}_{Dis, \infty}) +\epsilon_0$ to  hold with high probability $1 - \delta_0$ (here, $\epsilon_0$ and $\delta_0$ are some fixed constant), it suffices to pick $m = O(n)$ samples.
\end{proof}

\subsection{Proof of Theorem~\ref{cor: sample complexity of multiclass lr}}
\label{proof: Proof of Corollary cor: sample complexity of multiclass lr}

\begin{proof}
By Theorem~\ref{thm: H-consistency bound for log} we know that for $R_{\ell_{0-1}}(\boldsymbol{h}_{Dis, m}) \leq R_{\ell_{0-1}}(\boldsymbol{h}_{Dis, \infty}) +\epsilon_0$, it is sufficient to ensure that $R_{\ell_{log}}(\boldsymbol{h}_{Dis, m}) \leq R_{\ell_{log}}(\boldsymbol{h}_{Dis, \infty}) +\frac{1}{2}\epsilon_0^2$. Then by Proposition~\ref{Prop: multiclass logisticbound}, it suffices to sample $m = O(\frac{K^3n}{\epsilon_0^4}) = O(K^3n)$.
\end{proof}

\section{Proofs of Section~\ref{sec: multiclass H-consistency framework}}
\label{Proofs of sec: multiclass H-consistency framework}

\subsection{Proof of Proposition~\ref{Thm: Distribution-dependent convex bound}}
\label{proof: Thm: Distribution-dependent convex bound}
\begin{proof}
Fix $\boldsymbol{h} \in \mathcal{H}$, because $g(\langle \Delta \mathscr{C}_{\ell_2, \mathcal{H}}(\boldsymbol{h}, \vx) \rangle_\epsilon) < \Delta \mathscr{C}_{\ell_1, \mathcal{H}}(\boldsymbol{h}, \vx)$ for all $x \in \mathcal{X}$, we have:
\begin{align*}
    &g(R_{\ell_2}(\boldsymbol{h}) - R^*_{\ell_2, \mathcal{H}} + M_{\ell_2, \mathcal{H}})\\
    &= g(\mathbb{E}_{\vx} [\mathscr{C}_{\ell_2}(\boldsymbol{h}, \vx)] - R^*_{\ell_2, \mathcal{H}} + R_{\ell_2, \mathcal{H}}^* - \mathbb{E}_{\vx}[\mathscr{C}_{\ell_2, \mathcal{H}}^*(\vx)])\\
    &= g(\mathbb{E}_{\vx} [\mathscr{C}_{\ell_2}(\boldsymbol{h}, \vx) - \mathscr{C}_{\ell_2, \mathcal{H}}^*(\vx)])\\
    &\leq \mathbb{E}_{\vx}[g(\mathscr{C}_{\ell_2}(\boldsymbol{h}, \vx) - \mathscr{C}_{\ell_2, \mathcal{H}}^*(\vx))] & (\text{Jensen's inequality})\\
    &= \mathbb{E}_{\vx}[g(\Delta \mathscr{C}_{\ell_2, \mathcal{H}}(\boldsymbol{h}, \vx))]\\
    &= \mathbb{E}_{\vx}[g(\Delta \mathscr{C}_{\ell_2, \mathcal{H}}(\boldsymbol{h}, \vx)\mathbbm{1}_{\mathscr{C}_{\ell_2, \mathcal{H}}(\boldsymbol{h}, \vx) > \epsilon} + \Delta \mathscr{C}_{\ell_2, \mathcal{H}}(\boldsymbol{h}, \vx)\mathbbm{1}_{\mathscr{C}_{\ell_2, \mathcal{H}}(\boldsymbol{h}, \vx) \leq \epsilon})]\\
    &\leq \mathbb{E}_{\vx}[g(\Delta \mathscr{C}_{\ell_2, \mathcal{H}}(\boldsymbol{h}, \vx)\mathbbm{1}_{\mathscr{C}_{\ell_2, \mathcal{H}}(\boldsymbol{h}, \vx) > \epsilon}) + g(\Delta \mathscr{C}_{\ell_2, \mathcal{H}}(\boldsymbol{h}, \vx)\mathbbm{1}_{\mathscr{C}_{\ell_2, \mathcal{H}}(\boldsymbol{h}, \vx) \leq \epsilon})] & (\text{$g(0) \ge 0$})\\
    &\leq \mathbb{E}_{\vx}[\Delta \mathscr{C}_{\ell_1, \mathcal{H}}(\boldsymbol{h}, \vx) + \sup_{t \in [0, \epsilon]} g(t)] & (\text{assumption})\\
    &= R_{\ell_1}(\boldsymbol{h}) - R^*_{\ell_1, \mathcal{H}} + M_{\ell_1, \mathcal{H}} + \max(g(0), g(\epsilon)). & (\text{$g$ is convex})
\end{align*}

\end{proof}

 \subsection{Proof of Theorem~\ref{cor: Distribution-independent convex Psi bound}}
 \label{proof: cor: Distribution-independent convex Psi bound}

 \begin{lemma}[Character of conditional $\epsilon$-regret for $\ell_{0-1}$]
\label{Prop: Character of conditional epsilon-regret}
Suppose that $\mathcal{H}$ satisfies that $\{\mathop{\mathrm{argmax}}_{y \in \mathcal{Y}} h_y(\vx) : \boldsymbol{h} \in \mathcal{H}\} = \{1, \dots, K\}$ for any $\vx \in \mathcal{X}$, then the minimal conditional zero-one loss $\ell_{0-1}$ is
\begin{equation*}
    \mathscr{C}_{\ell_{0-1}, \mathcal{H}}^*(\vx) = \mathscr{C}_{\ell_{0-1}, \mathcal{H}_{all}}^*(\vx) = 1 - \max_y p_y(\vx).
\end{equation*}
Furthermore, the conditional $\epsilon$-regret for $\ell_{0-1}$ can be characterized as 
\begin{equation*}
     \langle \Delta \mathscr{C}_{\ell_{0-1}, \mathcal{H}}(\boldsymbol{h}, \vx) \rangle_\epsilon = \langle \max_y p_y(\vx) - p_{\hat{y}}(\vx) \rangle_\epsilon,
\end{equation*}
where $\hat{y} = \mathop{\mathrm{argmax}}_{y \in \mathcal{Y}} h_y(\vx)$.
\end{lemma}
\begin{proof}
By the definition of $ \mathscr{C}_{\ell_{0-1}}(\boldsymbol{h},\vx)$, we have:
\begin{equation*}
    \mathscr{C}_{\ell_{0-1}}(\boldsymbol{h}, \vx) = \sum_{y=1}^K p_y(\vx) \ell_{0-1}(\boldsymbol {h}(\vx), y) = \sum_{y=1}^K p_y(\vx) \mathbbm{1}_{\hat{y} \ne y}.
\end{equation*}
By the assumption, we know that there exists $\boldsymbol{h}^* \in \mathcal{H}$ which satisfies $\mathop{\mathrm{argmax}}_{y \in \mathcal{Y}} h^*_y(\vx) = \mathop{\mathrm{argmax}}_{y \in \mathcal{Y}} p_y(\vx)$. Therefore, we have
\begin{equation*}
    \mathscr{C}_{\ell_{0-1}, \mathcal{H}}^*(\vx) = \inf_{\boldsymbol{h} \in \mathcal{H}} \mathscr{C}_{\ell_{0-1}}(\boldsymbol{h}, \vx) = \mathscr{C}_{\ell_{0-1}}(\boldsymbol{h}^*, \vx)  = 1 - \max_y p_y(\vx). 
\end{equation*}
Then we can find the characteristic of conditional $\epsilon$-regret for $\ell_{0-1}$ as follows:
 \begin{align*}
      \Delta \mathscr{C}_{\ell_{0-1}, \mathcal{H}}(\boldsymbol{h}, \vx) &= \mathscr{C}_{\ell_{0-1}}(\boldsymbol{h}, \vx) - \mathscr{C}_{\ell_{0-1}, \mathcal{H}}^*(\vx) \\
      &= \sum_{y=1}^K p_y(\vx) \mathbbm{1}_{\hat{y} \ne y} - (1 - \max_y p_y(\vx))\\
      &= \sum_{y \ne \hat{y}} p_y(\vx) -  \sum_{y \ne y_{max}} p_y(\vx)\\
      &= \max_y p_y(\vx) - p_{\hat{y}}(\vx).
 \end{align*}
\end{proof}

\begin{lemma}[Distribution-dependent convex $\ell_{0-1}$ bound]
\label{cor: Distribution-dependent convex bound}
Suppose that $\mathcal{H}$ satisfies that $\{\mathop{\mathrm{argmax}}_{y \in \mathcal{Y}} h_y(\vx) : \boldsymbol{h} \in \mathcal{H}\} = \{1, \dots, K\}$ for any $\vx \in \mathcal{X}$, and there exists a convex function $g: \mathbb{R}_+ \to \mathbb{R}$ with $g(0) = 0$ and $\epsilon \ge 0$ that the following holds for any $\hat{y} \in \mathcal{Y}$, $x \in \mathcal{X}$ and $\vh \in \mathcal{H}_{\hat{y}}(\vx)$:
\begin{equation*}
    g(\langle \max_y p_y(\vx) - p_{\hat{y}}(\vx) \rangle_\epsilon) \leq \inf_{\boldsymbol{h} \in \mathcal{H}_{\hat{y}}(\vx)} \Delta \mathscr{C}_{\ell, \mathcal{H}}(\boldsymbol{h}, \vx).
\end{equation*}
Then it holds for all $\boldsymbol{h} \in \mathcal{H}$ that
\begin{equation*}
    g(R_{\ell_{0-1}}(\boldsymbol{h}) - R^*_{\ell_{0-1}, \mathcal{H}} + M_{\ell_{0-1}, \mathcal{H}}) \leq R_{\ell(\boldsymbol{h})} - R^*_{\ell, \mathcal{H}} + M_{\ell, \mathcal{H}} + \max(0, g(\epsilon)).
\end{equation*}
\end{lemma}
\begin{proof}
For any $\vx_0 \in \mathcal{X}$ and $\boldsymbol{h}_0 \in \mathcal{H}$, let $\hat{y} $ be the index of the largest element of $\vh_0(\vx)$. Then by the precondition, we have
\begin{equation*}
    g(\langle \Delta \mathscr{C}_{\ell_{0-1}, \mathcal{H}}(\boldsymbol{h}_0, \vx_0) \rangle_\epsilon) = g(\langle \max_y p_y(\vx_0) - p_{\hat{y}}(\vx_0) \rangle_\epsilon) \leq  \inf_{\boldsymbol{h} \in \mathcal{H}_{\hat{y}}(\vx_0) } \Delta \mathscr{C}_{\ell, \mathcal{H}}(\boldsymbol{h}, \vx_0) \leq \Delta \mathscr{C}_{\ell, \mathcal{H}}(\boldsymbol{h}_0, \vx_0).
\end{equation*}
Combining the condition in Proposition~\ref{Thm: Distribution-dependent convex bound} we can see that this lemma is correct.
\end{proof}

Built upon the above lemmas, we can prove Theorem~\ref{cor: Distribution-independent convex Psi bound} as follows. 
\begin{proof}

For any $\vx_0 \in \mathcal{X}$, $\vp(\vx_0) \in \Delta_K$, $\hat{y}_0 \in \mathcal{Y}$, and $\vh \in\mathcal{H}_{\hat{y}_0}(\vx_0)$, we can write:
\begin{align*}
    &g(\max_y p_y(\vx_0) - p_{\hat{y}_0}(\vx_0)) \\
    &\leq \inf_{\hat{y} \in \mathcal{Y}, x \in \mathcal{X}, \boldsymbol{h} \in \mathcal{H}_{\hat{y}}(\vx), \vp \in \mathcal{P}_{\hat{y}}({\max_y p_y(\vx_0) - p_{\hat{y}_0}(\vx_0)})} \Delta \mathscr{C}_{\ell, \mathcal{H}}(\boldsymbol{h}, \vx, \vp) & (\text{Assumption}) \\
    &\leq \inf_{ x \in \mathcal{X}, \boldsymbol{h} \in \mathcal{H}_{\hat{y}_0}(\vx), \vp \in \mathcal{P}_{\hat{y}_0}({\max_y p_y(\vx_0) - p_{\hat{y}_0}(\vx_0)})} \Delta \mathscr{C}_{\ell, \mathcal{H}}(\boldsymbol{h}, \vx, \vp) \\
    &\leq \inf_{x \in \mathcal{X}, \boldsymbol{h} \in \mathcal{H}_{\hat{y}_0}(\vx)} \Delta \mathscr{C}_{\ell, \mathcal{H}}(\boldsymbol{h}, \vx, \vp(\vx_0))\\
    & \leq \inf_{ \boldsymbol{h} \in \mathcal{H}_{\hat{y}_0}(\vx_0)}\Delta \mathscr{C}_{\ell, \mathcal{H}}(\boldsymbol{h}, \vx_0, \vp(\vx_0))
    = \inf_{ \boldsymbol{h} \in \mathcal{H}_{\hat{y}_0}(\vx_0)} \Delta \mathscr{C}_{\ell, \mathcal{H}}(\boldsymbol{h}, \vx_0).
\end{align*}
Combining the result of Lemma~\ref{cor: Distribution-dependent convex bound}, we conclude the proof of Theorem~\ref{cor: Distribution-independent convex Psi bound}.

\end{proof}

\subsection{Proofs of Theorem~\ref{thm:tightness}}
\label{Proofs of thm:tightness}
\begin{proof}
By Theorem~\ref{cor: Distribution-independent convex Psi bound}, if $\mathcal{J}_{\ell}(t)$ is convex with $ \mathcal{J}_{\ell}(0) = 0$, the first inequality holds. For any $t\in [0,1]$, denote that the solution of $\inf_{\hat{y} \in \mathcal{Y}, \vp \in \mathcal{P}_{\hat{y}}(t), \vx \in \mathcal{X}, \boldsymbol{h} \in \mathcal{H}_{\hat{y}}(\vx)}\Delta \mathscr{C}_{\ell, \mathcal{H}}(\boldsymbol{h}, \vx, \vp)$ by $\vx^*, \vh^*, \vp^*, \hat{y}^*$. We then consider the distribution that is supported on the single point $\vx_0 = \vx^*$ and satisfy that $\vp(\vx_0) = \vp^*$. Thus,
\begin{align*}
 \inf_{\hat{y} \in \mathcal{Y}, \vp \in \mathcal{P}_{\hat{y}}(t), \vx \in \mathcal{X}, \boldsymbol{h} \in \mathcal{H}_{\hat{y}}(\vx)}\Delta \mathscr{C}_{\ell, \mathcal{H}}(\boldsymbol{h}, \vx, \vp) =  \inf_{\boldsymbol{h} \in \mathcal{H}_{\hat{y}^*}(\vx_0)}\Delta \mathscr{C}_{\ell, \mathcal{H}}(\boldsymbol{h}, \vx_0, \vp(\vx_0)) = \inf_{\boldsymbol{h} \in \mathcal{H}_{\hat{y}^*}(\vx_0)}\Delta \mathscr{C}_{\ell, \mathcal{H}}(\boldsymbol{h}, \vx_0).
\end{align*}
For any $\delta>0$, take $\vh_0\in \mathcal{H}$ such that $\vh_0 \in \mathcal{H}_{\hat{y}^*}(\vx_0)$ and
\begin{align*}
\Delta\sC_{\ell, \mathcal{H}}(\vh_0, \vx_0)\leq \inf_{\boldsymbol{h} \in \mathcal{H}_{\hat{y}^*}(\vx_0)}\Delta \mathscr{C}_{\ell, \mathcal{H}}(\boldsymbol{h}, \vx_0) + \delta= \inf_{\hat{y} \in \mathcal{Y}, \vp \in \mathcal{P}_{\hat{y}}(t), \vx \in \mathcal{X}, \boldsymbol{h} \in \mathcal{H}_{\hat{y}}(\vx)}\Delta \mathscr{C}_{\ell, \mathcal{H}}(\boldsymbol{h}, \vx, \vp) + \delta.    
\end{align*}
Then, we have
\begin{align*}
R_{\ell_{0-1}}(\vh_0) - R^*_{\ell_{0-1}, \mathcal{H}} + M_{\ell_{0-1}, \mathcal{H}}
&= R_{\ell_{0-1}}(\vh_0) - \mathbb{E}_{\vx} [\sC^*_{\ell_{0-1},\mathcal{H}}(\vx)]\\
& =\Delta\sC_{\ell_{0-1},\mathcal{H}}(\vh_0,\vx_0)\\ 
& = \max_y \vp_y(\vx_0) - \vp_{\hat{y}^*}(\vx_0)\\
& = t,\\
R_{\ell}(\vh_0) - R^*_{\ell, \mathcal{H}} + M_{\ell, \mathcal{H}}
& = R_{\ell}(\vh_0) -\mathbb{E}_{\vx} [\sC^*_{\ell,\mathcal{H}}(\vx)]\\
& = \Delta\sC_{\ell,\sH}(\vh_0, \vx_0)\\
& \leq \inf_{\hat{y}, \vp \in \mathcal{P}_{\hat{y}}(t), \vx \in \mathcal{X}, \boldsymbol{h} \in \mathcal{H}_{\hat{y}}(\vx)}\Delta \mathscr{C}_{\ell, \mathcal{H}}(\boldsymbol{h}, \vx, \vp) + \delta \\
& = \mathcal{J}_{\ell}(t) + \delta,
\end{align*}
which completes the proof.
\end{proof}

\subsection{Proof of Theorem~\ref{thm: H-consistency bound for log}}
 \label{proof: thm: H-consistency bound for log}
To prove the  Theorem~\ref{thm: H-consistency bound for log}, we first list the following lemmas.
\begin{lemma}[Convexity of $\mathscr{C}_{\ell}(\boldsymbol{h}, \vx, \vp)$]
\label{lemma: Convexity of C}
$\mathscr{C}_{\ell}(\boldsymbol{h}, \vx, \vp) = \sum_{y=1}^K p_y (-h_y + \log(\sum_{j=1}^K \exp{(h_j)})$ is convex with respect to $\boldsymbol{h}$.
\end{lemma}
\begin{proof}
For any fixed $\vx$ and $\vp$, we have
\begin{align*}
    \frac{ \partial \mathscr{C}_{\ell}(\boldsymbol{h}, \vx, \vp) }{ \partial h_i} = -p_i + \frac{\exp{(h_i)}}{\sum_{k=1}^K \exp{(h_k)}}.
\end{align*}
Let $A_{ij} = \frac{\partial^2 \mathscr{C}_{\ell}(\boldsymbol{h}, \vx, \vp)}{\partial h_i \partial h_j}$, we have
\begin{align*}
    A_{ij} = \left\{
        \begin{array}{ccl}
        -\frac{\exp{(h_i)}\exp{(h_j)}}{(\sum_{k=1}^K \exp{(h_k)})^2}, & & {i \ne j},\\
        \frac{\exp{(h_j)}\sum_{k=1, k \ne j}^K \exp{(h_k)}}{(\sum_{k=1}^K \exp{(h_k)})^2}, & & {i = j}.
        \end{array} \right.
\end{align*}
To prove $\mathscr{C}_{\ell}(\boldsymbol{h}, \vx, \vp)$ is convex with respect to $\boldsymbol{h}$, it's sufficient to show that $\vA$ is positive semidefinite, which equals to $\vx^T\vA\vx \ge 0$ for any $\vx \in \mathbb{R}^n$. We can calculate it as follows:
\begin{align*}
    \vx^T\vA \vx &= A_{11}x_1^2 + \dots + A_{nn}x_n^2 + \sum_{i \ne j} A_{ij}x_ix_j\\
    &= \frac{1}{(\sum_{k=1}^K \exp{(h_k)})^2}[\sum_{i=1}^K \exp{(h_i)}x_i^2(\sum_{k=1, k \ne i}^K \exp{(h_k)}) - \sum_{i=1}^K \exp{(h_i)}x_i(\sum_{j=1, j \ne i}^K \exp{(h_j)} x_j)]\\
    &= \frac{1}{(\sum_{k=1}^K \exp{(h_k)})^2}[\sum_{i=1}^K \exp{(h_i)}x_i(\sum_{j=1, j \ne i}^K \exp{(h_j)} (x_i - x_j))]\\
    &= \frac{1}{(\sum_{k=1}^K \exp{(h_k)})^2} [\sum_{i<j} \exp{(h_i)}\exp{(h_j)} (x_i - x_j)^2] \ge 0.
\end{align*}
which proves this lemma.
\end{proof}

\begin{lemma}[Property of $M_{\ell_{0-1}, \mathcal{H}}$]
\label{lemma: property of M_Psi_F}
Suppose that $\mathcal{H}$ satisfies that $\{\mathop{\mathrm{argmax}}_{y \in \mathcal{Y}} h_y(\vx) : \boldsymbol{h} \in \mathcal{H}\} = \{1, \dots, K\}$ for any $\vx \in \mathcal{X}$. Then $M_{\ell_{0-1}, \mathcal{H}}$ coincides with the approximation error $R_{\ell_{0-1}, \mathcal{H}}^* -  R_{\ell_{0-1}, \mathcal{H}_{all}}^*$.
\end{lemma}
\begin{proof}
\begin{align*}
M_{\ell_{0-1}, \mathcal{H}} &= R_{\ell_{0-1}, \mathcal{H}}^* - \mathbb{E}_{\vx}[\mathscr{C}_{\ell_{0-1}, \mathcal{H}}^*(\vx)] \\
&=R_{\ell_{0-1}, \mathcal{H}}^* - \mathbb{E}_{\vx}[\inf_{\boldsymbol{h} \in \mathcal{H}} \mathscr{C}_{\ell_{0-1}}(\boldsymbol{h}, \vx, \vp(\vx)] \\
&=R_{\ell_{0-1}, \mathcal{H}}^* - \mathbb{E}_{\vx}[1 - \max_y p_y(\vx)] & (\text{Lemma~\ref{Prop: Character of conditional epsilon-regret}}) \\
&=R_{\ell_{0-1}, \mathcal{H}}^* -  R_{\ell_{0-1}, \mathcal{H}_{all}}^*.
\end{align*}
\end{proof}

\begin{lemma}
\label{lemma: t leq T}
Given $\vx$, and $\vp \in \Delta_K$, the following statements are equivalent:
\item (1)Optimation problem (\ref{thm6: inf h log loss}) can reach the global optimum,
\item (2)$\max_y p_y - \min_y p_y \leq \frac{\exp({W\Vert \vx\Vert + B}) - \exp{(-(W\Vert \vx\Vert + B))}}{\exp({W\Vert \vx\Vert + B}) + (K-1)\exp{(-(W\Vert \vx\Vert + B))}}$.
\end{lemma}
\begin{proof}
First, we prove that (1) implies (2). By the solutions of KKT conditions in~(\ref{eq: KKT solution}), (1) means that $\exists \vh \in \mathbb{R}^K$, $\vert h_i\vert \leq W\Vert x\Vert + B$, and $p_i = \frac{\exp(h_i)}{\sum_{j = 1}^K \exp(h_j)}$. We can directly write
\begin{align*}
\max_y p_y - \min_y p_y &= \frac{\max_y \exp(h_y) - \min_y \exp(h_y)}{\sum_{j = 1}^K \exp(h_j)}\\
&\leq \frac{\max_y \exp(h_y) - \min_y \exp(h_y)}{\max_y \exp(h_y) + (K-1)\min_y \exp(h_y)} \\
&\leq \frac{\exp({W\Vert \vx\Vert + B}) - \min_y\exp(h_y)}{\exp(W\Vert \vx\Vert + B) + (K-1)\min_y \exp(h_y)} & \text{(increasing w.r.t $\max_y \exp(h_y)$)}\\
&\leq \frac{\exp({W\Vert \vx\Vert + B}) - \exp(-(W\Vert \vx\Vert + B))}{\exp(W\Vert \vx\Vert + B) + (K-1) \exp(-(W\Vert \vx\Vert + B))} & \text{(decreasing w.r.t $\min_y \exp(h_y)$)}\\
\end{align*}

Second, we prove that (2) implies (1). We suppose that if (1) does not hold, then in this case, let $\mathcal{I}_2 = \{y:  p_y > \frac{\exp{(W\Vert \vx\Vert + B)}}{\sum_{k=1}^K \exp{(h_k^*)}}\}$, $\mathcal{I}_3 = \{y:  p_y < \frac{\exp{(-W\Vert \vx\Vert - B)}}{\sum_{k=1}^K \exp{(h_k^*)}}\}$ and $\mathcal{I}_1 = \{1, \dots, K\} - \mathcal{I}_2 - \mathcal{I}_3$. We note that $\#{\mathcal{I}_2}, \#{\mathcal{I}_3} > 0$. By (\ref{thm6: inf f KKT conditions}) we know that
\begin{equation}
\label{thm6: either I_2 or I_3 must be non-empty}
    \sum_{i=1}^K -p_i + \frac{\exp{(h_i^*)}}{\sum_{k=1}^K \exp{(h_k^*)}} + \lambda_i^* - \mu_i^* = \sum_{i=1}^K(\lambda_i^* - \mu_i^*) = 0
\end{equation}
Because either $\mathcal{I}_2$ or $\mathcal{I}_3$ must be non-empty. We assume that $\mathcal{I}_2$ is not empty, then there exists $y_1$ such that $\lambda_{y_1}^* - \mu_{y_1}^* = \lambda_{y_1}^* > 0$. To make (\ref{thm6: either I_2 or I_3 must be non-empty}) hold, there should exists $y_2$ such that $\lambda_{y_2}^* - \mu_{y_2}^* < 0$, which implies that $y_2 \in \mathcal{I}_3$. Thus, for any $\vh$, we have $\max_y p_y \ge \frac{\exp(W\Vert x\Vert + B)}{\sum_{i=j}^K \exp(h_j)}$ and $\min_y p_y \leq \frac{\exp(-(W\Vert x\Vert + B))}{\sum_{i=j}^K \exp(h_j)}$. Then $\max_y p_y - \min_y p_y \ge \frac{\exp(W\Vert x\Vert + B) - \exp(-(W\Vert x\Vert + B))}{\sum_{i=j}^K \exp(h_j)}$ for any $\vh$. Therefore, $\max_y p_y - \min_y p_y \ge \frac{\exp(W\Vert x\Vert + B) - \exp(-(W\Vert x\Vert + B))}{\exp({W\Vert \vx\Vert + B}) + (K-1)\exp{(-(W\Vert \vx\Vert + B))}}$, which leads to a confliction.
\end{proof}

\begin{lemma}
\label{lemma: inf H bar delta C}
For any $\hat{y} \ne y_{max}$, it holds that $\inf_{\boldsymbol{h} \in \mathcal{H}_{\hat{y}}(\vx)} \mathscr{C}_{\ell}(\boldsymbol{h}, \vx, \vp)) \ge -(p_{max} + p_{\hat{y}}) \log(\frac{p_{max} + p_{\hat{y}}}{2}) - \sum_{y \notin \{y_{max}, \hat{y}\}} p_y \log(p_y)$, where $y_{max} = \mathop{\mathrm{argmax}}_{y \in \mathcal{Y}} p_y$ and $\hat{y} = \mathop{\mathrm{argmax}}_{y \in \mathcal{Y}} h_y(\vx)$.
\end{lemma}
\begin{proof}
For all $\boldsymbol{h} \in \mathcal{H}$ and $\vx \in \mathcal{X}$, we have
\begin{align*}
    \mathscr{C}_{\ell}(\boldsymbol{h}, \vx, \vp)) = \sum_{y=1}^K p_y \ell(y, \boldsymbol{h}(\vx)) =  \sum_{y=1}^K p_y (-h_y + \log(\sum_{j=1}^K \exp{(h_j)})),
\end{align*}
where we use $h_j$ to denote $h_j(\vx)$ for simplicity. To get the $\inf_{\boldsymbol{h} \in \mathcal{H}_{\hat{y}}(\vx)} \mathscr{C}_{\ell}(\boldsymbol{h}, \vx, \vp))$, we consider the following problem
\begin{equation*}
\begin{split}
&\min_{{\boldsymbol{h} \in \mathcal{H}(\vx)}} \,\, \sum_{y=1}^K p_y (-h_y + \log(\sum_{j=1}^K \exp{(h_j)})),\\
&s.t.\quad  
\left\{\begin{array}{lc}
h_i - (W\Vert \vx\Vert + B) \leq 0, &\forall i,\\
-h_i - (W\Vert \vx\Vert + B) \leq 0, &\forall i,\\
h_{i} - h_{\hat{y}} \leq 0, & \forall i \ne \hat{y}.
\end{array}\right.
\end{split}
\end{equation*}
We drop some constraints, and consider another problem, whose optimum is lower than the above:
\begin{equation*}
\begin{split}
&\min_{{\boldsymbol{h} \in \mathcal{H}(\vx)}} \,\, \sum_{y=1}^K p_y (-h_y + \log(\sum_{j=1}^K \exp{(h_j)})),\\
&s.t.\quad  
h_{i} - h_{\hat{y}} \leq 0, \forall i \ne \hat{y}.
\end{split}
\end{equation*}
Due to the convexity of $\mathscr{C}_{\ell}(\boldsymbol{h}, \vx, \vp))$ by Lemma~\ref{lemma: Convexity of C}, we could write its KKT conditions to obtain the necessary conditions to reach the optimum. They are listed as follows:
\begin{align}
\left\{
\begin{aligned}
    &h_{i} - h_{\hat{y}} \leq 0, & \forall i \ne \hat{y}, \\
    &\lambda_i^* \ge 0, & \forall i \ne \hat{y},\\
    &\lambda_i^* (h_{i} - h_{\hat{y}}) = 0, &\forall i \ne \hat{y},\\
    & -p_i + \frac{\exp{(h_i^*)}}{\sum_{k=1}^K \exp{(h_k^*)}} + \lambda_i^* = 0, &\forall i \ne \hat{y},\\
    & -p_{\hat{y}} + \frac{\exp{(h_{\hat{y}}^*)}}{\sum_{k=1}^K \exp{(h_k^*)}}  - \sum_{i \ne \hat{y}} \lambda_i^* = 0.
\end{aligned}
\right.
\end{align}
If $\lambda_i^* = 0$ for all $i \in \{1,\dots,K\}$, then $h_i^* = \log(p_i(\sum_{k=1}^K \exp(h_k^*)))$. It means that $h_{y_{max}}^* \ge h_{\hat{y}}^*$, which conflicts with the precondition that $\hat{y} \ne y_{max}$. Thus, there exists a $y_m \ne \hat{y}$, $\lambda_{y_m}^* = 0$, and $h_{y_m}^* = h_{\hat{y}}$. It implies that
\begin{align*}
\left\{
\begin{aligned}
    & h_i^* = \log(p_i \sum_{k=1}^K \exp{(h_k^*)}), & i \notin \{y_{m}, \hat{y}\},\\
    & h_{\hat{y}}^* = h_{y_{m}}^*,\\
    & \frac{\exp{(h_{y_{m}}^*)} + \exp{(h_{\hat{y}}^*)}}{\sum_{k=1}^K \exp{(h_k^*)}} = p_{y_{m}} +  p_{\hat{y}}.
\end{aligned}
\right.
\end{align*}
Then we have $h_{\hat{y}}^* = h_{y_m}^* = \log( \frac{p_{y_m} + p_{\hat{y}}}{2} \sum_{k=1}^K \exp{(h_k^*)})$. If $y_m \ne y_{max}$, then $h_{\hat{y}}^* = \log( \frac{p_{y_m} + p_{\hat{y}}}{2} \sum_{k=1}^K \exp{(h_k^*)}) < \log( p_{max} \sum_{k=1}^K \exp{(h_k^*)}) = h_{y_{max}}^*$, which conflicts with $\hat{y} \ne y_{max}$. Thus, we conclude that $y_m = y_{max}$. We define  $s = \exp(h_{\hat{y}}^*) = \exp(h_{y_{max}}^*)$ for simplicity, then we can obtain that
\begin{align*}
    &\inf_{\boldsymbol{h} \in \mathcal{H}_{\hat{y}}(\vx)} \mathscr{C}_{\ell}(\boldsymbol{h}, \vx, \vp)) = \inf_{\boldsymbol{h} \in \mathcal{H}_{\hat{y}}(\vx)} \sum_{y=1}^K p_y \ell(\boldsymbol{h}(\vx), y)\\
    &=  \inf_{\boldsymbol{h} \in \mathcal{H}_{\hat{y}}(\vx)} \sum_{y=1}^K p_y (-h_y + \log(\sum_{j=1}^K \exp{(h_j)})) \\
    &\ge p_{y_{max}} (-h_{y_{max}}^* + \log(\sum_{j=1}^K \exp{(h_j^*)})) +  p_{\hat{y}} (-h_{\hat{y}}^* + \log(\sum_{j=1}^K \exp{(h_j^*)})) + \sum_{y \notin \{y_{max}, \hat{y}\}} p_y (-h_y^* + \log(\sum_{j=1}^K \exp{(h_j^*)}))\\
    &= - (p_{y_{max}} +  p_{\hat{y}}) (\log(s) - \log(\sum_{j=1}^K \exp{(h_j^*)})) - \sum_{y \notin \{y_{max}, \hat{y}\}} p_y \log(p_y)\\
    &= - (p_{y_{max}} +  p_{\hat{y}}) \log(\frac{s}{\sum_{j=1}^K \exp{(h_j^*)}}) - \sum_{y \notin \{y_{max}, \hat{y}\}} p_y \log(p_y)\\
    &= - (p_{y_{max}} +  p_{\hat{y}}) \log( \frac{\exp{(h_{y_{max}}^*)} + \exp{(h_{\hat{y}}^*)}}{2\sum_{k=1}^K \exp{(h_k^*)}}) - \sum_{y \notin \{y_{max}, \hat{y}\}} p_y \log(p_y)\\
    &= - (p_{y_{max}} +  p_{\hat{y}}) \log( \frac{p_{y_{max}} +  p_{\hat{y}}}{2}) - \sum_{y \notin \{y_{max}, \hat{y}\}} p_y \log(p_y),
\end{align*}
which completes the proof.
\end{proof}

\begin{lemma}[Technical lemma 1]
\label{lemma: technical lemma 8}
For all $x \in [0, 1-t]$ and fixed $t \in \mathbb{R}_+$, it holds that $-(2x + t) \log(\frac{2x + t}{2}) + (x + t)\log(x + t) + x\log(x) \ge u(1-t) = -(2-t)\log(\frac{2-t}{2}) + (1-t)\log(1-t)$.
\end{lemma}
\begin{proof}
We first prove that $u(x) =-(2x + t) \log(\frac{2x + t}{2}) + (x + t)\log(x + t) + t\log(t)$ is decreasing on $x \in [0,1]$, which could be obtained by $\frac{d u}{d x} = \log(\frac{4x(x+t)}{(2x+t)^2}) \leq 0$. Thus we have $u(x) \ge u(1-t) = -(2-t)\log(\frac{2-t}{2}) + (1-t)\log(1-t)$, which complete the proof.
\end{proof}

\begin{lemma}[Technical lemma 2]
\label{lemma: technical lemma 9}
For all $t \in [0, 1]$, it holds that $\frac{1+t}{2} \log(1+t) + \frac{1-t}{2} \log(1-t) \ge \frac{t^2}{2}$.
\end{lemma}
\begin{proof}
We define $u(t) = \frac{1+t}{2} \log(1+t) + \frac{1-t}{2} \log(1-t) - \frac{t^2}{2}$. Then we calculate $\frac{d u}{d t} = \frac{1}{2}\log(\frac{1+t}{1-t}) - t$ and $\frac{d^2 u}{dt^2} = \frac{1}{1-t^2} - 1 \ge 0$. We have $\frac{d u}{d t} = \frac{1}{2}\log(\frac{1+t}{1-t}) - t \ge \frac{1}{2}\log(\frac{1+0}{1-0}) - 0 = 0$. Thus, $u(x)$ is increasing on $[0,1]$ and $u(x) \ge u(0) = 0$, which proves the lemma.
\end{proof}

We now are ready to prove the Theorem~\ref{thm: H-consistency bound for log} as follows.

\begin{proof}
We first rewrite the $\mathcal{J}_{\ell}(t)$ as follows.
\begin{align*}
    \mathcal{J}_{\ell}(t) &= \inf_{\hat{y} \in \mathcal{Y}, \vp \in \mathcal{P}_{\hat{y}}(t), \vx \in \mathcal{X}, \boldsymbol{h} \in \mathcal{H}_{\hat{y}}(\vx) }\Delta \mathscr{C}_{\ell, \mathcal{H}}(\boldsymbol{h}, \vx, \vp)\\
    &= \inf_{\hat{y} \in \mathcal{Y}} \inf_{\vp \in \mathcal{P}_{\hat{y}}(t)} \inf_{\vx \in \mathcal{X}} \inf_{\boldsymbol{h} \in \mathcal{H}_{\hat{y}}(\vx) }\Delta \mathscr{C}_{\ell, \mathcal{H}}(\boldsymbol{h}, \vx, \vp)\\
    &= \inf_{\hat{y} \in \mathcal{Y}} \inf_{\vp \in \mathcal{P}_{\hat{y}}(t)} \inf_{\vx \in \mathcal{X}} \inf_{\boldsymbol{h} \in \mathcal{H}_{\hat{y}}(\vx)} (\mathscr{C}_{\ell}(\boldsymbol{h}, \vx, \vp) - \inf_{\boldsymbol{h} \in \mathcal{H}} \mathscr{C}_{\ell}(\boldsymbol{h}, \vx, \vp))\\
    &= \inf_{\hat{y} \in \mathcal{Y}} \inf_{\vp \in \mathcal{P}_{\hat{y}}(t)} \inf_{\vx \in \mathcal{X}}( \inf_{\boldsymbol{h} \in \mathcal{H}_{\hat{y}}(\vx)} \mathscr{C}_{\ell}(\boldsymbol{h}, \vx, \vp) - \inf_{\boldsymbol{h} \in \mathcal{H}} \mathscr{C}_{\ell}(\boldsymbol{h}, \vx, \vp)).
\end{align*}
For all $\boldsymbol{h} \in \mathcal{H}$ and $\vx \in \mathcal{X}$, we have
\begin{align*}
    \mathscr{C}_{\ell}(\boldsymbol{h}, \vx, \vp)) = \sum_{y=1}^K p_y \ell(y, \boldsymbol{h}(x)) =  \sum_{y=1}^K p_y (-h_y + \log(\sum_{j=1}^K \exp{(h_j)})).
\end{align*}

To get the $\inf_{\boldsymbol{h} \in \mathcal{H}} \mathscr{C}_{\ell}(\boldsymbol{h}, \vx, \vp))$, we consider the following problem
\begin{equation}
\label{thm6: inf h log loss}
\begin{split}
&\min_{\boldsymbol{h}} \,\, \sum_{y=1}^K p_y (-h_y + \log(\sum_{j=1}^K \exp{(h_j)})), \\
&s.t.\quad  
\left\{\begin{array}{lc}
h_i - (W\Vert \vx\Vert + B) \leq 0, \forall i,\\
-h_i - (W\Vert \vx\Vert + B) \leq 0, \forall i.
\end{array}\right.
\end{split}
\end{equation}
By Lemma~\ref{lemma: Convexity of C}, we know that this problem is convex, we can make use of KKT conditions~\cite{boyd2004convex} to find the points that are primal and dual optimal, which can be listed as follows
\begin{align}
\label{thm6: inf f KKT conditions}
\left\{
\begin{aligned}
    &h_i^* - (W\Vert \vx\Vert + B) \leq 0, & i = 1, \dots, K,\\
    &-h_i^* - (W\Vert \vx\Vert + B) \leq 0,  & i = 1, \dots, K,\\
    &\lambda_i^* \ge 0 , \mu_i^* \ge 0, & i = 1, \dots, K,\\
    &\lambda_i^* (h_i^* - W\Vert \vx\Vert - B) = 0, & i = 1, \dots, K,\\
    &\mu_i^* (-h_i^* - W\Vert \vx\Vert - B) = 0, & i = 1, \dots, K,\\
    & -p_i + \frac{\exp{(h_i^*)}}{\sum_{k=1}^K \exp{(h_k^*)}} + \lambda_i^* - \mu_i^* = 0, & i = 1, \dots, K.
\end{aligned}
\right.
\end{align}
It implies that
\begin{align}
\label{eq: KKT solution}
\left\{
\begin{aligned}
    &h_i^* = W\Vert \vx\Vert + B, & p_i \ge \frac{\exp{(W\Vert \vx\Vert + B)}}{\sum_{k=1}^K \exp{(h_k^*)}},\\
    &h_i^* = -(W\Vert \vx\Vert + B), & p_i \leq \frac{\exp{(-(W\Vert \vx\Vert + B))}}{\sum_{k=1}^K \exp{(h_k^*)}},\\
    & h_i^* = \log(p_i \sum_{k=1}^K \exp{(h_k^*)}), & \text{otherwise.}
\end{aligned}
\right.
\end{align}
By the precondition of Theorem~\ref{thm: H-consistency bound for log}, we have
\begin{equation}
\label{thm6: t case 1}
t = p_{max} - p_{\hat{y}} \leq p_{max} - p_{min} \leq \frac{\exp({B}) - \exp{(-B)}}{\exp({B}) + (K-1)\exp{(-B)}} \leq \frac{\exp({W\Vert \vx\Vert + B}) - \exp{(-(W\Vert \vx\Vert + B))}}{\exp({W\Vert \vx\Vert + B}) + (K-1)\exp{(-(W\Vert \vx\Vert + B))}}.
\end{equation}
In addition, in this case, by Lemma~\ref{lemma: t leq T}, the global optimum could be reached, so we can omit the boundary situation
\begin{align*}
\inf_{\boldsymbol{h} \in \mathcal{H}} \mathscr{C}_{\ell}(\boldsymbol{h}, \vx, \vp)) = \sum_{y=1}^K p_y (-h^*_y + \log(\sum_{j=1}^K \exp{(h^*_j)})) = - \sum_{y=1}^K p_y \log(p_y),
\end{align*}
which is the entropy of distribution $\vp$. Denote the index of the largest element of $\vp$ by $y_{max}$. When $t > 0$, because $\vp_{y_{max}} - \vp_{\hat{y}} = t > 0$, then $y_{max} \ne \hat{y}$. By Lemma~\ref{lemma: inf H bar delta C}, we know that
\begin{align*}
\inf_{\boldsymbol{h} \in \mathcal{H}_{\hat{y}}(\vx) } \mathscr{C}_{\ell}(\boldsymbol{h}, \vx, \vp)) &\ge -(p_{y_{max}} + p_{\hat{y}}) \log(\frac{p_{y_{max}} + p_{\hat{y}}}{2}) - \sum_{y \notin \{y_{max}, \hat{y}\}}p_y \log(p_y).
\end{align*}
Then we have
\begin{align}
\label{eq: global optimum}
     &\inf_{\boldsymbol{h} \in \mathcal{H}_{\hat{y}}(\vx)} \mathscr{C}_{\ell}(\boldsymbol{h}, \vx, \vp) - \inf_{\boldsymbol{h} \in \mathcal{H}} \mathscr{C}_{\ell}(\boldsymbol{h}, \vx, \vp) \ge  -(p_{y_{max}} + p_{\hat{y}}) \log(\frac{p_{y_{max}} + p_{\hat{y}}}{2}) + p_{y_{max}} \log(p_{y_{max}}) + p_{\hat{y}} \log(p_{\hat{y}}).
\end{align}
To make (\ref{thm6: t case 1}) holds for all $\vx \in \mathcal{X}$, we need $t \leq \min_{\vx} \frac{\exp({W\Vert \vx\Vert + B}) - \exp{(-(W\Vert \vx\Vert + B))}}{\exp({W\Vert \vx\Vert + B}) + (K-1)\exp{(-(W\Vert \vx\Vert + B))}} = \frac{\exp({B}) - \exp{(-B)}}{\exp({B}) + (K-1)\exp{(-B)}}$. Then, in this case, we can take infimum with regard to $\vx$ as follows. 
\begin{align*}
    &\inf_{\vx \in \mathcal{X}} (\inf_{\boldsymbol{h} \in \mathcal{H}_{\hat{y}}(\vx)} (\mathscr{C}_{\ell}(\boldsymbol{h}, \vx, \vp) - \inf_{\boldsymbol{h} \in \mathcal{H}} \mathscr{C}_{\ell}(\boldsymbol{h}, \vx, \vp)) )\ge -(p_{y_{max}} + p_{\hat{y}}) \log(\frac{p_{y_{max}} + p_{\hat{y}}}{2}) + p_{y_{max}} \log(p_{y_{max}}) + p_{p_{\hat{y}}} \log(p_{p_{\hat{y}}}).
\end{align*}
Now, we meet the following problem
\begin{equation*}
\begin{split}
&\min_{\vp} \,\, -(p_{y_{max}} + p_{\hat{y}}) \log(\frac{p_{y_{max}} + p_{\hat{y}}}{2}) + p_{y_{max}} \log(p_{y_{max}}) + p_{\hat{y}} \log(p_{\hat{y}}),\\
&s.t.\quad  
\left\{\begin{array}{lc}
p_{y_{max}} - p_{\hat{y}} = t, \forall i.\\
\sum_{i=1}^K p_i = 1, \\
p_i \ge 0, \forall i,
\end{array}\right.
\end{split}
\end{equation*}
which is equivalent to find the minimum of $-(2p_{\hat{y}} + t) \log(\frac{2p_{\hat{y}} + t}{2}) + (p_{\hat{y}} + t) \log((p_{\hat{y}} + t)) +p_{\hat{y}} \log(p_{\hat{y}})$ when $p_{\hat{y}} \in [0, \frac{1-t}{2}]$. By Lemma~\ref{lemma: technical lemma 8}, we know it is $\frac{1+t}{2} \log(1+t) + \frac{1-t}{2} \log(1-t)$. Thus
\begin{align*}
    \mathcal{J}_{\ell}(t) &= \inf_{\hat{y} \in \mathcal{Y}} \inf_{\vp \in \mathcal{P}_{\hat{y}}(t)} \inf_{\vx \in \mathcal{X}}( \inf_{\boldsymbol{h} \in \mathcal{H}_{\hat{y}}(\vx)} \mathscr{C}_{\ell}(\boldsymbol{h}, \vx, \vp) - \inf_{\boldsymbol{h} \in \mathcal{H}} \mathscr{C}_{\ell}(\boldsymbol{h}, \vx, \vp))\\
    &\ge \inf_{\hat{y} \in \mathcal{Y}} -(2-t)\log(\frac{2-t}{2}) + (1-t)\log(1-t)\\
    &= \frac{1+t}{2} \log(1+t) + \frac{1-t}{2} \log(1-t)\\
    &\ge \frac{t^2}{2}. & (\text{Lemma~\ref{lemma: technical lemma 9}})
\end{align*}
It is worthwhile to note that when the number of classes is 2, then the derivations and results above coincide with that in binary case~\cite{DBLP:conf/icml/AwasthiMM022}. Let $g(t) = \frac{t^2}{2}$ in Theorem~\ref{cor: Distribution-independent convex Psi bound}, we have
\begin{align*}
\frac{1}{2} (R_{\ell_{0-1}}(\boldsymbol{h}) - R^*_{\ell_{0-1}, \mathcal{H}} + M_{\ell_{0-1}, \mathcal{H}})^2 \leq R_{\ell_{log}}(\boldsymbol{h}) - R^*_{\ell_{log}, \mathcal{H}} + M_{\ell_{log}, \mathcal{H}},
\end{align*}
which implies
\begin{align*}
R_{\ell_{0-1}}(\boldsymbol{h}) - R^*_{\ell_{0-1}, \mathcal{H}} + M_{\ell_{0-1}, \mathcal{H}}  \leq \sqrt{2}(R_{\ell_{log}}(\boldsymbol{h}) - R^*_{\ell_{log}, \mathcal{H}} + M_{\ell_{log}, \mathcal{H}})^{\frac{1}{2}},
\end{align*}
when $R_{\ell_{log}}(\boldsymbol{h}) - R^*_{\ell_{log}, \mathcal{H}} + M_{\ell_{log}, \mathcal{H}} \leq \frac{1}{2} (\frac{\exp({2B}) - 1}{\exp({2B}) + K-1})^2$. By Lemma~\ref{lemma: property of M_Psi_F}, we have  $M_{\ell_{0-1}, \mathcal{H}}$ coincides with the approximation error $R_{\ell_{0-1}, \mathcal{H}}^* -  R_{\ell_{0-1}, \mathcal{H}_{all}}^*$. We also note that $M_{\ell_{log}, \mathcal{H}}$ coincides with $R_{\ell_{log}, \mathcal{H}}^* -  R_{\ell_{log}, \mathcal{H}_{all}}^*$ because
\begin{align*}
M_{\ell_{log}, \mathcal{H}} &= R_{\ell_{log}, \mathcal{H}}^* - \mathbb{E}_{\vx}[\mathscr{C}_{\ell_{log}, \mathcal{H}}^*(\vx)] \\
&=R_{\ell_{log}, \mathcal{H}}^* - \mathbb{E}_{\vx}[\inf_{\boldsymbol{h} \in \mathcal{H}} \mathscr{C}_{\ell_{log}}(\boldsymbol{h}, \vx, \vp(\vx)] \\
&=R_{\ell_{log}, \mathcal{H}}^* - \mathbb{E}_{\vx}[- \sum_{y=1}^K p_y(\vx) \log(p_y(\vx))] \\
&=R_{\ell_{log}, \mathcal{H}}^* -  R_{\ell_{log}, \mathcal{H}_{all}}^*.\\
\end{align*}
Finally, we can conclude that
\begin{align*}
R_{\ell_{0-1}}(\boldsymbol{h}) - R^*_{\ell_{0-1}, \mathcal{H}} \leq R_{\ell_{0-1}}(\boldsymbol{h}) - R^*_{\ell_{0-1}, \mathcal{H}} + M_{\ell_{0-1}, \mathcal{H}}  \leq \sqrt{2}(R_{\ell_{log}}(\boldsymbol{h}) - R^*_{\ell_{log}, \mathcal{H}_{all}})^\frac{1}{2}.
\end{align*}
\end{proof}

\section{Proofs of Appendix~\ref{sec: Discriminative vs. Generative: Binary Classification}}
\label{Proofs of sec: Discriminative vs. Generative: Binary Classification}

\subsection{Proof of Proposition~\ref{Prop: logisticbound}}
\label{proof: Proposition logisticbound}
We first present the following lemmas to show Proposition~\ref{Prop: logisticbound}.

\begin{lemma}[\cite{mohri2018foundations}, Lemma 5.7, Talagrand’s lemma]
\label{lemma: talagrand}
Let $\Phi$ be L-Lipschitz functions from $\mathbb{R} \to \mathbb{R}$
and $\sigma_1, \dots, \sigma_m$ be Rademacher random variables. Then, for any hypothesis set $\mathcal{H}$ of real-valued functions, the following inequality holds:
\begin{equation*}
    \mathcal{R}_m(\Phi \circ \mathcal{H}) \leq L\mathcal{R}_m(\mathcal{H}).
\end{equation*}
\end{lemma}

\begin{lemma}[Rademacher complexity of constrained linear hypotheses]
\label{lemma: Rademacher complexity of constrained linear hypotheses}
Let $S = \{x_1, \dots, x_m\}$ where $x_i \in [0,1]$ for all $i \in \{1,\dots,m\}$ and $\mathcal{H} = \{x \to \langle w, x\rangle + b: \Vert w \Vert_2 \leq W, \vert b\vert \leq B\}$. Then, the Rademacher complexity of $\mathcal{H}$ can be bounded
as follows:
\begin{equation*}
    \mathcal{R}_m(\mathcal{H}) \leq W \sqrt{\frac{n}{m}}.
\end{equation*}
\end{lemma}
\begin{proof}
\begin{align*}
    \mathcal{R}_m(\mathcal{H}) &= \frac{1}{m} \mathbb{E}_{S, \sigma}[\sup_h \sum_{i=1}^m \sigma_i h(x_i)] = \frac{1}{m} \mathbb{E}_{S, \sigma}[\sup_h \sum_{i=1}^m \sigma_i (\langle w, x_i\rangle + b)] \\
    &= \frac{1}{m} \mathbb{E}_{S, \sigma}[\sup_h \sum_{i=1}^m \sigma_i \langle w, x_i\rangle + b\sum_{i=1}^m \sigma_i] \leq \frac{1}{m} \mathbb{E}_{S, \sigma}[\sup_{w} \sum_{i=1}^m \sigma_i \langle w, x_i\rangle + \sup_b b\sum_{i=1}^m \sigma_i]\\
    &=\frac{1}{m} \mathbb{E}_{S, \sigma}[\sup_{w} \sum_{i=1}^m \sigma_i \langle w, x_i\rangle] = \frac{1}{m} \mathbb{E}_{S, \sigma}[\sup_{w} \langle w, \sum_{i=1}^m \sigma_i x_i\rangle]\\
    &\leq \frac{W}{m} \mathbb{E}_{S, \sigma}[\Vert \sum_{i=1}^m \sigma_i x_i\Vert_2] \leq \frac{W}{m} \sqrt{ \mathbb{E}_{S, \sigma}[\Vert \sum_{i=1}^m \sigma_i x_i\Vert_2^2]}\\
    &= \frac{W}{m} \sqrt{ \mathbb{E}_{S, \sigma}[ \sum_{i,j=1}^m \sigma_i \sigma_j \langle x_i, x_j\rangle]} = \frac{W}{m} \sqrt{\sum_{i=1}^m \Vert x_i\Vert_2^2} \leq \frac{W}{m} \sqrt{m \times n}\\
    &= W\sqrt{\frac{n}{m}}.
\end{align*}
\end{proof}

\begin{lemma}[Rademacher complexity of $\widetilde{\mathcal{H}}$]
\label{lemma: Rademacher complexity of yH}
Let $\widetilde{\mathcal{H}} = \{z = (\vx,y) \to yh(x): h \in \mathcal{H}\}$. Then, the Rademacher complexity of $\widetilde{\mathcal{H}}$ satisfies:
\begin{equation*}
    \mathcal{R}_m(\widetilde{\mathcal{H}}) = \mathcal{R}_m(\mathcal{H}).
\end{equation*}
\end{lemma}
\begin{proof}
\begin{align*}
    \mathcal{R}_m(\widetilde{\mathcal{H}}) &= \frac{1}{m} \mathbb{E}_{S, \sigma}[\sup_h \sum_{i=1}^m \sigma_i y_ih(x_i)] \\
    &= \frac{1}{2m} \mathbb{E}_{S, \sigma}[\sup_h \sum_{i=1}^m \sigma_i (2y_i - 1)h(x_i) + \sum_{i=1}^m \sigma_i h(x_i)]\\
    &= \frac{1}{2m} \mathbb{E}_{S, \sigma}[\sup_h \sum_{i=1}^m \sigma_i h(x_i) + \sum_{i=1}^m \sigma_i h(x_i)] &  (2y_i-1 \in \{-1,+1\})\\
    &= \frac{1}{m} \mathbb{E}_{S, \sigma}[\sup_h \sum_{i=1}^m \sigma_i h(x_i)]\\
    &= \mathcal{R}_m(\mathcal{H}).
\end{align*}
\end{proof}

We now prove Proposition~\ref{Prop: logisticbound} by using the above lemmas.

\begin{proof}
We first rewrite the $R_{\ell_{log}}(h_{Dis, m}) - R_{\ell_{log}}(h_{Dis, \infty})$.
\begin{align*}
     & R_{\ell_{log}}(h_{Dis, m}) - R_{\ell_{log}}(h_{Dis, \infty}) \\
     &= R_{\ell_{log}}(h_{Dis, m}) - \hat{R}_{\ell_{log}, S}(h_{Dis, m}) + \hat{R}_{\ell_{log}, S}(h_{Dis, m}) -\hat{R}_{\ell_{log}, S}(h_{Dis, \infty}) +  \hat{R}_{\ell_{log}, S}(h_{Dis, \infty}) - R_{\ell_{log}}(h_{Dis, \infty})  \\
    & \leq  (R_{\ell_{log}}(h_{Dis, m}) - \hat{R}_{\ell_{log}, S}(h_{Dis, m})) +  (\hat{R}_{\ell_{log}, S}(h_{Dis, \infty}) - R_{\ell_{log}}(h_{Dis, \infty}).
\end{align*}
The first summand on the right-hand side can be bounded by making use of Lemma~\ref{lemma: rademacher},\ref{lemma: talagrand},\ref{lemma: Rademacher complexity of constrained linear hypotheses} and~\ref{lemma: Rademacher complexity of yH} in sequence. Let $\widetilde{\mathcal{H}} = \{z = (\vx,y) \to yh(x): h \in \mathcal{H}\}$ and $\Phi = \{\ell_{log} \circ \widetilde{h}: \widetilde{h} \in \widetilde{\mathcal{H}} \}$ With probability of at least $1 - \delta$, we have:
\begin{align*}
     &R_{\ell_{log}}(h_{Dis, m}) - \hat{R}_{\ell_{log}, S}(h_{Dis, m}) \\
     &\leq 2\mathcal{R}_m(\ell_{log} \circ \widetilde{\mathcal{H}}) + \log(1 + \exp{(W\sqrt{n} + B)}) \sqrt{\frac{1}{2m} \log(\frac{2}{\delta})} & (\ell_{log} \circ \widetilde{\mathcal{H}} \text{ is bounded, Lemma~\ref{lemma: rademacher}})\\
    & \leq 2\mathcal{R}_m(\widetilde{\mathcal{H}}) + \log(1 + \exp{(W\sqrt{n} + B)}) \sqrt{\frac{1}{2m} \log(\frac{2}{\delta})} & (\ell_{log} \text{ is 1-Lipschitz,  Lemma~\ref{lemma: talagrand}}) \\
    & = 2\mathcal{R}_m(\mathcal{H}) + \log(1 + \exp{(W\sqrt{n} + B)}) \sqrt{\frac{1}{2m} \log(\frac{2}{\delta})} & (\text{by Lemma~\ref{lemma: Rademacher complexity of yH}}) \\
    & \leq 2W\sqrt{\frac{n}{m}} +\log(1 + \exp{(W\sqrt{n} + B)}) \sqrt{\frac{1}{2m} \log(\frac{2}{\delta})}  & (\text{by Lemma~\ref{lemma: Rademacher complexity of constrained linear hypotheses}}).
\end{align*}
For the second summand, we use the fact that $R_{\ell_{log}}(h_{Dis, \infty})$ does not depend on sampled training dataset $S$; hence by Lemma~\ref{lemma: hoeffding}, we obtain its bound:
\begin{align*}
    \mathbb{P}(\vert \hat{R}_{\ell_{log}, S}(h_{Dis, \infty}) - R_{\ell_{log}}(h_{Dis, \infty}) \vert > \epsilon) \leq 2\exp({-\frac{2m\epsilon^2}{(c-0)^2}}) = 2\exp({-\frac{2m\epsilon^2}{c^2}}),
\end{align*}
where $c = \log(1 + \exp{(W\sqrt{n} + B)})$. It implies that with the probability of at least $1-\delta$, we have:
\begin{equation*}
    \hat{R}_{\ell_{log}, S}(h_{Dis, \infty}) - R_{\ell_{log}}(h_{Dis, \infty}) \leq c\sqrt{\frac{1}{2m} \log(\frac{2}{\delta})}.
\end{equation*}
At last, we use the union bound to get the final result. With probability at least $1-\delta$, the following holds:
\begin{align*}
     &R_{\ell_{log}}(h_{Dis, m}) -  R_{\ell_{log}}(h_{Dis, \infty}) \\
     &\leq 2W\sqrt{\frac{n}{m}} +\log(1 + \exp{(W\sqrt{n} + B)}) \sqrt{\frac{1}{2m} \log(\frac{4}{\delta})} + c\sqrt{\frac{1}{2m} \log(\frac{4}{\delta})}\\
    &= 2W\sqrt{\frac{n}{m}} + (c + \log(1 + \exp{(W\sqrt{n} + B)})) \sqrt{\frac{1}{2m} \log(\frac{4}{\delta})} \\
    &= 2W\sqrt{\frac{n}{m}} +  + 2\log(1 + \exp{(W\sqrt{n} + B)}) \sqrt{\frac{1}{2m} \log(\frac{4}{\delta})} \\
    & = O(\sqrt{\frac{n}{m}}).
\end{align*}
Therefore, for $R_{\ell_{log}}(h_{Dis, m}) \leq R_{\ell_{log}}(h_{Dis, \infty}) +\epsilon_0$ to hold with high probability $1-\delta_0$ (here, $\epsilon_0$ and $\delta_0$ are some fixed constant in $[0,1]$), it suffices to pick $m = O(n)$ samples.
\end{proof}

\subsection{Proof of Theorem~\ref{cor: sample complexity of binary lr}}
\label{proof: Proof of Corollary cor: sample complexity of binary lr}


\begin{proof}
By Theorem~\ref{lemma: binary H consistency bounds} we know that for $R_{\ell_{0-1}}(h_{Dis, m}) \leq R_{\ell_{log}}(h_{Dis, \infty}) +\epsilon_0$, it is sufficient to ensure that $R_{log}(h_{Dis, m}) \leq R_{\ell_{log}}(h_{Dis, \infty}) +\frac{1}{2}\epsilon_0^2$. Then by Proposition~\ref{Prop: logisticbound}, it suffices to sample $m = O(\frac{n}{\epsilon_0^4}) = O(n)$.
\end{proof}

\subsection{Proof of Theorem~\ref{Thm: generalization bound}}
\label{proof: thm generalization bound}
To show Theorem~\ref{Thm: generalization bound}, we first present the following lemmas.
\begin{lemma}\label{Prop: parameters}
In terms of binary na\"ive Bayes, let any $\epsilon, \delta > 0$ and any Laplace smoothing parameter $\alpha \ge 0$ be fixed. Assume that Assumption~\ref{Assumption: p(y=k)} holds. Let $m = O((\frac{1}{\epsilon^2}) log(\frac{n}{\delta}))$, then with the probability of at least $1 - \delta$:
\begin{enumerate}
\item In case of discrete inputs, $\vert \hat{p}(x_i\vert y=k) - p(x_i\vert y=k)\vert \leq \epsilon$ and $\vert \hat{p}(y=k) - p(y=k)\vert \leq \epsilon$ for all $i \in \{1,\dots n\}$ and $k \in \{0,1\}$.
\item In case of continuous inputs, $\vert \hat{\mu}_{ki} - {\mu}_{ki}\vert \leq \epsilon$,  $\vert \hat{\sigma}^2_{i} - {\sigma}^2_{i}\vert \leq \epsilon$ and $\vert \hat{p}(y=k) - p(y=k)\vert \leq \epsilon$ for all $i \in \{1,\dots n\}$ and $k \in \{0,1\}$.
\end{enumerate}
\end{lemma}
\begin{proof}
First, we consider the discrete case, and let $\alpha = 0$ for now. Let $\epsilon \leq \rho_0/2$. By the Lemma~\ref{lemma: hoeffding}, with probability at least $1 - \delta_1 = 1 - 2\exp({-2m\epsilon^2})$ we have $\vert \hat{p}(y=k) - p(y=k)\vert \leq \epsilon$. It implies that  $\hat{p}(y=k) \ge p(y=k) - \epsilon \ge \rho_0 - \epsilon = \gamma = \Omega(1)$. So $\#\{y=k\} \ge \gamma m$  with probability at least $1 - \delta_1$.To bound the $\vert \hat{p}(x_i\vert y=k) - p(x_i\vert y=k)\vert$, for fixed $i,k$, the following holds:
\begin{align*}
     &\mathbb{P}[\vert \hat{p}(x_i\vert y=k) - p(x_i\vert y=k)\vert > \epsilon] \\
     &= \mathbb{P}(\vert \hat{p}(x_i\vert y=k) - p(x_i\vert y=k)\vert > \epsilon \vert \#\{y=k\} \ge \gamma m) \mathbb{P}(\#\{y=k\} \ge \gamma m) \\
     &+ \mathbb{P}(\vert \hat{p}(x_i\vert y=k) - p(x_i\vert y=k)\vert > \epsilon \vert \#\{y=k\} < \gamma m) \mathbb{P}(\#\{y=k\} < \gamma m)\\
    &\leq  2\exp({-2\epsilon^2 \#\{y=k\}})\vert_{ \#\{y=k\} \ge \gamma m} + \delta_1\\
    & \leq 2\exp({-2\epsilon^2 \gamma m}) + \delta_1 = \delta_2.
\end{align*}
Then we use the union bound to get the first result on the condition that $\alpha = 0$:
\begin{align*}
    &\mathbb{P}(\cup_{k=0}^1(\vert \hat{p}(y=k) - p(y=k)\vert > \epsilon) \cup (\cup_{i=1}^n \cup_{k=0}^1  \vert \hat{p}(x_i\vert y=k) - p(x_i\vert y=k)\vert > \epsilon)) \\
     &=\mathbb{P}((\vert \hat{p}(y=1) - p(y=1)\vert > \epsilon) \cup (\cup_{i=1}^n \cup_{k=0}^1  \vert \hat{p}(x_i\vert y=k) - p(x_i\vert y=k)\vert > \epsilon)) \\
     &\leq \delta_1 + 2n \delta_2 \\
     & = 2\exp({-2m\epsilon^2}) + 2n(2\exp({-2\epsilon^2 \gamma m}) + \delta_1) \\
     & = (2n+1)2\exp({-2m\epsilon^2}) + 2n \times 2 \exp({-2\epsilon^2 \gamma m}) \\
     & \leq 2(4n+1)\exp({-2 \gamma m} \epsilon^2).
\end{align*}
Therefore, for Lemma~\ref{Prop: parameters}.1 to hold with probability at least $1-\delta$, it suffices to pick $m$ samples that
\begin{equation*}
m = \frac{1}{2\gamma \epsilon^2} \log(\frac{2(4n+1)}{\delta}) \leq \frac{1}{\rho_0 \epsilon^2} \log(\frac{2(4n+1)}{\delta}) = O(\frac{1}{\epsilon^2} \log(\frac{n}{\delta})).
\end{equation*}
Second, we consider the discrete case, and let $\alpha > 0$. To bound $\vert \hat{p}(y=k) - p(y=k)\vert$, we calculate it based on the above condition as follows:
\begin{align*}
     &\mathbb{P}(\vert \hat{p}(y=k) - p(y=k)\vert > \epsilon) \\
     &= \mathbb{P}(\vert \hat{p}(y=k) - \hat{p}(y=k)\vert_{\alpha = 0} + \hat{p}(y=k)\vert_{\alpha = 0} - p(y=k)\vert > \epsilon) \\
     &\leq \mathbb{P}(\vert \hat{p}(y=k) - \hat{p}(y=k)\vert_{\alpha = 0}\vert + \vert \hat{p}(y=k)\vert_{\alpha = 0} - p(y=k)\vert > \epsilon),
\end{align*}
where the $\vert \hat{p}(y=k)\vert_{\alpha = 0} - p(y=k)\vert$ has been discussed above, so we only need to bound $\vert \hat{p}(y=k) - \hat{p}(y=k)\vert_{\alpha = 0}\vert$. We have,
\begin{equation*}
     \vert \hat{p}(y=k) - \hat{p}(y=k)\vert_{\alpha = 0}\vert = \vert \frac{\#\{y = k\} + \alpha}{m + 2\alpha} - \frac{\#\{y = k\}}{m }\vert = \vert\frac{\alpha(m - \#\{y = k\})}{m(m+2\alpha)}\vert = O(\frac{1}{m}).
\end{equation*}
So
\begin{align*}
     \mathbb{P}(\vert \hat{p}(y=k) - p(y=k)\vert > \epsilon) &\leq \mathbb{P}(\vert \hat{p}(y=k)\vert_{\alpha = 0} - p(y=k)\vert > \epsilon - O(\frac{1}{m})) \\
     &\leq 2\exp({-2m(\epsilon - O(\frac{1}{m}))^2}) = \delta_1.
\end{align*}
In the same way, we can write
\begin{equation*}
     \vert \hat{p}(x_i\vert y=k) - \hat{p}(x_i\vert y=k)\vert_{\alpha = 0}\vert = \frac{\alpha(\#\{y = k\} - \vert\#\{x_i, y = k\})}{\#\{y = k\}(\#\{y = k\}+2\alpha)}\vert = O(\frac{1}{\#\{y = k\}}) = O(\frac{1}{m}),
\end{equation*}
and
\begin{align*}
     \mathbb{P}(\vert \hat{p}(x_i\vert y=k) - p(x_i\vert y=k)\vert > \epsilon) &\leq \mathbb{P}(\vert \hat{p}(x_i\vert y=k)\vert_{\alpha = 0} - p(x_i\vert y=k)\vert > \epsilon - O(\frac{1}{m})) \\
     &\leq \delta_1 + 2\exp({-2\gamma m(\epsilon - O(\frac{1}{m}))^2}).
\end{align*}
Besides,
\begin{equation*}
m = \frac{1}{2\gamma (\epsilon - O(\frac{1}{m}))^2} \log(\frac{2(4n+1)}{\delta}) \leq \frac{1}{\rho_0 (\epsilon - O(\frac{1}{m}))^2} \log(\frac{2(4n+1)}{\delta}) = O(\frac{1}{\epsilon^2} \log(\frac{n}{\delta})).
\end{equation*}
In the following proofs, we will not consider Laplace smoothing anymore due to its small influence on the results.

Third, we consider the continuous case. In the same way as discrete case, with probability at least $1 - \delta_1 = 1 - 2\exp({-2m\epsilon^2})$ we have $\vert \hat{p}(y=k) - p(y=k)\vert \leq \epsilon$, and $\#\{y=k\} \ge \gamma m$. We only need to bound $\vert \hat{\mu}_{ki} - {\mu}_{ki}\vert$ and $\vert \hat{\sigma}^2_{i} - {\sigma^2_i}\vert$. Fix $i,k$, the following holds:
\begin{align*}
     \mathbb{P}[\vert\hat{\mu}_{ki} - {\mu}_{ki}\vert > \epsilon]
     &= \mathbb{P}(\vert \hat{\mu}_{ki} - {\mu}_{ki}\vert > \epsilon \vert \#\{y=k\} \ge \gamma m) \mathbb{P}(\#\{y=k\} \ge \gamma m) \\
     &+ \mathbb{P}(\vert \hat{\mu}_{ki} - {\mu}_{ki}\vert > \epsilon \vert \#\{y=k\} < \gamma m) \mathbb{P}(\#\{y=k\} < \gamma m)\\
    &\leq  2\exp({-2\epsilon^2 \#\{y=k\}})\vert_{ \#\{y=k\} \ge \gamma m} + \delta_1\\
    & \leq 2\exp({-2\epsilon^2 \gamma m}) + \delta_1 = \delta_2,
\end{align*}
where the first inequality use the fact that $x_i \in [0,1]$. For $\vert \hat{\sigma}^2_{i} - {\sigma^2_i}\vert$, because $(x_i\vert_{y=k} - \mu_{ki})^2 \in [0,1]$, by Lemma~\ref{lemma: hoeffding}, we can write:
\begin{equation*}
\mathbb{P}[\vert\hat{\sigma}^2_{i} - {\sigma^2_i}\vert > \epsilon] \leq 2\exp({-2m\epsilon^2}) = \delta_3.
\end{equation*}
Finally, we use the union bound to get the result for the continuous case:
\begin{align*}
     &\mathbb{P}((\vert \hat{p}(y=k) - p(y=k)\vert > \epsilon) \cup (\cup_{i=1}^n (\vert\hat{\sigma}^2_{i} - {\sigma^2_i}\vert > \epsilon) \cup (\cup_{k=1}^2  \vert \hat{\mu}_{ki} - {\mu}_{ki}\vert > \epsilon))) \\
     &\leq \delta_1 + n (2\delta_2 + \delta_3) \\
     & = (3n+1)2\exp({-2m\epsilon^2}) + 2n \times 2 \exp({-2\epsilon^2 \gamma m}) \\
     & \leq 2(5n+1)\exp({-2\epsilon^2 \gamma m}).
\end{align*}
Thus, for Lemma~\ref{Prop: parameters}.2 to hold with probability at least $1-\delta$, it suffices to pick m samples which satisfies
\begin{equation*}
m = \frac{1}{2\gamma \epsilon^2} \log(\frac{2(5n+1)}{\delta}) \leq \frac{1}{\rho_0 \epsilon^2} \log(\frac{2(5n+1)}{\delta}) = O(\frac{1}{\epsilon^2} \log(\frac{n}{\delta})).
\end{equation*}
The proposition’s proof is complete.
\end{proof}

\begin{lemma}
\label{lemma: discrete delta a}
In case of discrete inputs, and suppose that Assumption~\ref{Assumption: parmeters bounded} holds, then with probability at least $1 - \delta$, the following holds
\begin{align*}
    \vert \Delta a_{Gen}(\vx, 1, 0) - \Delta a_{Gen, \infty}(\vx, 1, 0)\vert \leq \frac{4(n+1)}{\rho_0} \sqrt{\frac{1}{\rho_0m} \log(\frac{2(4n+1)}{\delta})} = O\bigl(n\sqrt{\frac{1}{m}\log(\frac{n}{\delta})}\bigr).
\end{align*}
\end{lemma}
\begin{proof}
By the derivation of Lemma~\ref{Prop: parameters}, let $\epsilon < \rho_0 / 2$, then with probability at least $1 - \delta = 1 - 2(4n+1)\exp({-2 \gamma m} \epsilon^2)$, where $\gamma = \rho_0 - \epsilon$ the following holds:
\begin{align*}
    &\vert \Delta a_{Gen}(\vx, 1, 0) - \Delta a_{Gen, \infty}(\vx, 1, 0)\vert \\
    &= \vert \sum_{i=1}^n \log \frac{\hat{p}(x_i\vert y=1){p}(x_i\vert y=0) }{\hat{p}(x_i\vert y=0) {p}(x_i\vert y=1)} + \log\frac{\hat{p}(y=1){p}(y=0)  }{\hat{p}(y=0){p}(y=1)  } \vert \\
    &= \vert \sum_{i=1}^n (\log \hat{p}(x_i\vert y=1) - \log {p}(x_i\vert y=1)) + \sum_{i=1}^n (\log {p}(x_i\vert y=0) - \log \hat{p}(x_i\vert y=0)) \\
    & \quad + \log \hat{p}(y=1) - \log {p}(y=1) + \log {p}(y=0) - \log \hat{p}(y=0) \vert \\
    & \leq  \sum_{i=1}^n \vert\log \hat{p}(x_i\vert y=1) - \log {p}(x_i\vert y=1)\vert + \sum_{i=1}^n \vert\log {p}(x_i\vert y=0) - \log \hat{p}(x_i\vert y=0)\vert \\
    & \quad + \vert\log \hat{p}(y=1) - \log {p}(y=1)\vert + \vert\log {p}(y=0) - \log \hat{p}(y=0) \vert \\
    & \leq \frac{1}{\gamma}\bigl( \sum_{i=1}^n \epsilon + \sum_{i=1}^n \epsilon +\epsilon + \epsilon \bigr) \leq \frac{4(n+1)}{\rho_0} \epsilon. \\
\end{align*}
The penultimate inequality makes use of Lemma~\ref{Prop: parameters} and the concavity of $\log()$ together. Replace $\epsilon$ with the expressions with respect to $\delta$, we can write:
\begin{align*}
    \vert \Delta a_{Gen}(\vx, 1, 0) - \Delta a_{Gen, \infty}(\vx, 1, 0)\vert &\leq \frac{4(n+1)}{\rho_0} \sqrt{\frac{1}{2\gamma m} \log(\frac{2(4n+1)}{\delta})} \\
    & \leq \frac{4(n+1)}{\rho_0} \sqrt{\frac{1}{\rho_0 m} \log(\frac{2(4n+1)}{\delta})} = O\bigl(n\sqrt{\frac{1}{m}\log(\frac{n}{\delta})}\bigr).
\end{align*}
\end{proof}

\begin{lemma}
\label{lemma: sigma mu}
Let $\epsilon < \rho_0 / 2$, assume that Assumption~\ref{Assumption: parmeters bounded} holds, $\vert \hat{\sigma}^2_i - {\sigma}^2_i\vert \leq \epsilon$ and $\vert \hat{\mu}_{ki} - {\mu}_{ki}\vert \leq \epsilon$ for all $i, k$. Then we have:
\begin{align*}
    \vert{\sigma}_i \hat{\mu}_{ki} - \hat{\sigma}_i {\mu}_{ki} \vert \leq (1 + \frac{2}{3\rho_0})\epsilon.
\end{align*}
\end{lemma}
\begin{proof}
On the one hand, we can write:
\begin{align*}
    {\sigma}_i \hat{\mu}_{ki} - \hat{\sigma}_i {\mu}_{ki} \leq {\sigma}_i ({\mu}_{ki} + \epsilon) - \hat{\sigma}_i {\mu}_{ki} = ({\sigma}_i - \hat{\sigma}_i){\mu}_{ki} + \epsilon {\sigma}_i.
\end{align*}
On the other hand, we have:
\begin{align*}
    {\sigma}_i \hat{\mu}_{ki} - \hat{\sigma}_i {\mu}_{ki} \ge {\sigma}_i ({\mu}_{ki} - \epsilon) - \hat{\sigma}_i {\mu}_{ki} = ({\sigma}_i - \hat{\sigma}_i){\mu}_{ki} - \epsilon {\sigma}_i.
\end{align*}
We conclude that:
\begin{align*}
    \vert{\sigma}_i \hat{\mu}_{ki} - \hat{\sigma}_i {\mu}_{ki} \vert &\leq  \vert({\sigma}_i - \hat{\sigma}_i){\mu}_{ki}\vert + \vert\epsilon {\sigma}_i\vert \leq \vert{\sigma}_i - \hat{\sigma}_i\vert + \epsilon\\
    & = \vert \frac{{\sigma}^2_i - \hat{\sigma}^2_i}{{\sigma}_i + \hat{\sigma}_i} \vert + \epsilon \leq \frac{\epsilon}{\rho_0 + \rho_0 - \epsilon}  + \epsilon \leq (1 + \frac{2}{3\rho_0})\epsilon.
\end{align*}
\end{proof}

\begin{lemma}
\label{lemma: conti delta a}
In case of continuous inputs, and suppose that Assumption~\ref{Assumption: parmeters bounded} holds, then with probability at least $1 - \delta$, the following holds:
\begin{align*}
   \vert \Delta a_{Gen}(\vx, 1, 0) - \Delta a_{Gen, \infty}(\vx, 1, 0)\vert \leq 4(\frac{n}{3\rho_0}(\frac{4}{\rho^2_0} + \frac{3}{\rho_0} + \sqrt{\frac{2}{\rho_0}}) + \frac{1}{\rho_0}) \sqrt{\frac{1}{\rho_0 m} \log(\frac{2(5n+1)}{\delta})} = O\bigl(n\sqrt{\frac{1}{m}\log(\frac{n}{\delta})}\bigr).
\end{align*}
\end{lemma}
\begin{proof}
The following holds:
\begin{align*}
    &\vert \Delta a_{Gen}(\vx, 1, 0) - \Delta a_{Gen, \infty}(\vx, 1, 0)\vert \\
    &= \vert \sum_{i=1}^n \log \frac{\hat{p}(x_i\vert y=1){p}(x_i\vert y=0) }{\hat{p}(x_i\vert y=0) {p}(x_i\vert y=1)} + \log\frac{\hat{p}(y=1){p}(y=0)  }{\hat{p}(y=0){p}(y=1)  } \vert \\
    &= \vert \sum_{i=1}^n (\log \hat{p}(x_i\vert y=1) - \log {p}(x_i\vert y=1)) + \sum_{i=1}^n (\log {p}(x_i\vert y=0) - \log \hat{p}(x_i\vert y=0)) \\
    & \quad + \log \hat{p}(y=1) - \log {p}(y=1) + \log {p}(y=0) - \log \hat{p}(y=0) \vert \\
    & \leq  \sum_{i=1}^n \sum_{k=0}^1 \vert\log \hat{p}(x_i\vert y=k) - \log {p}(x_i\vert y=k)\vert + \sum_{k=0}^1 \vert\log \hat{p}(y=k) - \log {p}(y=k)\vert.
\end{align*}
To bound $\vert\log \hat{p}(x_i\vert y=k) - \log {p}(x_i\vert y=k)\vert$, let $\epsilon < \rho_0/2$, then by Lemma~\ref{Prop: parameters}, with probability at least $1 - \delta = 1 - 2(5n+1)\exp({-2\epsilon^2 \gamma m})$, where $\gamma = \rho_0 - \epsilon$, we can write:
\begin{align*}
    &\vert\log \hat{p}(x_i\vert y=k) - \log {p}(x_i\vert y=k)\vert  \\
    & = \vert \log(\sigma_i) - \log(\hat{\sigma}_i) + \frac{1}{2\hat{\sigma}^2_i{\sigma}^2_i}(\hat{\sigma}^2_i (x_i - \mu_{ki})^2 -  {\sigma}^2_i (x_i - \hat{\mu}_{ki})^2) \vert  \\
    & \leq \vert \log(\sigma_i) - \log(\hat{\sigma}_i) \vert + \frac{1}{2\hat{\sigma}^2_i{\sigma}^2_i} \vert \hat{\sigma}^2_i (x_i - \mu_{ki})^2 -  {\sigma}^2_i (x_i - \hat{\mu}_{ki})^2 \vert \\
    & \leq \frac{1}{\min(\sigma_i, \hat{\sigma}_i)} \vert \sigma_i - \hat{\sigma}_i \vert + \frac{1}{2\rho_0\hat{\sigma}^2_i} \vert \hat{\sigma}_i (x_i - \mu_{ki}) +  {\sigma}_i (x_i - \hat{\mu}_{ki}) \vert \vert \hat{\sigma}_i (x_i - \mu_{ki}) -  {\sigma}_i (x_i - \hat{\mu}_{ki}) \vert \\
    & \leq \frac{1}{\min(\sigma_i, \hat{\sigma}_i)} \vert \sigma_i - \hat{\sigma}_i \vert + \frac{1}{\rho_0\hat{\sigma}^2_i} \vert \hat{\sigma}_i (x_i - \mu_{ki}) -  {\sigma}_i (x_i - \hat{\mu}_{ki}) \vert \\
    & \leq \frac{1}{\min(\sigma_i, \hat{\sigma}_i)} \vert \sigma_i - \hat{\sigma}_i \vert + \frac{1}{\rho_0\hat{\sigma}^2_i}( \vert \hat{\sigma}_i - {\sigma}_i \vert + \vert{\sigma}_i \hat{\mu}_{ki} - \hat{\sigma}_i {\mu}_{ki} \vert) \\
    & \leq \frac{1}{\sqrt{\gamma}}\frac{2\epsilon}{3\rho_0} + \frac{1}{\rho_0 \gamma}(\frac{2\epsilon}{3\rho_0} + (1 + \frac{2}{3\rho_0})\epsilon) \\
    & \leq \sqrt{\frac{2}{\rho_0}} \frac{2\epsilon}{3\rho_0} + \frac{2}{\rho^2_0}(\frac{2\epsilon}{3\rho_0} + (1 + \frac{2}{3\rho_0})\epsilon) = \frac{2}{3\rho_0}(\frac{4}{\rho^2_0} + \frac{3}{\rho_0} + \sqrt{\frac{2}{\rho_0}})\epsilon.
\end{align*}
The last two inequalities make use of Lemma~\ref{Prop: parameters} the concavity of $\log()$ together. At the same time, we have:
\begin{align*}
    \vert\log \hat{p}(y=k) - \log {p}(y=k)\vert \leq \frac{1}{\gamma} \vert \hat{p}(y=k) - {p}(y=k)\vert \leq \frac{2}{\rho_0} \vert \hat{p}(y=k) - {p}(y=k)\vert \leq \frac{2}{\rho_0} \epsilon.
\end{align*}
At last, combining the above findings and replace $\epsilon$ with the expressions with respect to $\delta$, we can get:
\begin{align*}
    \vert \Delta a_{Gen}(\vx, 1, 0) - \Delta a_{Gen, \infty}(\vx, 1, 0)\vert &\leq 2n\frac{2}{3\rho_0}(\frac{4}{\rho^2_0} + \frac{3}{\rho_0} + \sqrt{\frac{2}{\rho_0}})\epsilon + 2 \frac{2}{\rho_0} \epsilon \\
    &= 4(\frac{n}{3\rho_0}(\frac{4}{\rho^2_0} + \frac{3}{\rho_0} + \sqrt{\frac{2}{\rho_0}}) + \frac{1}{\rho_0}) \sqrt{\frac{1}{\rho_0 m} \log(\frac{2(5n+1)}{\delta})} \\
    &= O\bigl(n\sqrt{\frac{1}{m}\log(\frac{n}{\delta})}\bigr).
\end{align*}
\end{proof}
Now, we are ready to prove Theorem~\ref{Thm: generalization bound}. 
\begin{proof}
Let $\epsilon = O\bigl(n\sqrt{\frac{1}{m}\log(\frac{n}{\delta})}\bigr)$ which are claimed in the Lemma~\ref{lemma: discrete delta a} for discrete case and  Lemma~\ref{lemma: conti delta a} for continuous case. Then we simplify the $\vert R_{\ell_{0-1}}(h_{Gen, m}) - R_{\ell_{0-1}}(h_{Gen, \infty})\vert$ as follows:
\begin{align*}
    &\vert R_{\ell_{0-1}}(h_{Gen, m}) - R_{\ell_{0-1}}(h_{Gen, \infty})\vert \\
    &= \vert \mathbb{E}_{(\vx,y) \sim \mathcal{D}} [\ell_{0-1}(h_{Gen, m}, (\vx,y)) - \ell_{0-1}(h_{Gen, \infty}, (\vx,y))]\vert \\
    & \leq \mathbb{E}_{(\vx,y) \sim \mathcal{D}} \vert \ell_{0-1}(h_{Gen, m},(x, y)) - \ell_{0-1}(h_{Gen, \infty}, (\vx,y))\vert \\
    &= \mathbb{P}_{(\vx,y) \sim \mathcal{D}} (h_{Gen, m}(x) \neq h_{Gen, \infty}(x)) \\
    &= \bigl (\mathbb{P} (h_{Gen, m}(x) \neq h_{Gen, \infty}(x) \vert \vert \Delta a_{Gen}(\vx, 1, 0) - a_{Gen, \infty}(x)\vert \leq \epsilon) \mathbb{P}(\vert \Delta a_{Gen}(\vx, 1, 0) - a_{Gen, \infty}(x)\vert \leq \epsilon) \\
    & \quad + \mathbb{P} (h_{Gen, m}(x) \neq h_{Gen, \infty}(x) \vert \vert \Delta a_{Gen}(\vx, 1, 0) - a_{Gen, \infty}(x)\vert > \epsilon) \mathbb{P}(\vert \Delta a_{Gen}(\vx, 1, 0) - a_{Gen, \infty}(x)\vert > \epsilon) \bigr)\\
    &\leq  \mathbb{P} (h_{Gen, m}(x) \neq h_{Gen, \infty}(x) \vert \vert \Delta a_{Gen}(\vx, 1, 0) - a_{Gen, \infty}(x)\vert \leq \epsilon) + \delta\\
    &\leq \mathbb{P} (\vert a_{Gen, \infty}(x)\vert \leq O\bigl(n\sqrt{\frac{1}{m}\log(\frac{n}{\delta})}\bigr)) + \delta\\
    &=  G\bigl(O(\sqrt{\frac{1}{m} \log(\frac{n}{\delta})})\bigr) + \delta.
\end{align*}
\end{proof}

\subsection{Proof of Proposition~\ref{Prop: binary ploy -n}}
\label{proof: Prop: binary ploy -n}

\begin{lemma}
\label{Prop: binary high E}
Suppose that Assumption~\ref{Assumption: binary KL} holds, then $\mathbb{E}[\Delta a_{Gen}(\vx, 1, 0)\vert y=1] = \Omega(n)$, and $\mathbb{E}[-\Delta a_{Gen}(\vx, 1, 0)\vert y=0] = \Omega(n)$.
\end{lemma}
\begin{proof}
We calculate $\mathbb{E}[\Delta a_{Gen}(\vx, 1, 0)\vert y=1]$ straightly:
\begin{align*}
    \mathbb{E}_{\vx}[\Delta a_{Gen}(\vx, 1, 0)\vert y=1] &=  \mathbb{E}_{\vx}[\sum_{i=1}^n \log \frac{{p}(x_i\vert y=1) }{{p}(x_i\vert y=0) } + \log\frac{{p}(y=1) }{{p}(y=0) } \vert y = 1] \\
    & = \sum_{i=1}^n \mathbb{E}_{x_i}[\log \frac{{p}(x_i\vert y=1) }{{p}(x_i\vert y=0) }\vert y = 1] +  \log\frac{{p}(y=1) }{{p}(y=0)}.
\end{align*}
We note that $\mathbb{E}_{x_i}[\log \frac{{p}(x_i\vert y=1) }{{p}(x_i\vert y=0) }\vert y = 1]$ is the KL Divergence $D({p}(x_i\vert y=1) \Vert {p}(x_i\vert y=0))$. It is nonnegative and equals 0 if and only if ${p}(x_i\vert y=1) = {p}(x_i\vert y=0)$ for all $x_i \in \mathcal{X}_i$ ($\{0,1\}$ in case of discrete inputs and $[0,1]$ in case of continuous inputs). By assumption~\ref{Assumption: binary KL}, we obtain that 
\begin{align*}
    \mathbb{E}_{\vx}[\Delta a_{Gen}(\vx, 1, 0)\vert y=1] & = \sum_{i=1}^n D({p}(x_i\vert y=1) \Vert {p}(x_i\vert y=0)) +  \log\frac{{p}(y=1) }{{p}(y=0)} \\
    & = \beta_{1,0} n +  \log\frac{{p}(y=1) }{{p}(y=0)}\\
    & \ge\beta_{1,0} n + \log(\frac{\rho_0 }{1-\rho_0}),
\end{align*}
which implies that $\mathbb{E}[\Delta a_{Gen}(\vx, 1, 0)\vert y=1] = \Omega(n)$. In the same way, we can know that $\mathbb{E}[-\Delta a_{Gen}(\vx, 1, 0)\vert y=0] = \Omega(n)$ as well. Then the proposition has been proved.
\end{proof}

Based on Lemma~\ref{Prop: binary high E}, we can prove Proposition~\ref{Prop: binary ploy -n}.
\begin{proof}
For convenience, we denote $\mathbb{E}[\Delta a_{Gen}(\vx, 1, 0)\vert y=k]$ by $\zeta_{k}$. To bound $G(\tau)\vert_{y=1} = \mathbb{P}(\vert \Delta a_{Gen, \infty}(\vx, 1, 0)\vert \leq \tau n\vert y=1)$, the following holds:
\begin{align*}
    &\mathbb{P}(\vert \Delta a_{Gen, \infty}(\vx, 1, 0) \vert \leq \tau n \vert y=1) \\
    &\leq \mathbb{P}( \Delta a_{Gen, \infty}(\vx, 1, 0) \leq \tau n \vert y=1) \\
    &=\mathbb{P}(  \Delta a_{Gen, \infty}(\vx, 1, 0) - \zeta_{1} n \leq \tau n - \zeta_{1} n\vert y=1) \\
    &=\mathbb{P}(\sum_{i=1}^n \log \frac{{p}(x_i\vert y=1) }{{p}(x_i\vert y=0) } - \mathbb{E}_{\vx}(\sum_{i=1}^n \log \frac{{p}(x_i\vert y=1) }{{p}(x_i\vert y=0) }) \leq (\tau  - \zeta_{1}) n \vert y=1)\\
    &=\mathbb{P}(\vert \sum_{i=1}^n \log \frac{{p}(x_i\vert y=1) }{{p}(x_i\vert y=0) } - \mathbb{E}_{\vx}(\sum_{i=1}^n \log \frac{{p}(x_i\vert y=1) }{{p}(x_i\vert y=0) })\vert \ge (\zeta_{1} - \tau) n \vert y=1)\\
    &\leq \frac{\mathbb{V} [\sum_{i=1}^n \log \frac{{p}(x_i\vert y=1) }{{p}(x_i\vert y=0) } \vert y = 1]}{(\tau -  \zeta_{1})^2 n^2} & \text{(Chebyshev inequality)}\\
    & = \frac{\alpha_{1} n }{(\tau -  \zeta_{1})^2 n^2} & \text{(Assumption~\ref{Assumption: binary likelihood ratio var})}\\ 
    &= \frac{\alpha_{1} }{(\tau -  \zeta_{1})^2 n}.
\end{align*}
Similar to the above discussion, we have: $G(\tau)\vert_{y=0}\leq \frac{\alpha_{0} }{(\tau -  \vert\zeta_{0}\vert)^2 n}$. Finally, we can conclude that:

\begin{align*}
    \widetilde{G}(\tau) &= p(y=1)G(\tau)\vert_{y=1} + p(y=0)G(\tau)\vert_{y=0}\\
    &\leq p(y=1)\frac{\alpha_{1} }{(\tau -  \zeta_{1})^2 n} + p(y=0)\frac{\alpha_{0} }{(\tau -  \vert\zeta_{0}\vert)^2 n}\\
    &\leq \frac{\alpha }{(\tau - \zeta)^2 n}.
\end{align*}
\end{proof}

\subsection{Proof of Proposition~\ref{Prop: binary exp -n}}
\label{proof: Prop: binary exp -n}
\begin{proof}
Based on the results from Lemma~\ref{Prop: binary high E}, we first consider the discrete condition and the event that a test sample $\vx$ with label 1. To bound $G(\tau)\vert_{y=1} = \mathbb{P}(\vert \Delta a_{Gen, \infty}(\vx, 1, 0)\vert \leq \tau n\vert y=1)$, the following holds:
\begin{align*}
    G(\tau)\vert_{y=1} &= \mathbb{P}(\vert\Delta a_{Gen, \infty}(\vx, 1, 0)\vert \leq \tau n\vert y=1) \\
    &\leq \mathbb{P}(\Delta a_{Gen, \infty}(\vx, 1, 0) \leq \tau n\vert y=1) \\
    &=\mathbb{P}(\Delta a_{Gen, \infty}(\vx, 1, 0) - \zeta_1 n \leq \tau n - \zeta_1 n\vert y=1) \\
    &=\mathbb{P}(\sum_{i=1}^n \log \frac{{p}(x_i\vert y=1) }{{p}(x_i\vert y=0) } - \mathbb{E}_{\vx}(\sum_{i=1}^n \log \frac{{p}(x_i\vert y=1) }{{p}(x_i\vert y=0) }) \leq (\tau  - \zeta_1) n \vert y=1)\\
    &\leq \exp({- \frac{2(\tau  - \zeta_1)^2n^2}{n(\log \frac{1-\rho_0}{\rho_0} - \log \frac{\rho_0}{1-\rho_0})^2}}) = \exp\bigl({- \frac{(\tau  - \zeta_1)^2n}{2(\log \frac{1-\rho_0}{\rho_0})^2}}\bigr). & \text{(by Lemma~\ref{lemma: hoeffding})}
\end{align*}
Similar to the above discussion, we have: $G(\tau)\vert_{y=0}\leq \exp\bigl({- \frac{(\tau  - \zeta_2)^2n}{2(\log \frac{1-\rho_0}{\rho_0})^2}}\bigr)$. Finally, we can conclude that:
\begin{align*}
    G(\tau) &= p(y=1)G(\tau)\vert_{y=1} + p(y=0)G(\tau)\vert_{y=0}\\
    &\leq p(y=1)\exp\bigl({- \frac{(\tau  - \zeta_1)^2n}{2(\log \frac{1-\rho_0}{\rho_0})^2}}\bigr) + p(y=0)\exp\bigl({- \frac{(\tau  - \zeta_2)^2n}{2(\log \frac{1-\rho_0}{\rho_0})^2}}\bigr)\\
    &\leq \exp\bigl({- \frac{(\tau  - \zeta)^2n}{2(\log \frac{1-\rho_0}{\rho_0})^2}}\bigr) = \exp{(-O((\tau - \zeta)^2n))}.
\end{align*}
Second, we consider the continuous case, the only difference from the discrete case is that the range of $\log \frac{{p}(x_i\vert y=1) }{{p}(x_i\vert y=0) }$. For all $i$, it satisfies:
\begin{align*}
    \vert\log \frac{{p}(x_i\vert y=1) }{{p}(x_i\vert y=0) }\vert &= \vert\log \frac{\frac{1}{\sqrt{2\pi}\sigma_i} \exp({-\frac{(x_i - \mu_{1i})^2}{2\sigma_i^2}})}{\frac{1}{\sqrt{2\pi}\sigma_i} \exp({-\frac{(x_i - \mu_{0i})^2}{2\sigma_i^2}})}\vert \\
    &= \vert\frac{\mu_{1i} - \mu_{0i}}{\sigma_i^2}x_i  + \frac{ \mu_{0i}^2 -  \mu_{1i}^2}{2\sigma_i^2} \vert \\
    &\leq \vert\frac{\mu_{1i} - \mu_{0i}}{\sigma_i^2}x_i \vert  +  \vert \frac{ (\mu_{0i} -  \mu_{1i}) (\mu_{0i} +  \mu_{1i})}{2\sigma_i^2} \vert\\
    &\leq \frac{1}{\rho_0} + \frac{2}{2\rho_0} = \frac{2}{\rho_0}.
\end{align*}
So we can get:
\begin{align*}
    G(\tau) \leq \exp\bigl({- \frac{2(\tau  - \zeta)^2n}{(\frac{4}{\rho_0})^2}}\bigr) = \exp{(-O((\tau - \zeta)^2n))}.
\end{align*}
\end{proof}

\subsection{Proof of Theorem\ref{cor: binary gen logn}}
\label{Proof of Corollary cor: binary gen logn}

\begin{proof}
In the case that precondition of Proposition~\ref{Prop: binary ploy -n} holds, combining Theorem~\ref{Thm: generalization bound} and Proposition~\ref{Prop: binary ploy -n}, we know that there exist positive $c = \Theta(1)$ such that when $c \sqrt{\frac{1}{m} \log(\frac{n}{\delta})} < \zeta$, with probability at least $1 - \delta$, we have
\begin{align*}
    R_{\ell_{0-1}}(h_{Gen, m}) &\leq R_{\ell_{0-1}}(h_{Gen, \infty}) + \frac{\alpha}{(c\sqrt{\frac{1}{m} \log(\frac{n}{\delta})} - \zeta)^2n}  + \delta.
\end{align*}
For fixed $\epsilon_0 \in (0,1)$, the logical relations listed in the following is correct:
\begin{align*}
    &R_{\ell_{0-1}}(h_{Gen, m}) \leq R_{\ell_{0-1}}(h_{Gen, \infty}) + \epsilon_0 \text{ with probability at least 1 - $\delta$} \\
    & \Leftarrow c \sqrt{\frac{1}{m} \log(\frac{n}{\delta})} < \zeta \wedge 0 < \delta < 1 \wedge  \frac{\alpha}{(c\sqrt{\frac{1}{m} \log(\frac{n}{\delta})} - \zeta)^2n}  + \delta \leq \epsilon_0\\
    & \Leftrightarrow c \sqrt{\frac{1}{m} \log(\frac{n}{\delta})} < \zeta \wedge  0 < \delta < 1  \wedge \frac{\alpha}{(c\sqrt{\frac{1}{m} \log(\frac{n}{\delta})} - \zeta)^2n} \leq \epsilon_0 - \delta\\
    & \Leftrightarrow c \sqrt{\frac{1}{m} \log(\frac{n}{\delta})} < \zeta \wedge 0 < \delta < 1 \wedge \epsilon_0 - \delta > 0 \wedge (c\sqrt{\frac{1}{m} \log(\frac{n}{\delta})} - \zeta)^2 \ge \frac{\alpha}{(\epsilon_0 - \delta)n}\\
    & \Leftrightarrow c \sqrt{\frac{1}{m} \log(\frac{n}{\delta})} < \zeta \wedge 0 < \delta < \epsilon_0 \wedge (c \sqrt{\frac{1}{m} \log(\frac{n}{\delta})} - \zeta)^2 \ge \frac{\alpha}{(\epsilon_0 - \delta)n}\\
    & \Leftrightarrow 0 < \delta < \epsilon_0 \wedge \zeta - c \sqrt{\frac{1}{m} \log(\frac{n}{\delta})} \ge \sqrt{\frac{\alpha}{(\epsilon_0 - \delta)n}}\\
    & \Leftrightarrow 0 < \delta < \epsilon_0 \wedge \zeta - \sqrt{\frac{\alpha}{(\epsilon_0- \delta)n}} > 0 \wedge (\zeta - \sqrt{\frac{\alpha}{(\epsilon_0 - \delta)n}})^2 \ge c^2 \frac{1}{m} \log(\frac{n}{\delta})\\
    & \Leftarrow 0 < \delta < \epsilon_0 - \frac{\alpha}{\zeta^2 n} \wedge \epsilon_0 - \frac{\alpha}{\zeta^2 n} > 0 \wedge m  \ge \frac{c^2}{(\zeta - \sqrt{\frac{\alpha}{(\epsilon_0 - \delta)n}})^2} \log(\frac{n}{\delta})\\
    & \Leftarrow 0 < \delta \leq \frac{\epsilon_0}{2} \wedge  \epsilon_0 - \frac{\alpha}{\zeta^2 n} > \frac{\epsilon_0}{2} \wedge m  \ge \frac{c^2}{(\zeta - \sqrt{\frac{\alpha}{(\epsilon_0 - \delta)n}})^2} \log(\frac{n}{\delta})\\
    & \Leftrightarrow 0 < \delta \leq \frac{\epsilon_0}{2} \wedge n < \frac{2\alpha}{\epsilon_0 \zeta^2} \wedge m  \ge \frac{c^2}{(\zeta - \sqrt{\frac{\alpha}{(\epsilon_0 - \delta)n}})^2} \log(\frac{n}{\delta})\\
    & \Leftarrow 0 < \delta \leq \frac{\epsilon_0}{2} \wedge n < 2\frac{2\alpha}{\epsilon_0 \zeta^2} \wedge m \ge \frac{c^2}{(\zeta - \sqrt{\frac{2\alpha}{\epsilon_0 n}} )^2}\log(\frac{n}{\delta})\\
    & \Leftarrow 0 < \delta \leq \frac{\epsilon_0}{2} \wedge n < \frac{4\alpha}{\epsilon_0 \zeta^2} \wedge  m \ge \frac{c^2}{(\zeta - \sqrt{\frac{1}{2}} \zeta)^2}\log(\frac{n}{\delta})\\
    & \Leftarrow 0 < \delta \leq \frac{\epsilon_0}{2} \wedge n < \frac{4\alpha}{\epsilon_0 \zeta^2} \wedge m = O(\log(n)).
\end{align*}

In the case that precondition of Proposition~\ref{Prop: binary exp -n} holds, then by Theorem~\ref{Thm: generalization bound} and Proposition~\ref{Prop: binary exp -n}, we know that there exist positive constant $b, c = \Theta(1)$ such that when $c \sqrt{\frac{1}{m} \log(\frac{n}{\delta})} < \zeta$, with probability at least $1 - \delta$, we have
\begin{align*}
    R_{\ell_{0-1}}(h_{Gen, m}) &\leq R_{\ell_{0-1}}(h_{Gen, \infty}) + \exp(-b(c \sqrt{\frac{1}{m} \log(\frac{n}{\delta})} - \zeta)^2n) + \delta.
\end{align*}
For fixed $\epsilon_0 \in (0,1)$, the logical relations listed in the following is correct:
\begin{align*}
    &R_{\ell_{0-1}}(h_{Gen, m}) \leq R_{\ell_{0-1}}(h_{Gen, \infty}) + \epsilon_0 \text{ with probability at least 1 - $\delta$} \\
    & \Leftarrow c \sqrt{\frac{1}{m} \log(\frac{n}{\delta})} < \zeta \wedge 0 < \delta < 1 \wedge \exp(-b(c \sqrt{\frac{1}{m} \log(\frac{n}{\delta})} - \zeta)^2n) + \delta \leq \epsilon_0\\
    & \Leftrightarrow c \sqrt{\frac{1}{m} \log(\frac{n}{\delta})} < \zeta \wedge  0 < \delta < 1  \wedge \exp(-b(c \sqrt{\frac{1}{m} \log(\frac{n}{\delta})} - \zeta)^2n) \leq \epsilon_0 - \delta\\
    & \Leftrightarrow c \sqrt{\frac{1}{m} \log(\frac{n}{\delta})} < \zeta \wedge 0 < \delta < 1 \wedge \epsilon_0 - \delta > 0 \wedge -b(c \sqrt{\frac{1}{m} \log(\frac{n}{\delta})} - \zeta)^2n \leq \log(\epsilon_0 - \delta)\\
    & \Leftrightarrow c \sqrt{\frac{1}{m} \log(\frac{n}{\delta})} < \zeta \wedge 0 < \delta < \epsilon_0 \wedge (c \sqrt{\frac{1}{m} \log(\frac{n}{\delta})} - \zeta)^2 \ge \frac{1}{bn} \log(\frac{1}{\epsilon_0 - \delta})\\
    & \Leftarrow c \sqrt{\frac{1}{m} \log(\frac{n}{\delta})} < \zeta \wedge 0 < \delta < \epsilon_0 \wedge \zeta - c \sqrt{\frac{1}{m} \log(\frac{n}{\delta})} \ge \sqrt{\frac{1}{bn} \log(\frac{1}{\epsilon_0 - \delta})}\\
    & \Leftrightarrow 0 < \delta < \epsilon_0 \wedge \zeta - \sqrt{\frac{1}{bn} \log(\frac{1}{\epsilon_0 - \delta})} \ge c \sqrt{\frac{1}{m} \log(\frac{n}{\delta})}\\
    & \Leftarrow 0 < \delta < \epsilon_0 \wedge \zeta - \sqrt{\frac{1}{bn} \log(\frac{1}{\epsilon_0 - \delta})} > 0 \wedge (\zeta - \sqrt{\frac{1}{bn} \log(\frac{1}{\epsilon_0 - \delta})})^2 \ge c^2 \frac{1}{m} \log(\frac{n}{\delta})\\
    & \Leftarrow 0 < \delta < \epsilon_0 - \exp(-b\zeta^2n) \wedge \epsilon_0 - \exp(-b\zeta^2n) > 0 \wedge m  \ge \frac{c^2}{(\zeta - \sqrt{\frac{1}{bn} \log(\frac{1}{\epsilon_0 - \delta})})^2} \log(\frac{n}{\delta})\\
    & \Leftarrow 0 < \delta \leq \frac{\epsilon_0}{2} \wedge \epsilon_0 - \exp(-b\zeta^2n) > \frac{\epsilon_0}{2} \wedge m  \ge \frac{c^2}{(\zeta - \sqrt{\frac{1}{bn} \log(\frac{1}{\epsilon_0 - \delta})})^2} \log(\frac{n}{\delta})\\
    & \Leftrightarrow 0 < \delta \leq \frac{\epsilon_0}{2} \wedge \epsilon_0 \exp(b\zeta^2n) > 2 \wedge m  \ge \frac{c^2}{(\zeta - \sqrt{\frac{1}{bn} \log(\frac{1}{\epsilon_0 - \delta})})^2} \log(\frac{n}{\delta})\\
    & \Leftarrow 0 < \delta \leq \frac{\epsilon_0}{2} \wedge \epsilon_0 \exp(b\zeta^2n) > 3 \wedge m  \ge \frac{c^2}{(\zeta - \sqrt{\frac{1}{bn} \log(\frac{2}{\epsilon_0})})^2} \log(\frac{2n}{\epsilon_0})\\
    & \Leftarrow 0 < \delta \leq \frac{\epsilon_0}{2} \wedge \epsilon_0 \exp(b\zeta^2n) > 3 \wedge m  \ge \frac{c^2}{\zeta^2 (1 - \frac{\log(2/\epsilon_0)}{\log(3/\epsilon_0)})^2} \log(\frac{2n}{\epsilon_0})\\
    & \Leftrightarrow 0 < \delta \leq \frac{\epsilon_0}{2} \wedge \epsilon_0 \exp(b\zeta^2n) > 3 \wedge m = O(\log(n)).
\end{align*}
\end{proof}

\section{Proofs of Appendix~\ref{sec: Deferred Results: Multiclass H-consistency}}
\label{sec: proof Deferred Results: Multiclass H-consistency}
\subsection{Proof of Proposition~\ref{Thm: Distribution-dependent concave bound}}
\label{proof: Theorem Thm: Distribution-dependent concave bound}
\begin{proof}
Because $\langle \Delta \mathscr{C}_{\ell_2, \mathcal{H}}(\boldsymbol{h}, \vx) \rangle_\epsilon < s(\Delta \mathscr{C}_{\ell_1, \mathcal{H}}(\boldsymbol{h}, \vx))$ for all $x \in \mathcal{X}$, we have:
\begin{align*}
    &R_{\ell_2}(\boldsymbol{h}) - R^*_{\ell_2, \mathcal{H}} + M_{\ell_2, \mathcal{H}}\\
    &= \mathbb{E}_{\vx} [\mathscr{C}_{\ell_2}(\boldsymbol{h}, \vx)] - R^*_{\ell_2, \mathcal{H}} + R_{\ell_2, \mathcal{H}}^* - \mathbb{E}_{\vx}[\mathscr{C}_{\ell_2, \mathcal{H}}^*(\vx)] & \text{(by definition)}\\
    &= \mathbb{E}_{\vx} [\mathscr{C}_{\ell_2}(\boldsymbol{h}, \vx) - \mathscr{C}_{\ell_2, \mathcal{H}}^*(\vx)]\\
    &= \mathbb{E}_{\vx}[\Delta \mathscr{C}_{\ell_2, \mathcal{H}}(\boldsymbol{h}, \vx)]\\
    &= \mathbb{E}_{\vx}[\Delta \mathscr{C}_{\ell_2, \mathcal{H}}(\boldsymbol{h}, \vx)\mathbbm{1}_{\mathscr{C}_{\ell_2, \mathcal{H}}(\boldsymbol{h}, \vx) > \epsilon} + \Delta \mathscr{C}_{\ell_2, \mathcal{H}}(\boldsymbol{h}, \vx)\mathbbm{1}_{\mathscr{C}_{\ell_2, \mathcal{H}}(\boldsymbol{h}, \vx) \leq \epsilon}]\\
    &\leq \mathbb{E}_{\vx}[s(\Delta \mathscr{C}_{\ell_1, \mathcal{H}}(\boldsymbol{h}, \vx))] + \epsilon\\
    &\leq s(\mathbb{E}_{\vx}[\Delta \mathscr{C}_{\ell_1, \mathcal{H}}(\boldsymbol{h}, \vx)]) + \epsilon & \text{(Jensen's inequality)}\\
    &= s(R_{\ell_1}(\boldsymbol{h}) - R^*_{\ell_1, \mathcal{H}} + M_{\ell_1, \mathcal{H}}) + \epsilon.
\end{align*}
\end{proof}

\subsection{Proof of Theorem~\ref{cor: Distribution-independent concave bound}}
\label{proof: Proof of Theorem cor: Distribution-independent concave bound}
\begin{lemma}[Distribution-dependent concave $\ell_{0-1}$ bound]
\label{cor: Distribution-dependent concave bound}
Suppose that $\mathcal{H}$ satisfies that $\{\mathop{\mathrm{argmax}}_{y \in \mathcal{Y}} h_y(\vx) : \boldsymbol{h} \in \mathcal{H}\} = \{1, \dots, K\}$ for any $\vx \in \mathcal{X}$, and there exists a non-decreasing concave function $s: \mathbb{R}_+ \to \mathbb{R}_+$ and $\epsilon \ge 0$ that the following holds for any $\hat{y} \in \mathcal{Y}$, $x \in \mathcal{X}$ and $\vh \in \mathcal{H}_{\hat{y}}(\vx)$:
\begin{equation*}
    \langle \max_y p_y(\vx) - p_{\hat{y}}(\vx) \rangle_\epsilon \leq s(\inf_{\boldsymbol{h} \in \mathcal{H}_{\hat{y}}(\vx)} \Delta \mathscr{C}_{\ell, \mathcal{H}}(\boldsymbol{h}, \vx)).
\end{equation*}
Then it holds for all $\boldsymbol{h} \in \mathcal{H}$ that
\begin{equation*}
    R_{\ell_{0-1}}(\boldsymbol{h}) - R^*_{\ell_{0-1}, \mathcal{H}} + M_{\ell_{0-1}, \mathcal{H}} \leq s(R_{\ell(\boldsymbol{h})} - R^*_{\ell, \mathcal{H}} + M_{\ell, \mathcal{H}}) + \epsilon.
\end{equation*}
\end{lemma}
\begin{proof}
For any $\vx_0 \in \mathcal{X}$ and $\boldsymbol{h}_0 \in \mathcal{H}$, let $\hat{y} $ be the index of the largest element of $\vh_0(\vx)$. Then by the precondition, we have
\begin{equation*}
    \langle \Delta \mathscr{C}_{\ell_{0-1}, \mathcal{H}}(\boldsymbol{h}_0, \vx_0) \rangle_\epsilon = \langle \max_y p_y(\vx_0) - p_{\hat{y}}(\vx_0) \rangle_\epsilon \leq  s(\inf_{\boldsymbol{h} \in \mathcal{H}_{\hat{y}}(\vx_0) } \Delta \mathscr{C}_{\ell, \mathcal{H}}(\boldsymbol{h}, \vx_0)) \leq s(\Delta \mathscr{C}_{\ell, \mathcal{H}}(\boldsymbol{h}_0, \vx_0)).
\end{equation*}
where we use the assumption that $s$ is non-decreasing.
Combining the condition in Proposition~\ref{Thm: Distribution-dependent concave bound} we can conclude the proof.
\end{proof}

Built upon Lemma~\ref{cor: Distribution-dependent concave bound}, we can prove Theorem~\ref{cor: Distribution-independent concave bound} as follows.
\begin{proof}
For any $\vx_0 \in \mathcal{X}$, $\vp(\vx_0) \in \Delta_K$, $\hat{y}_0 \in \mathcal{Y}$, and $\vh \in\mathcal{H}_{\hat{y}_0}(\vx_0)$, we can write:
\begin{align*}
    &\max_y p_y(\vx_0) - p_{\hat{y}_0}(\vx_0) \\
    &\leq s(\inf_{\hat{y} \in \mathcal{Y}, x \in \mathcal{X}, \boldsymbol{h} \in \mathcal{H}_{\hat{y}}(\vx), \vp \in \mathcal{P}_{\hat{y}}({\max_y p_y(\vx_0) - p_{\hat{y}_0}(\vx_0)})} \Delta \mathscr{C}_{\ell, \mathcal{H}}(\boldsymbol{h}, \vx, \vp)) & (\text{Assumption}) \\
    &\leq s(\inf_{ x \in \mathcal{X}, \boldsymbol{h} \in \mathcal{H}_{\hat{y}_0}(\vx), \vp \in \mathcal{P}_{\hat{y}_0}({\max_y p_y(\vx_0) - p_{\hat{y}_0}(\vx_0)})} \Delta \mathscr{C}_{\ell, \mathcal{H}}(\boldsymbol{h}, \vx, \vp)) \\
    & \leq s(\inf_{x \in \mathcal{X}, \boldsymbol{h} \in \mathcal{H}_{\hat{y}_0}(\vx)} \Delta \mathscr{C}_{\ell, \mathcal{H}}(\boldsymbol{h}, \vx, \vp(\vx_0)))\\
    &\leq s(\inf_{\boldsymbol{h} \in \mathcal{H}_{\hat{y}_0}(\vx_0)} \Delta \mathscr{C}_{\ell, \mathcal{H}}(\boldsymbol{h}, \vx_0, \vp(\vx_0)))\\
    &= s(\inf_{\boldsymbol{h} \in \mathcal{H}_{\hat{y}_0}(\vx_0)} \Delta \mathscr{C}_{\ell, \mathcal{H}}(\boldsymbol{h}, \vx_0)).
\end{align*}
Combining the result of Lemma~\ref{cor: Distribution-dependent concave bound} we can prove Theorem~\ref{cor: Distribution-independent concave bound}.
\end{proof}

\subsection{Proofs of Theorem~\ref{thm: H-consistency bound for log neural network}}
\label{Proofs of thm:H-consistency bound for log neural network}
\begin{proof}
The proof is essentially the same as that of Theorem~\ref{thm: H-consistency bound for log}. We use $\mathcal{H}, \ell$ to replace the $\mathcal{H}_{NN}, \ell_{log}$ in the following proof, which will not bring ambiguity. We can rewrite the $\mathcal{J}_{\ell}(t)$ as follows:
\begin{align*}
    \mathcal{J}_{\ell}(t) = \inf_{\hat{y} \in \mathcal{Y}} \inf_{\vp \in \mathcal{P}_{\hat{y}}(t)} \inf_{\vx \in \mathcal{X}}( \inf_{\boldsymbol{h} \in \mathcal{H}_{\hat{y}}(\vx)} \mathscr{C}_{\ell}(\boldsymbol{h}, \vx, \vp) - \inf_{\boldsymbol{h} \in \mathcal{H}} \mathscr{C}_{\ell}(\boldsymbol{h}, \vx, \vp)).
\end{align*}
For all $\boldsymbol{h} \in \mathcal{H}$ and $\vx \in \mathcal{X}$, we have
\begin{align*}
    \mathscr{C}_{\ell}(\boldsymbol{h}, \vx, \vp)) = \sum_{y=1}^K p_y \ell(y, \boldsymbol{h}(x)) =  \sum_{y=1}^K p_y (-h_y + \log(\sum_{j=1}^K \exp{(h_j)})).
\end{align*}

To get the $\inf_{\boldsymbol{h} \in \mathcal{H}} \mathscr{C}_{\ell}(\boldsymbol{h}, \vx, \vp))$, we consider the following problem
\begin{equation*}
\min_{\boldsymbol{h}} \,\, \sum_{y=1}^K p_y (-h_y + \log(\sum_{j=1}^K \exp{(h_j)})).
\end{equation*}
By Lemma~\ref{lemma: Convexity of C}, we know that this problem is convex, we can make use of KKT conditions~\cite{boyd2004convex} to find the points that are primal and dual optimal, which can be written as follows
\begin{align}
-p_i + \frac{\exp{(h_i^*)}}{\sum_{k=1}^K \exp{(h_k^*)}} = 0 \quad i = 1, \dots, K.
\end{align}
It implies that $h_i^* = \log(p_i \sum_{k=1}^K \exp{(h_k^*)}$. Thus, we have
\begin{align*}
\inf_{\boldsymbol{h} \in \mathcal{H}} \mathscr{C}_{\ell}(\boldsymbol{h}, \vx, \vp)) = \sum_{y=1}^K p_y (-h^*_y + \log(\sum_{j=1}^K \exp{(h^*_j)})) = - \sum_{y=1}^K p_y \log(p_y).
\end{align*}
which is the entropy of distribution $\vp$. By Lemma~\ref{lemma: inf H bar delta C}, we know that
\begin{align*}
\inf_{\boldsymbol{h} \in \mathcal{H}_{\hat{y}}(\vx)} \mathscr{C}_{\ell}(\boldsymbol{h}, \vx, \vp)) &\ge -(p_{max} + p_{\hat{y}}) \log(\frac{p_{max} + p_{\hat{y}}}{2}) - \sum_{y \notin \{y_{max}, \hat{y}\}}p_y \log(p_y).
\end{align*}
Then we have
\begin{align*}
     &\inf_{\boldsymbol{h} \in \mathcal{H}_{\hat{y}}(\vx)} \mathscr{C}_{\ell}(\boldsymbol{h}, \vx, \vp) - \inf_{\boldsymbol{h} \in \mathcal{H}} \mathscr{C}_{\ell}(\boldsymbol{h}, \vx, \vp) \ge  -(p_{max} + p_{\hat{y}}) \log(\frac{p_{max} + p_{\hat{y}}}{2}) + p_{y_{max}} \log(p_{y_{max}}) + p_{\hat{y}} \log(p_{\hat{y}}),
\end{align*}
and
\begin{align*}
    &\inf_{\vx \in \mathcal{X}} (\inf_{\boldsymbol{h} \in \mathcal{H}_{\hat{y}}(\vx)} (\mathscr{C}_{\ell}(\boldsymbol{h}, \vx, \vp) - \inf_{\boldsymbol{h} \in \mathcal{H}} \mathscr{C}_{\ell}(\boldsymbol{h}, \vx, \vp)) )\ge -(p_{y_{max}} + p_{\hat{y}}) \log(\frac{p_{y_{max}} + p_{\hat{y}}}{2}) + p_{y_{max}} \log(p_{y_{max}}) + p_{p_{\hat{y}}} \log(p_{p_{\hat{y}}}).
\end{align*}
Now, we meet the following problem
\begin{equation*}
\begin{split}
&\min_{\vp} \,\, -(p_{y_{max}} + p_{\hat{y}}) \log(\frac{p_{y_{max}} + p_{\hat{y}}}{2}) + p_{y_{max}} \log(p_{y_{max}}) + p_{\hat{y}} \log(p_{\hat{y}})\\
&s.t.\quad  
\left\{\begin{array}{lc}
p_{y_{max}} - p_{\hat{y}} = t, \forall i,\\
\sum_{i=1}^K p_i = 1, \\
p_i \ge 0, \forall i,
\end{array}\right.
\end{split}
\end{equation*}
which is equivalent to find the minimum of $-(2p_{\hat{y}} + t) \log(\frac{2p_{\hat{y}} + t}{2}) + (p_{\hat{y}} + t) \log((p_{\hat{y}} + t)) +p_{\hat{y}} \log(p_{\hat{y}})$ when $p_{\hat{y}} \in [0, \frac{1-t}{2}]$. By Lemma~\ref{lemma: technical lemma 8}, we know it is $\frac{1+t}{2} \log(1+t) + \frac{1-t}{2} \log(1-t)$. Thus,
\begin{align*}
    \mathcal{J}_{\ell}(t) &= \inf_{\hat{y} \ne y_{max}} \inf_{\vp \in \{\vp: \vp \in \Delta_k,p_{max} - p_{\hat{y}} = t\}} \inf_{\vx \in \mathcal{X}}( \inf_{\boldsymbol{h} \in \mathcal{H}_{\hat{y}}(\vx)} \mathscr{C}_{\ell}(\boldsymbol{h}, \vx, \vp) - \inf_{\boldsymbol{h} \in \mathcal{H}} \mathscr{C}_{\ell}(\boldsymbol{h}, \vx, \vp))\\
    &\ge \inf_{\hat{y} \ne y_{max}} -(2-t)\log(\frac{2-t}{2}) + (1-t)\log(1-t)\\
    &= \frac{1+t}{2} \log(1+t) + \frac{1-t}{2} \log(1-t)\\
    &\ge \frac{t^2}{2}. & (\text{Lemma~\ref{lemma: technical lemma 9}})
\end{align*}
Let $g(t) = \frac{t^2}{4}$ in Theorem~\ref{cor: Distribution-independent convex Psi bound}, we have
\begin{align*}
\frac{1}{2} (R_{\ell_{0-1}}(\boldsymbol{h}) - R^*_{\ell_{0-1}, \mathcal{H}} + M_{\ell_{0-1}, \mathcal{H}})^2 \leq R_{\ell_{log}}(\boldsymbol{h}) - R^*_{\ell_{log}, \mathcal{H}} + M_{\ell_{log}, \mathcal{H}},
\end{align*}
which implies
\begin{align*}
R_{\ell_{0-1}}(\boldsymbol{h}) - R^*_{\ell_{0-1}, \mathcal{H}} + M_{\ell_{0-1}, \mathcal{H}}  \leq \sqrt{2}(R_{\ell_{log}}(\boldsymbol{h}) - R^*_{\ell_{log}, \mathcal{H}} + M_{\ell_{log}, \mathcal{H}})^{\frac{1}{2}}.
\end{align*}
By Lemma~\ref{lemma: property of M_Psi_F}, we have  $M_{\ell_{0-1}, \mathcal{H}}$ coincides with the approximation error $R_{\ell_{0-1}, \mathcal{H}}^* -  R_{\ell_{0-1}, \mathcal{H}_{all}}^*$. We also note that $M_{\ell_{log}, \mathcal{H}}$ coincides with $R_{\ell_{log}, \mathcal{H}}^* -  R_{\ell_{log}, \mathcal{H}_{all}}^*$ because
\begin{align*}
M_{\ell_{log}, \mathcal{H}} &= R_{\ell_{log}, \mathcal{H}}^* - \mathbb{E}_{\vx}[\mathscr{C}_{\ell_{log}, \mathcal{H}}^*(\vx)] \\
&=R_{\ell_{log}, \mathcal{H}}^* - \mathbb{E}_{\vx}[\inf_{\boldsymbol{h} \in \mathcal{H}} \mathscr{C}_{\ell_{log}}(\boldsymbol{h}, \vx, \vp(\vx)] \\
&=R_{\ell_{log}, \mathcal{H}}^* - \mathbb{E}_{\vx}[- \sum_{y=1}^K p_y(\vx) \log(p_y(\vx))] \\
&=R_{\ell_{log}, \mathcal{H}}^* -  R_{\ell_{log}, \mathcal{H}_{all}}^*.\\
\end{align*}
Finally, we can conclude that
\begin{align*}
R_{\ell_{0-1}}(\boldsymbol{h}) - R^*_{\ell_{0-1}, \mathcal{H}} \leq R_{\ell_{0-1}}(\boldsymbol{h}) - R^*_{\ell_{0-1}, \mathcal{H}} + M_{\ell_{0-1}, \mathcal{H}}  \leq \sqrt{2}(R_{\ell_{log}}(\boldsymbol{h}) - R^*_{\ell_{log}, \mathcal{H}_{all}})^\frac{1}{2}.
\end{align*}
\end{proof}

\subsection{Proof of Proposition~\ref{Prop: multiclass exp -n}}
\label{proof: Proposition Prop: multiclass exp -n}
\begin{proof}
Based on the results of Lemma~\ref{Prop: multiclass high E}, for $k_1, k_2$ and $k$ which satisfies $\zeta_{k_1, k_2, k} > 0$, to bound $\mathbb{P}(\vert \Delta a_{Gen, \infty}(\vx, k_1, k_2) \vert \leq \tau n \vert y=k)$, we can write:
\begin{align*}
    &\mathbb{P}(\vert \Delta a_{Gen, \infty}(\vx, k_1, k_2) \vert \leq \tau n \vert y=k) \\
    &\leq \mathbb{P}( \Delta a_{Gen, \infty}(\vx, k_1, k_2) \leq \tau n \vert y=k) \\
    &=\mathbb{P}(  \Delta a_{Gen, \infty}(\vx, k_1, k_2) - \zeta_{k_1, k_2, k} n \leq \tau n - \zeta_{k_1, k_2, k} n\vert y=k) \\
    &=\mathbb{P}(\sum_{i=1}^n \log \frac{{p}(x_i\vert y=k_1) }{{p}(x_i\vert y=k_2) } - \mathbb{E}_{\vx}(\sum_{i=1}^n \log \frac{{p}(x_i\vert y=k_1) }{{p}(x_i\vert y=k_2) }) \leq (\tau  - \zeta_{k_1, k_2, k}) n \vert y=k)\\
    &\leq \exp({- \frac{2(\tau  - \zeta_{k_1, k_2, k})^2n^2}{n(\log \frac{1-\rho_0}{\rho_0} - \log \frac{\rho_0}{1-\rho_0})^2}}) = \exp\bigl({- \frac{(\tau  - \zeta_{k_1, k_2, k})^2n}{2(\log \frac{1-\rho_0}{\rho_0})^2}}\bigr). & \text{(Assumption~\ref{Assumption: parmeters bounded} and Lemma~\ref{lemma: hoeffding})}
\end{align*}
Similar to the above discussion, we have $\mathbb{P}(\vert \Delta a_{Gen, \infty}(\vx, k_1, k_2) \vert \leq \tau n \vert y=k) \leq \exp\bigl({- \frac{(\tau  -\vert \zeta_{k_1, k_2, k}\vert)^2n}{2(\log \frac{1-\rho_0}{\rho_0})^2}}\bigr)$ for $k_1, k_2$ and $k$ which satisfies $\beta_{k_1, k_2, k} < 0$. Finally, we can conclude that:
\begin{align*}
    \widetilde{G}(\tau) &= \max_{k_1,k_2}\sum_{k=1}^K p(y=k)\mathbb{P}(\vert \Delta a_{Gen, \infty}(\vx, k_1, k_2) \vert \leq \tau n \vert y=k)\\
    &\leq \max_{k_1,k_2}\sum_{k=1}^K p(y=k)\exp\bigl({- \frac{(\tau  -\vert \zeta_{k_1, k_2, k}\vert)^2n}{2(\log \frac{1-\rho_0}{\rho_0})^2}}\bigr)\\
    &\leq \max_{k_1,k_2}\exp\bigl({- \frac{(\tau  -\min_{k}\vert \zeta_{k_1, k_2, k}\vert)^2n}{2(\log \frac{1-\rho_0}{\rho_0})^2}}\bigr)\\
    &= \exp\bigl({- \frac{(\tau  - \zeta)^2n}{2(\log \frac{1-\rho_0}{\rho_0})^2}}\bigr) = \exp{(-O((\tau - \zeta)^2n))}.
\end{align*}
Second, we consider the continuous case, the only difference from the discrete case is that the range of $\log \frac{{p}(x_i\vert y=k_1) }{{p}(x_i\vert y=k_2) }$ is $[-\frac{2}{\rho_0}, \frac{2}{\rho_0}]$. So we can get:
\begin{align*}
     \widetilde{G}(\tau) \leq \exp\bigl({- \frac{(\tau  - \zeta)^2n}{2(\frac{4}{\rho_0})^2}}\bigr) = \exp{(-O((\tau - \zeta)^2n))}.
\end{align*}
The proof is complete.
\end{proof}

\section{Details of Simulation Experiments}
\label{app: Configurations of Simulation Experiment}

\subsection{Implementation of Logistic Regression}
We train the logistic regression using scikit-learn's~\cite{scikit-learn} L-BFGS implementation, with a maximum of 1000 iterations. The weight of $\ell_2$ regularization of logistic regression is fixed as 1. All experiments are done on a single GeForce RTX 3090 GPU.

\subsection{Sythentic Dataset}
We construct a simulated multiclass balanced mixture Gaussian distribution dataset, which also satisfies all assumptions. The simulated data distribution satisfies $p(x\vert y=1) \sim \mathcal{N}(x; \{-1\}^n, diag\{\{n\}^{\frac{n}{2}},  \{1\}^{\frac{n}{2}}\})$ and $p(x\vert y=k) \sim \mathcal{N}(x; \{2^{k-2}\}^n, diag\{\{n\}^{\frac{n}{2}},  \{1\}^{\frac{n}{2}}\})$ for $k > 1$, where $\mathcal{N}$ is Gaussian distribution, $diag(\boldsymbol{a})$ means a matrix whose diagonal is $\boldsymbol{a}$, and $\{a\}^n$ means a vector whose length is $n$ and all its elements are $a$.

\subsection{Discussion about the synthetic dataset}
First, we note that the optimal classifier is a linear function, which means that  Assumption~\ref{Assumption: Optimal classifier has finite empirical loss} is valid with $\nu = 0$.

\textbf{Binary case}.
The data distribution satisfies $p(x\vert y=0) \sim \mathcal{N}(x; \{-1\}^n, diag\{\{n\}^{\frac{n}{2}},  \{1\}^{\frac{n}{2}}\})$ and $p(x\vert y=1) \sim \mathcal{N}(\{1\}^n, diag\{\{ n\}^{\frac{n}{2}},  \{1\}^{\frac{n}{2}}\})$. The boundary of Bayes classifier $\Delta a_{Gen}(\vx, 1, 0)$ can be calculated as follows:
\begin{align*}
    \Delta a_{Gen}(\vx, 1, 0) &= \sum_{i=1}^n \log \frac{\frac{1}{\sqrt{2\pi}\sigma_i} \exp({-\frac{(x_i - \mu_{1i})^2}{2\sigma_i^2}})}{\frac{1}{\sqrt{2\pi}\sigma_i} \exp({-\frac{(x_i - \mu_{0i})^2}{2\sigma_i^2}})} + \log \frac{q}{1-q} \\
    &= \sum_{i=1}^n \frac{\mu_{1i} - \mu_{0i}}{\sigma_i^2}x_i  + \sum_{i=1}^n \frac{ \mu_{0i}^2 -  \mu_{1i}^2}{2\sigma_i^2} + \log \frac{q}{1-q} \\
    &= \sum_{i=1}^{n/2} \frac{2}{n} x_i + \sum_{i=\frac{n}{2} + 1}^{n} 2 x_i.
\end{align*}
It is a linear function. In addition, the Bayes error $BE$ can be obtained as follows.
\begin{align*}
    BE &= \frac{1}{(2\pi)^{\frac{n}{2}} (n^\frac{n}{2})^\frac{1}{2}} \int_{\sum_{i=1}^{n/2} \frac{2}{n} x_i + \sum_{i=\frac{n}{2} + 1}^{n} 2 x_i < 0} \exp(-\frac{1}{2} ( \sum_{i=1}^{n/2} \frac{(x_i - 1)^2}{n} + \sum_{i=\frac{n}{2} + 1}^{n} (x_i - 1)^2 )) dx_1\dots dx_n \\
    &= \frac{1}{(2\pi)^{\frac{n}{2}} (n^\frac{n}{2})^\frac{1}{2}} \int_{\sum_{i=1}^{n/2} \frac{2}{n} y_i + \sum_{i=\frac{n}{2} + 1}^{n} 2 y_i < -(n+1)} \exp(-\frac{1}{2} ( \sum_{i=1}^{n/2} \frac{y_i^2}{n} + \sum_{i=\frac{n}{2} + 1}^{n} y_i^2 )) dy_1\dots dy_n \\
    &= \frac{1}{(2\pi)^{\frac{n}{2}}} \int_{\sum_{i=1}^{n/2} \frac{2}{\sqrt{n}} z_i + \sum_{i=\frac{n}{2} + 1}^{n} 2 z_i < -(n+1)} \exp(-\frac{1}{2} ( \sum_{i=1}^{n} {z_i^2} )) dz_1\dots dz_n \\
    &= \int_{\sum_{i=1}^{\frac{n}{2}} \frac{\boldsymbol{z_i}}{\sqrt{n}} +  \sum_{i=\frac{n}{2} + 1}^{n} \boldsymbol{z_i} + \frac{n+1}{2} < 0} \mathcal{N}(\boldsymbol{z}; 0, \boldsymbol{I}) d\boldsymbol{z},
\end{align*}
which approaches 0 quickly as $n$ increases and can be approximated by the Monte Caro method efficiently.

\textbf{Multiclass case}.
The boundary of Bayes classifier $a(\vx, k_1, k_2)$ for class $k_1 = 1$ and $k_2$ can be calculated as follows:
\begin{align*}
    \Delta a_{Gen}(\vx, k_1, k_2) &= \sum_{i=1}^n \log \frac{\frac{1}{\sqrt{2\pi}\sigma_i} \exp({-\frac{(x_i - \mu_{k_1i})^2}{2\sigma_i^2}})}{\frac{1}{\sqrt{2\pi}\sigma_i} \exp({-\frac{(x_i - \mu_{k_2i})^2}{2\sigma_i^2}})} + \log \frac{q_{k_1}}{q_{k_2}} \\
    &= \sum_{i=1}^n \frac{\mu_{k_1i} - \mu_{k_2i}}{\sigma_i^2}x_i  + \sum_{i=1}^n \frac{ \mu_{k_2i}^2 -  \mu_{k_1i}^2}{2\sigma_i^2} + \log \frac{q_{k_1}}{q_{k_2}} \\
    &= \sum_{i=1}^{n/2} \frac{-1 - 2^{k_2 - 2}}{n} x_i + \sum_{i=\frac{n}{2} + 1}^{n} (-1 + 2^{k_2-2}) x_i.
\end{align*}
In addition, the boundary of Bayes classifier $a(\vx, k_1, k_2)$ for class $k_1 \ne 1$ and $k_2 \ne 1$ can be calculated as follows:
\begin{align*}
    \Delta a_{Gen}(\vx, k_1, k_2) &= \sum_{i=1}^n \log \frac{\frac{1}{\sqrt{2\pi}\sigma_i} \exp({-\frac{(x_i - \mu_{k_1i})^2}{2\sigma_i^2}})}{\frac{1}{\sqrt{2\pi}\sigma_i} \exp({-\frac{(x_i - \mu_{k_2i})^2}{2\sigma_i^2}})} + \log \frac{q_{k_1}}{q_{k_2}} \\
    &= \sum_{i=1}^n \frac{\mu_{k_1i} - \mu_{k_2i}}{\sigma_i^2}x_i  + \sum_{i=1}^n \frac{ \mu_{k_2i}^2 -  \mu_{k_1i}^2}{2\sigma_i^2} + \log \frac{q_{k_1}}{q_{k_2}} \\
    &= \sum_{i=1}^{n/2} \frac{2^{k_1-2} - 2^{k_2-2}}{n} x_i + \sum_{i=\frac{n}{2} + 1}^{n} (2^{k_1-2} - 2^{k_2-2}) x_i + (4^{k_1-2} - 4^{k_2-2})\frac{n + 1}{4}
\end{align*}
The Bayes error is not easy to obtain in an analytic version. However, the test error can decrease to less than $10^{-4}$ in our multiclass experiments, so we set 0 as the estimated asymptotic error.

Second, Assumption~\ref{Assumption: multiclass KL} holds in this case, that is, for all $k_1, k_2 (k_1 \ne k_2)$ and $k \in \mathcal{Y}$, it holds that $\lvert \sum_{i=1}^n (D(p(x_i \vert y=k) \Vert p(x_i \vert y=k_1)) - D(p(x_i \vert y=k) \Vert p(x_i \vert y=k_2))) \rvert = \beta_{k_1, k_2, k}n = \Omega(n)$. For all $k_1, k_2 (k_1 \ne k_2) \in \mathcal{Y}$, we have 

\begin{align*}
\sum_{i=1}^n D(p(x_i \vert y=k_1) \Vert p(x_i \vert y=k_2)) &= \sum_{i=1}^n D(p(x_i \vert y=k_2) \Vert p(x_i \vert y=k_1)) \\
&= \sum_{i=1}^n \frac{ (\mu_{k_2i} -  \mu_{k_1i})^2}{2\sigma_i^2} = \sum_{i=1}^{n/2} \frac{(2^{k_1-2} - 2^{k_2-2})^2}{2n} + \sum_{i=\frac{n}{2} + 1}^{n} (2^{k_1-2} - 2^{k_2-2})^2\\
&= \frac{(2^{k_1-2} - 2^{k_2-2})^2}{4} + \frac{n}{2} (2^{k_1-2} - 2^{k_2-2})^2.
\end{align*}
So we have

\begin{align*}
&\lvert \sum_{i=1}^n (D(p(x_i \vert y=k) \Vert p(x_i \vert y=k_1)) - D(p(x_i \vert y=k) \Vert p(x_i \vert y=k_2))) \rvert \\
&=\lvert \frac{(2^{k_1-2} - 2^{k-2})^2}{4} + \frac{n}{2} (2^{k_1-2} - 2^{k-2})^2 - \frac{(2^{k_2-2} - 2^{k-2})^2}{4} - \frac{n}{2} (2^{k_2-2} - 2^{k-2})^2\rvert  = O(n).
\end{align*}
Third, Assumption~\ref{Assumption: multiclass likelihood ratio var} holds as well. This can be obtained by the property of conditional independence directly, that is, for all $k_1, k_2 (k_1 \ne k_2)$ and $k \in \mathcal{Y}$, it holds that $\mathbb{V}_{\vx}[\sum_{i=1}^n \log \frac{{p}(x_i\vert y=k_1) }{{p}(x_i\vert y=k_2) } \vert y = k] = \sum_{i=1}^n \mathbb{V}_{\vx}[\log \frac{{p}(x_i\vert y=k_1) }{{p}(x_i\vert y=k_2) } \vert y = k] = O(n)$.

Finally, we note that we can directly scale this dataset because scaling will not influence the establishment of the above assumptions. In our multiclass experiments ($K > 2$), we scale the dataset to boost logistic regression converging faster. The scale function we use is $\boldsymbol{f}(\vx) = \frac{\vx}{2^{K-3}} - 1$, which can make the mean of each class to $[-1,1]$.

\subsection{The number of samples required to converge}
For a fixed $K$, we traversal $n$ from 100 to 1000 gradually. For each selected $n$, we randomly generate $1 \times 10^4$ samples as a test set. We increase the training dataset size $m$ gradually until the errors of two classifiers approach their asymptotic error. Specially, we conduct 5 random repeats to keep the stability of our results. We record the training set size $m_{conv}$ when the gap between the error and the estimation of asymptotic error is less than $\epsilon_0 = 0.01$ for the first time.

\subsection{Additional Results of Simulations}
We present results with $K = 2,3,7$ here. Consistently, logistic regression and na\"ive Bayes require $O(n)$ and $O(\log n)$ samples to approach the estimated asymptotic error respectively. Error bars represent the variance estimated by 5 runs.

\begin{figure}[htbp]
\centering

\subfloat[$K = 2$]{
\includegraphics[width=0.3\columnwidth]{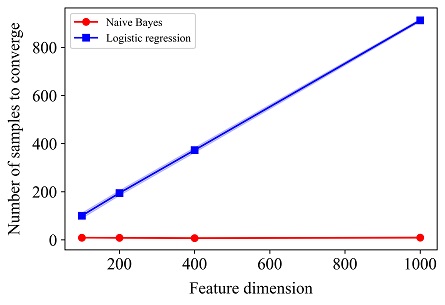}
}%
\subfloat[$K = 3$]{
\includegraphics[width=0.3\columnwidth]{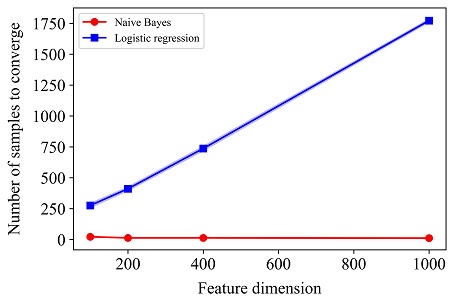}
}%
\subfloat[$K = 7$]{
\includegraphics[width=0.3\columnwidth]{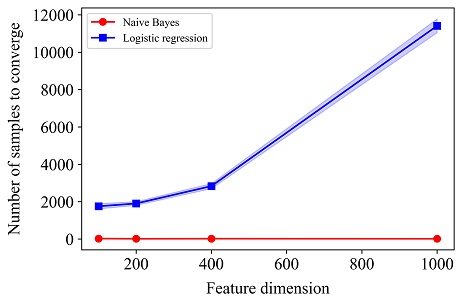}
}%

\centering
\caption{Additional results of simulations with $K = 2,3,7$.}
\label{figures: additional simulations}
\end{figure}

\section{Details of Deep Learning Experiments}
\label{app: details of DL Experiments}

\subsection{Models}
\textbf{ViT}. We include ViT-B/16~\cite{dosovitskiy2020image} checkpoint pretrained on the ImageNet-21k dataset~\cite{imagenet}.

\textbf{ResNet}. We add the ResNet50 checkpoint released by Pytorch~\cite{pytorch}.

\textbf{CLIP image encoder}. We use the image encoder released by CLIP~\cite{CLIP} project with ResNet50 backbone.

\textbf{MoCov2}. We include the MoCov2~\cite{DBLP:journals/corr/abs-2003-04297MocoV2} checkpoint trained with 800 epochs on the ImageNet dataset. The backbone is ResNet50.

\textbf{SimCLRv2}. The SimCLRv2~\cite{DBLP:conf/nips/simclrv2} project released various pre-trained and fine-tuned models. We use the pretrain-only checkpoint with selective Kernels. The backbone is ResNet50.

\textbf{MAE}. We adopt pre-trained checkpoint in~\cite{DBLP:conf/cvpr/HeCXLDG22MAE}. The backbone is ViT-B/16.

\textbf{SimMIM}.We use the checkpoint pre-trained on the ImageNet-1K dataset with 800 epochs released in~\cite{DBLP:conf/cvpr/simmim}. The backbone is ViT-B/16.

The used codes and their licenses are listed as follows.

\begin{table}[h]
\centering
\caption{The used codes and licenses.}
\vskip 0.15in
\begin{tabular}{ccc} 
\toprule
URL                                                    & citations & License                                                    \\ 
\midrule
https://github.com/google-research/vision\_transformer & \cite{dosovitskiy2020image}       & Apache-2.0 License                                         \\
https://github.com/pytorch/pytorch                     & \cite{pytorch}   &  \href{https://github.com/pytorch/pytorch/blob/master/LICENSE}{License}      \\
https://github.com/openai/CLIP                         & \cite{CLIP}      & MIT License                                                \\
https://github.com/facebookresearch/moco               & \cite{DBLP:journals/corr/abs-2003-04297MocoV2}      & MIT License                                                \\
https://github.com/google-research/simclr              & \cite{DBLP:conf/nips/simclrv2}    & Apache-2.0 License                                         \\
https://github.com/Separius/SimCLRv2-Pytorch              & -    & GPL-3.0 license     \\
https://github.com/facebookresearch/mae                & \cite{DBLP:conf/cvpr/HeCXLDG22MAE}       & \href{https://github.com/facebookresearch/mae/blob/main/LICENSE}{License}  \\
https://github.com/microsoft/SimMIM                    & \cite{DBLP:conf/cvpr/simmim}    & MIT License                                                \\
https://github.com/scikit-learn/scikit-learn                    & \cite{scikit-learn}    & BSD-3-Clause License                                                \\
\bottomrule
\end{tabular}
\end{table}

\subsection{Feature preprocessing}
For the reason that our theory assumes that $\mathcal{X} = [0,1]^n$, we scale each dimension of features to $[0,1]$. It is implemented by using the MinMaxScaler supported in scikit-learn's~\cite{scikit-learn}. Empirically, we note this transformation will not influence the happening of the ``two regimes'' phenomenon in practice.

\subsection{Additional Results of Validating the Assumptions}
\label{app: Additional Results of Validating the Assumptions}

\begin{figure}[htbp]
\centering

\subfloat[ViT]{
\includegraphics[width=0.35\columnwidth]{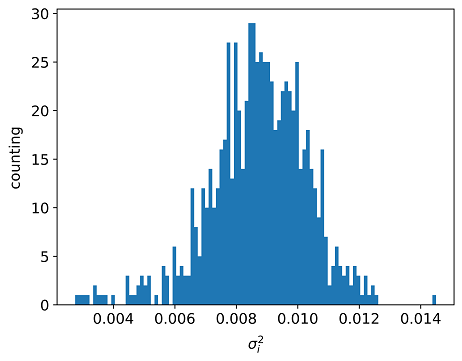}
}
\subfloat[ResNet]{
\includegraphics[width=0.35\columnwidth]{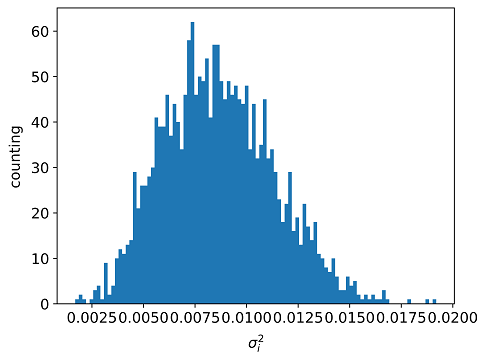}
}

\subfloat[CLIP]{
\includegraphics[width=0.35\columnwidth]{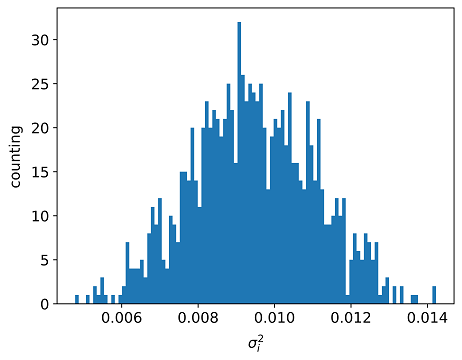}
}
\subfloat[MoCov2]{
\includegraphics[width=0.35\columnwidth]{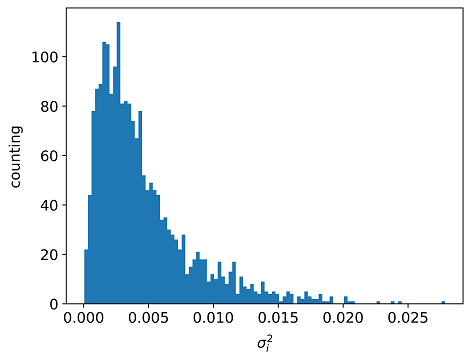}
}

\subfloat[SimCLRv2]{
\includegraphics[width=0.35\columnwidth]{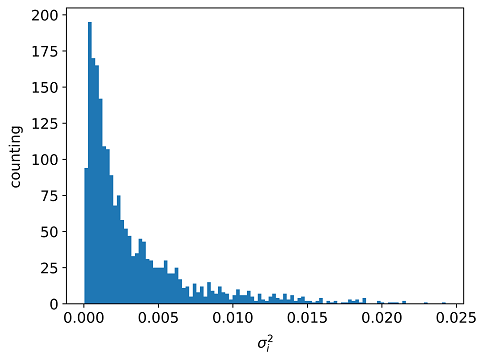}
}
\subfloat[MAE]{
\includegraphics[width=0.35\columnwidth]{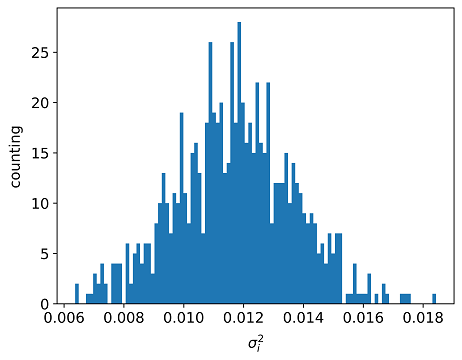}
}

\subfloat[SimMIM]{
\includegraphics[width=0.35\columnwidth]{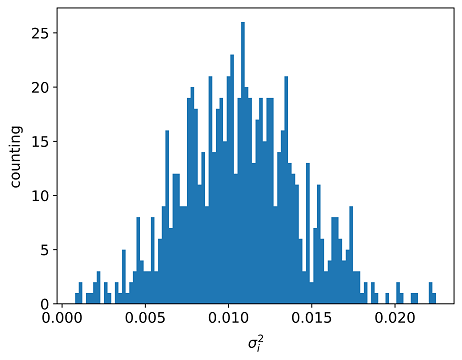}
}

\centering
\caption{Distribution histogram of $\sigma_i^2$}
\label{figures: sigmas}  
\end{figure}

\begin{figure}[htbp]
\centering

\subfloat[ViT]{
\includegraphics[width=0.35\columnwidth]{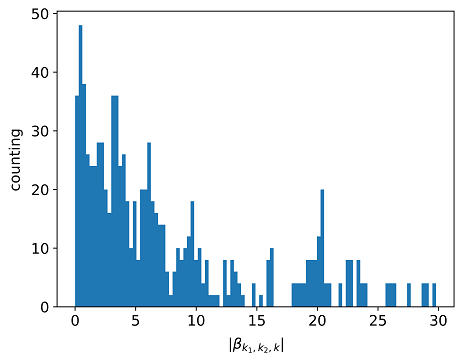}
}
\subfloat[ResNet]{
\includegraphics[width=0.35\columnwidth]{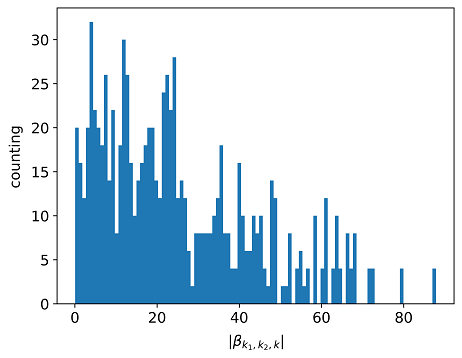}
}

\subfloat[CLIP]{
\includegraphics[width=0.35\columnwidth]{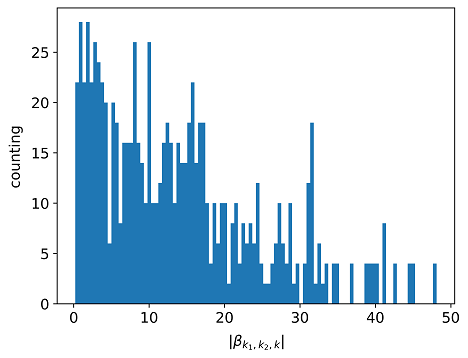}
}
\subfloat[MoCov2]{
\includegraphics[width=0.35\columnwidth]{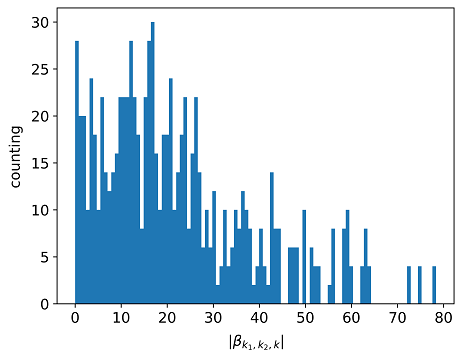}
}

\subfloat[SimCLRv2]{
\includegraphics[width=0.35\columnwidth]{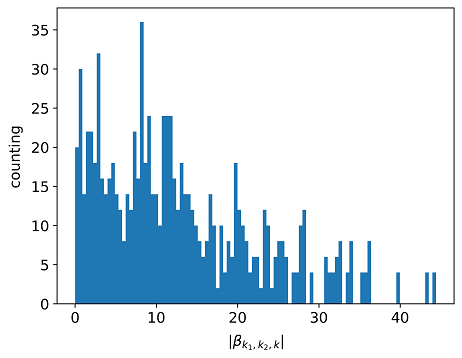}
}
\subfloat[MAE]{
\includegraphics[width=0.35\columnwidth]{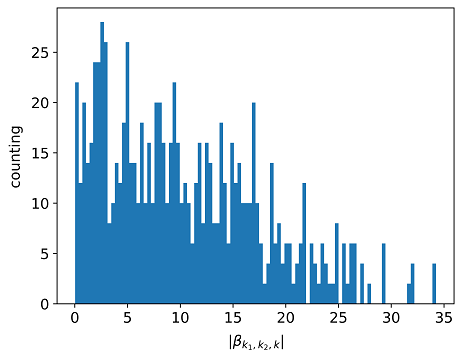}
}

\subfloat[SimMIM]{
\includegraphics[width=0.35\columnwidth]{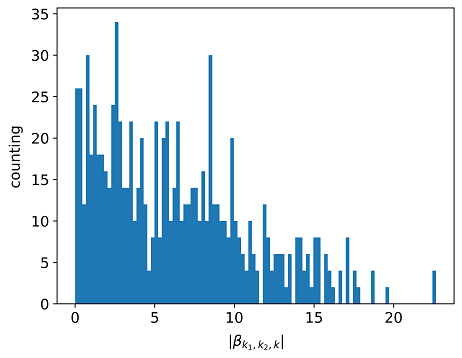}
}

\caption{Distribution histogram of $\vert\beta_{k_1,k_2,k}\vert$.} 
\label{figures: beta}  
\end{figure}

\begin{figure}[htbp]
\centering

\subfloat[ViT]{
\includegraphics[width=0.35\columnwidth]{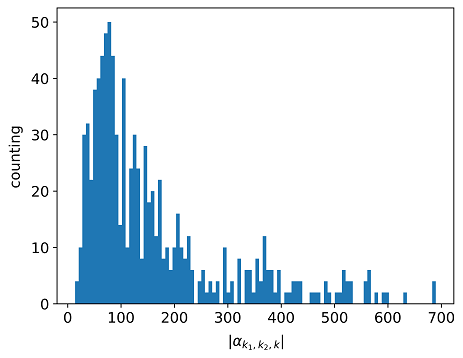}
}
\subfloat[ResNet]{
\includegraphics[width=0.35\columnwidth]{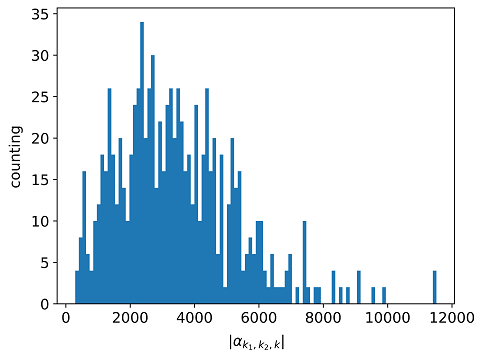}
}

\subfloat[CLIP]{
\includegraphics[width=0.35\columnwidth]{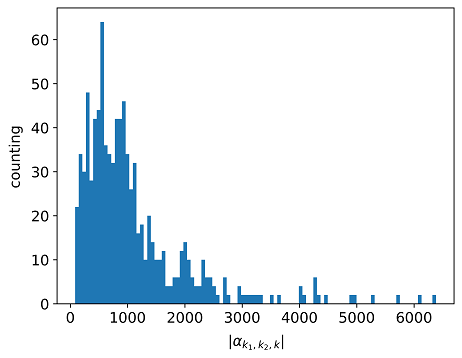}
}
\subfloat[MoCov2]{
\includegraphics[width=0.35\columnwidth]{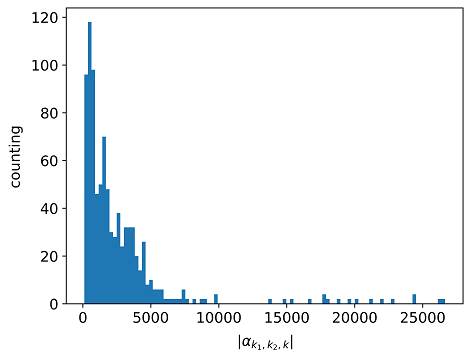}
}

\subfloat[SimCLRv2]{
\includegraphics[width=0.35\columnwidth]{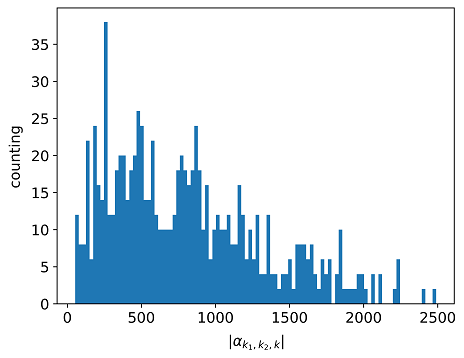}
}
\subfloat[MAE]{
\includegraphics[width=0.35\columnwidth]{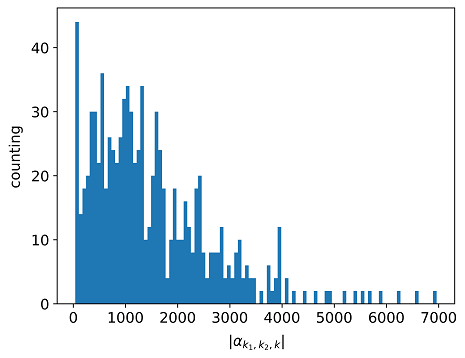}
}

\subfloat[SimMIM]{
\includegraphics[width=0.35\columnwidth]{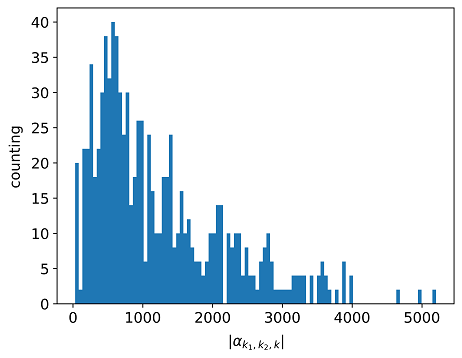}
}

\caption{Distribution histogram of $\alpha_{k_1,k_2,k}$.} 
\label{figures: alpha}  
\end{figure}

\newpage

\subsection{Additional Results of Deep Learning}
\label{app: Additional Deep Learning Results}

\begin{figure}[htbp]
\centering

\subfloat[CIFAR10, small m]{
\includegraphics[width=0.45\columnwidth]{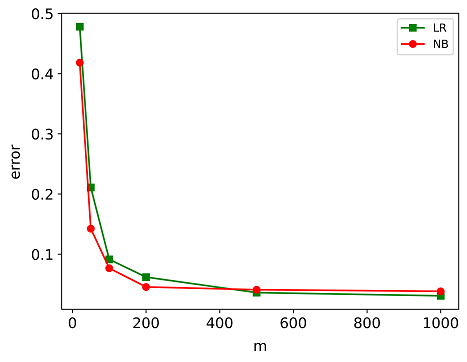}
}
\subfloat[CIFAR10, all m]{
\includegraphics[width=0.45\columnwidth]{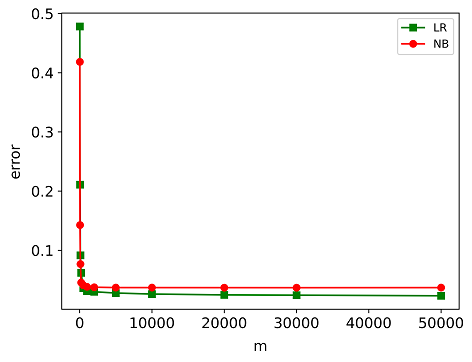}
}

\subfloat[CIFAR100, small m]{
\includegraphics[width=0.45\columnwidth]{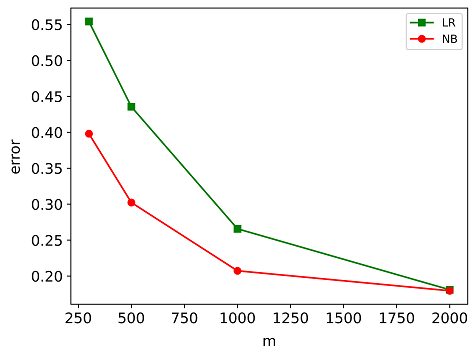}
}
\subfloat[CIFAR100, all m]{
\includegraphics[width=0.45\columnwidth]{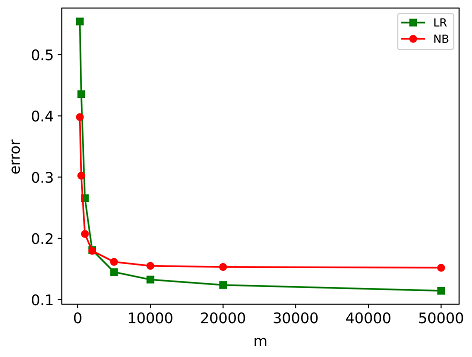}
}

\centering
\caption{Comparison between na\"ive Bayes and logistic regression trained on features extracted by ViT.}
\label{figures: vit}
\end{figure}

\begin{figure}[htbp]
\centering

\subfloat[CIFAR10, small m]{
\includegraphics[width=0.45\columnwidth]{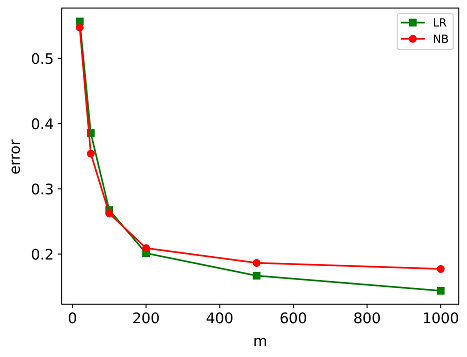}
}%
\subfloat[CIFAR10, all m]{
\includegraphics[width=0.45\columnwidth]{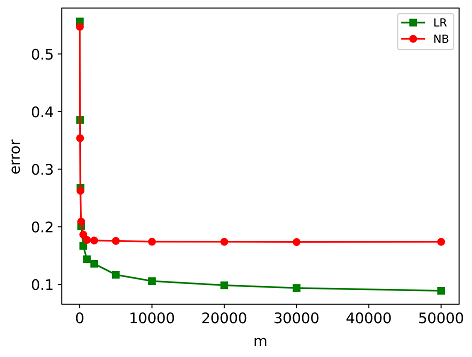}
}%

\subfloat[CIFAR100, small m]{
\includegraphics[width=0.45\columnwidth]{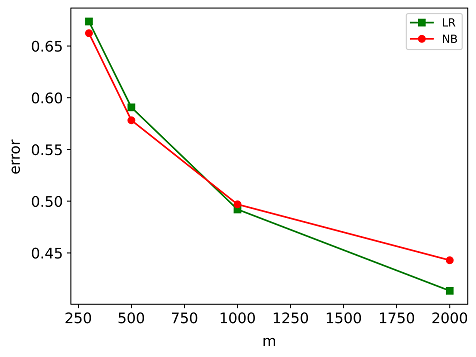}
}%
\subfloat[CIFAR100, all m]{
\includegraphics[width=0.45\columnwidth]{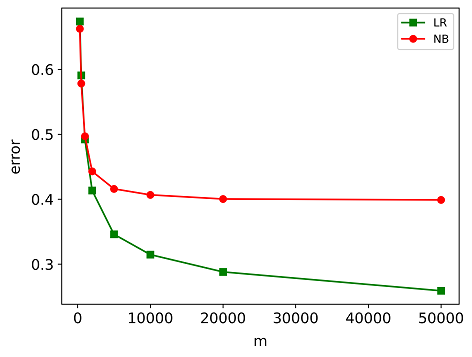}
}%

\centering
\caption{Comparison between na\"ive Bayes and logistic regression trained on features extracted by ResNet50.}
\label{figures: resnet}
\end{figure}

\begin{figure}[htbp]
\centering

\subfloat[CIFAR10, small m]{
\includegraphics[width=0.45\columnwidth]{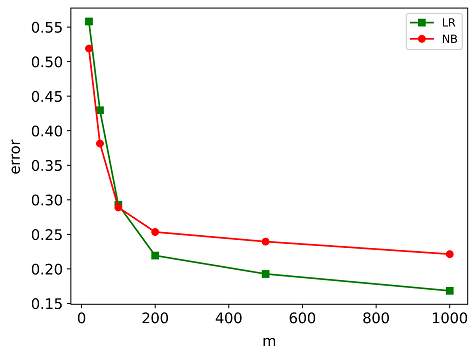}
}%
\subfloat[CIFAR10, all m]{
\includegraphics[width=0.45\columnwidth]{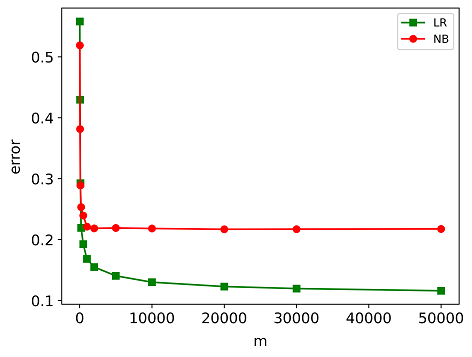}
}%

\subfloat[CIFAR100, small m]{
\includegraphics[width=0.45\columnwidth]{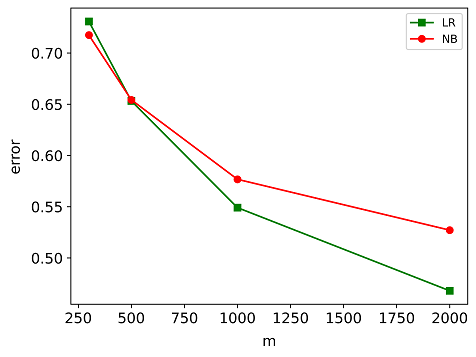}
}%
\subfloat[CIFAR100, all m]{
\includegraphics[width=0.45\columnwidth]{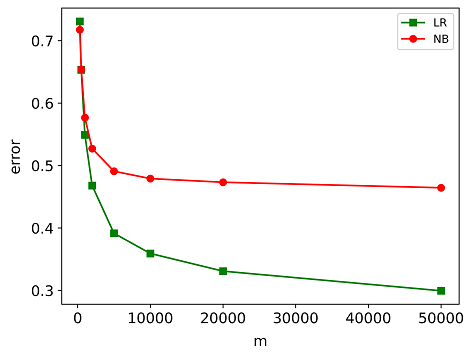}
}%

\centering
\caption{Comparison between na\"ive Bayes and logistic regression trained on features extracted by CLIP.}
\label{figures: clip}
\end{figure}

\begin{figure}[htbp]
\centering

\subfloat[CIFAR10, small m]{
\includegraphics[width=0.45\columnwidth]{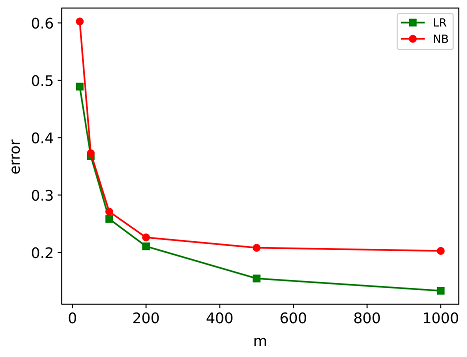}
}%
\subfloat[CIFAR10, all m]{

\includegraphics[width=0.45\columnwidth]{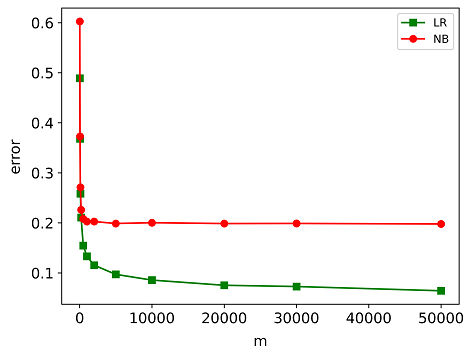}

}%

\subfloat[CIFAR100, small m]{

\includegraphics[width=0.45\columnwidth]{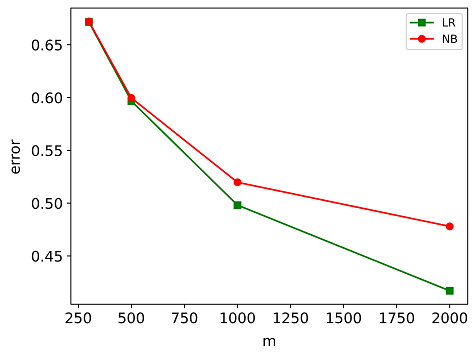}

}%
\subfloat[CIFAR100, all m]{

\includegraphics[width=0.45\columnwidth]{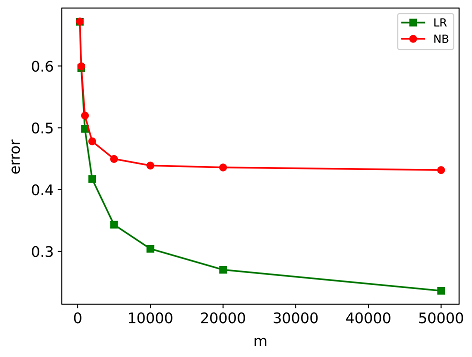}

}%

\centering
\caption{Comparison between na\"ive Bayes and logistic regression trained on features extracted by MoCov2.}
\label{figures: mocov2}
\end{figure}

\begin{figure}[htbp]
\centering

\subfloat[CIFAR10, small m]{

\includegraphics[width=0.45\columnwidth]{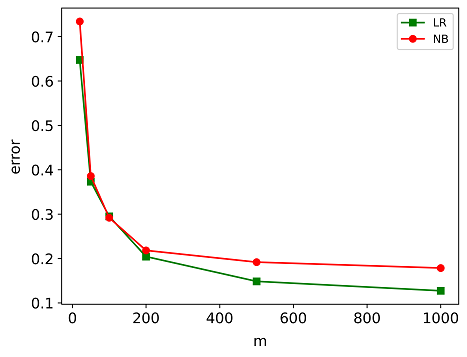}

}%
\subfloat[CIFAR10, all m]{

\includegraphics[width=0.45\columnwidth]{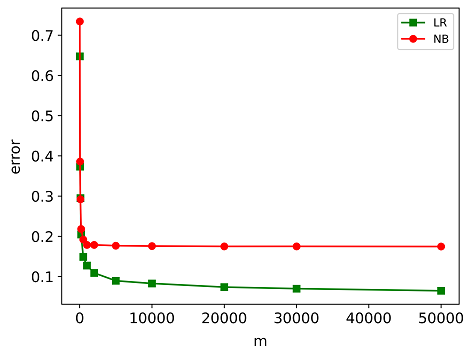}

}%

\subfloat[CIFAR100, small m]{

\includegraphics[width=0.45\columnwidth]{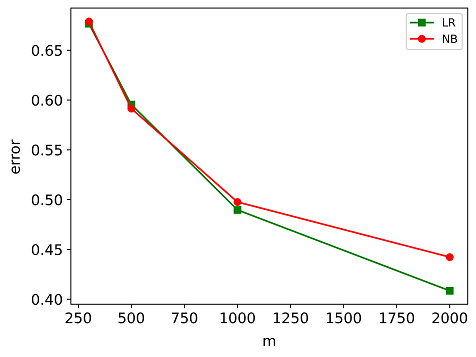}

}%
\subfloat[CIFAR100, all m]{

\includegraphics[width=0.45\columnwidth]{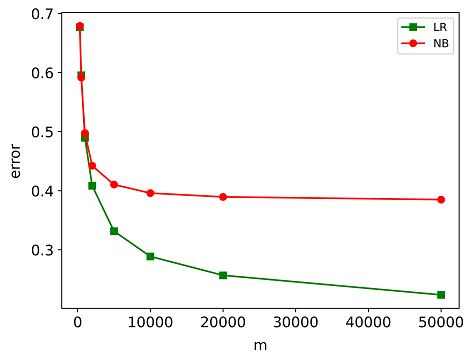}

}%

\centering
\caption{Comparison between na\"ive Bayes and logistic regression trained on features extracted by SimCLRv2.}
\label{figures: simclrv2}
\end{figure}

\begin{figure}[htbp]
\centering

\subfloat[CIFAR10, small m]{

\includegraphics[width=0.45\columnwidth]{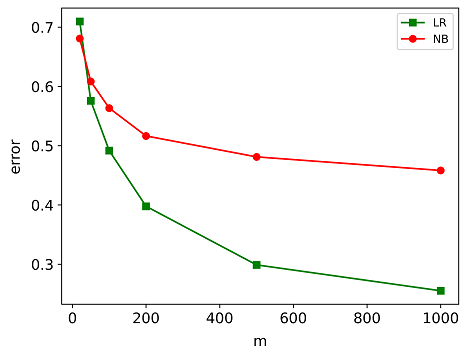}

}%
\subfloat[CIFAR10, all m]{

\includegraphics[width=0.45\columnwidth]{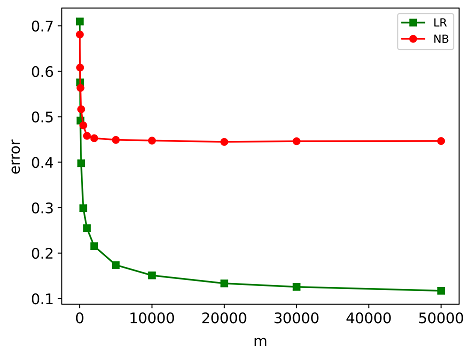}

}%

\subfloat[CIFAR100, small m]{

\includegraphics[width=0.45\columnwidth]{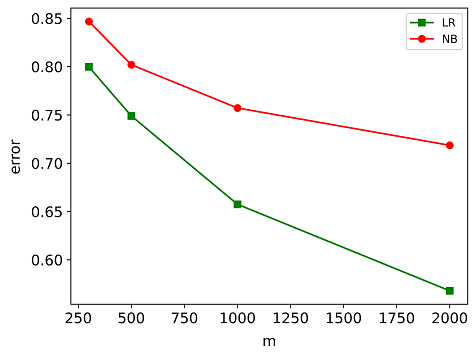}

}%
\subfloat[CIFAR100, all m]{

\includegraphics[width=0.45\columnwidth]{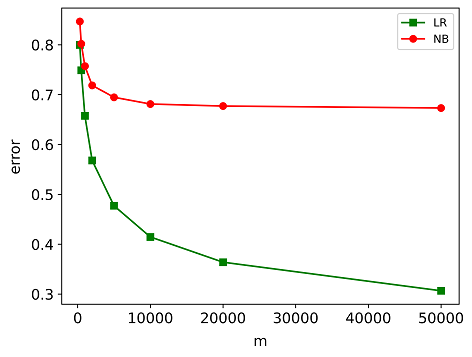}

}%

\centering
\caption{Comparison between na\"ive Bayes and logistic regression trained on features extracted by MAE.}
\label{figures: mae}
\end{figure}

\begin{figure}[htbp]
\centering

\subfloat[CIFAR10, small m]{

\includegraphics[width=0.45\columnwidth]{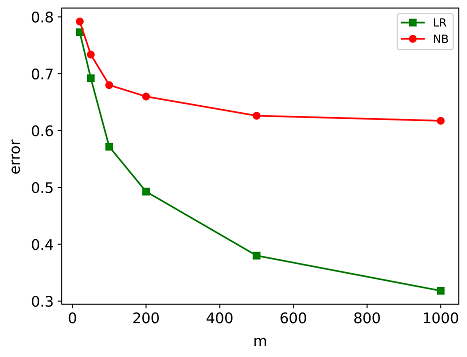}

}%
\subfloat[CIFAR10, all m]{

\includegraphics[width=0.45\columnwidth]{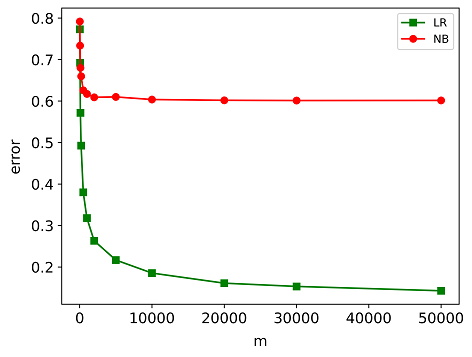}

}%

\subfloat[CIFAR100, small m]{

\includegraphics[width=0.45\columnwidth]{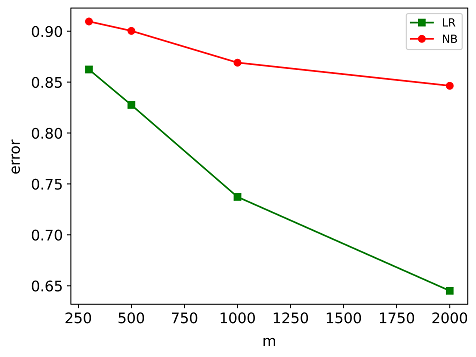}

}%
\subfloat[CIFAR100, all m]{

\includegraphics[width=0.45\columnwidth]{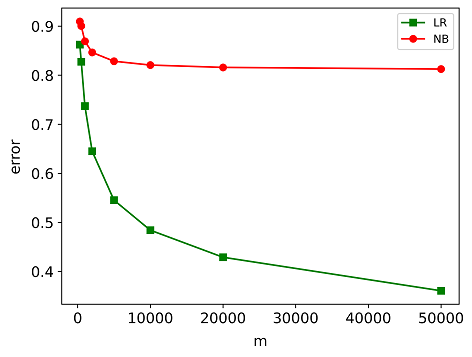}

}%

\centering
\caption{Comparison between na\"ive Bayes and logistic regression trained on features extracted by SimMIM.}
\label{figures: simmim}
\end{figure}

\end{document}